\documentclass[11pt]{article}
\usepackage{fullpage}
\usepackage{epsf}

\usepackage{graphics}

\usepackage[normalem]{ulem}

\usepackage[dvipsnames]{xcolor}
\usepackage{subcaption}
\usepackage{psfrag}
\usepackage{comment}

\usepackage{color}
\usepackage{pgfplots}
\usepackage{pgfplotstable}
\usepackage{multirow}
\usepackage{expl3}
\usepackage{mathtools}
\usepackage{amsfonts}

\usepackage{bm}
\usepackage{amsmath}
\usepackage{amssymb}
\usepackage{amsfonts}

\usepackage{MnSymbol}
\usepackage{scalerel}
\usepackage{nccmath}
\usepackage{algorithmic}

\usepackage[utf8]{inputenc} % allow utf-8 input
\usepackage[T1]{fontenc}    % use 8-bit T1 fonts
\usepackage{hyperref}       % hyperlinks
\usepackage{url}            % simple URL typesetting
\usepackage{booktabs}       % professional-quality tables
\usepackage{amsfonts}       % blackboard math symbols
\usepackage{nicefrac}       % compact symbols for 1/2, etc.
\usepackage{microtype}      % microtypography
\usepackage{amsthm}
\usepackage{xspace}
\usepackage{appendix}

\usepackage[style = alphabetic,citestyle=alphabetic,maxbibnames=99,backend=biber,natbib=true,backref=true]{biblatex}
\DefineBibliographyStrings{english}{%
  backrefpage = {Cited on page},% originally "cited on page"
  backrefpages = {Cited on pages},% originally "cited on pages"
}

\usepackage{hyperref}
\hypersetup{colorlinks,linkcolor = violet, urlcolor  = YellowOrange, citecolor = Aquamarine, anchorcolor = YellowOrange}
\usepackage[suppress]{color-edits}
\addauthor{kb}{blue}
\addauthor{ks}{red}

\newcommand{\trace}{\ensuremath{\operatorname{tr}}}
\newcommand{\real}{\ensuremath{\mathbb{R}}}

\newcommand{\ind}{\mathbb{I}}

\newcommand{\En}{\vphantom{p}\mathbb{E}} %Expectation no bracket
\newcommand{\Eu}[1]{\underset{#1}{\vphantom{p}\mathbb{E}}}  %underset

\DeclareMathOperator*{\argmin}{argmin}

\DeclareMathOperator*{\arginf}{arginf}
\DeclareMathOperator*{\argsup}{argsup}

\DeclareMathOperator*{\pinf}{\vphantom{p}inf}
\DeclareMathOperator*{\psup}{\vphantom{p}sup}
\newcommand\inner[2]{\langle #1, #2 \rangle}

\newcommand{\multiminimax}[1]{\ensuremath{\left\llangle #1 \right\rrangle}} %minimax operator

\newcommand{\alg}{\mathcal{A}} %some generic algorithm
\DeclareMathOperator{\polregT}{Reg^{\mathsf{pol}}_T}
\newcommand{\ztr}{\mathbf{\z}}
\newcommand{\zdtr}{\bm{\zd}}
\newcommand{\stf}[2]{\st_{#1}[{#2}]}
\newcommand{\stt}{\tilde{\st}}
\newcommand{\F}{\Pi}
\newcommand{\Fgen}{\mathcal{F}}
\newcommand{\psw}{p^{\sf{sw}}}
\newcommand{\Ksw}{K^{\sf{sw}}}
\newcommand{\ldb}{\bar{\ad}}
\newcommand{\zb}{\bar{\z}}
\newcommand{\zdb}{\bar{\zd}}

\newcommand{\wvec}{\mathbf{w}}
\newcommand{\losscf}{\loss^\Phi}
\newcommand{\U}{\mathcal{U}}
\newcommand{\Z}{\mathcal{Z}}
\newcommand{\Zl}{\Z_{\loss}}
\newcommand{\Zdyn}{\Z_{\dyn}}

\renewcommand{\O}{\mathcal{O}}
\newcommand{\Ot}{\widetilde{\mathcal{O}}}
\newcommand{\xcov}{\mathbf{x}}
\newcommand{\X}{\mathcal{X}}

\newcommand{\distl}{\mathcal{Q}}
\newcommand{\dista}{\mathcal{P}}

\newcommand{\loss}{\ell}
\newcommand{\losst}{\tilde{\ell}}
\newcommand{\losss}{\loss^\dyn_*}
\newcommand{\fs}{\f^*}
\newcommand{\z}{z}
\newcommand{\zd}{\zeta}
\newcommand{\f}{\pi}

\newcommand{\ld}{q}
\newcommand{\ad}{p}
\newcommand{\st}{x}
\newcommand{\ac}{u}
\newcommand{\dyn}{\Phi}
\newcommand{\T}{T}
\newcommand{\val}{\mathcal{V}}
\newcommand{\Rseq}{\mathfrak{R}^{\sf seq}}
\newcommand{\w}{w}
\newcommand{\Dnoise}{\mathcal{D}_{\w}}
\newcommand{\rad}{\epsilon}
\newcommand{\stable}{\beta}
\newcommand{\stables}{\beta^*}
\newcommand{\regf}{\Omega}
\newcommand{\regp}{\lambda}
\newcommand{\com}{\circ}

\newcommand{\Ergo}{Ergodic\xspace}
\newcommand{\defn}{:\,=}

\newtheorem{theorem}{Theorem}
\newtheorem{lemma}{Lemma}

\newtheorem{corollary}{Corollary}
\newtheorem{proposition}{Proposition}
\newtheorem{remark}{Remark}
\newtheorem{definition}{Definition}
\newtheorem{assumption}{Assumption}

\newcommand{\ERM}{\mathsf{ERM}}
\newcommand{\RERM}{\mathsf{RERM}}
\newcommand{\detm}{\mathsf{det}}
\newcommand{\1}{\ensuremath{{\sf (i)}}}
\newcommand{\2}{\ensuremath{{\sf (ii)}}}
\newcommand{\3}{\ensuremath{{\sf (iii)}}}

\newcommand{\sdim}{d}
\newcommand{\adim}{k}
\newcommand{\lqrQ}{Q}
\newcommand{\Qbnd}{\sigma_q}
\newcommand{\Rbnd}{\sigma_r}
\newcommand{\Bbnd}{\sigma_b}

\newcommand{\lqrA}{A}
\newcommand{\lqrB}{B}

\newcommand{\zbnd}{c_z}
\newcommand{\Ftrack}{\F_{\mathsf{track}}}
\newcommand{\K}{K}
\newcommand{\Kbnd}{c_\K}
\newcommand{\bias}{\eta}
\newcommand{\cbnd}{c_\bias}
\newcommand{\rholq}{\rho}
\newcommand{\Meff}{M}
\newcommand{\stablea}{C_{\mathsf{tar}, 1}}
\newcommand{\stableb}{C_{\mathsf{tar}, 2}}
\newcommand{\valf}{V}
\newcommand{\Flq}{\F_{\mathsf{LQR}}}
\newcommand{\stk}{\kappa}
\newcommand{\stg}{\gamma}
\newcommand{\lqH}{H}
\newcommand{\lqL}{L}
\newcommand{\lqA}{A}
\newcommand{\lqB}{B}
\newcommand{\lqQ}{Q}
\newcommand{\lqR}{R}
\newcommand{\N}{\mathcal{N}}
\newcommand{\W}{W}
\newcommand{\Wbnd}{\sigma_w}
\newcommand{\trw}{\Psi_w}
\newcommand{\lqr}{\mathsf{LQR}}
\newcommand{\cov}{X}
\newcommand{\pert}{\sigma}
\newcommand{\expd}{\text{Exp}}
\newcommand{\lip}{L_{\mathsf{Lip}}}
\newcommand{\stK}{\regp_\K}
\newcommand{\sigl}{\tau_{w}}
\newcommand{\Ha}{\alpha_h}
\newcommand{\Hb}{\beta_h}
\newcommand{\errX}{\Delta_\cov}
\newcommand{\err}{\Delta}
\newcommand{\Kt}{\tilde{\K}_{t+1}}
\newcommand{\tK}{\tilde{\K}}
\newcommand{\covt}{\widetilde{\cov}}
\newcommand{\ssize}{S}
\newcommand{\asize}{A}

\newcommand{\Fmdp}{\F_{\mathsf{MDP}}}
\newcommand{\trans}{P}
\newcommand{\mdpd}{d}
\newcommand{\mdp}{\textsf{MDP}}
\newcommand{\mixt}{\tau}

\newcommand{\hf}{\hat{\f}}
\newcommand{\batch}{\tau}
\newcommand{\Tb}{\hat{\T}}
\newcommand{\Loss}{L}
\newcommand{\Losscf}{\Loss^\dyn}
\newcommand{\Kst}{K^*}
\newcommand{\nlsf}{\mathsf{NL}}
\newcommand{\signl}{\sigma_\nlsf}
\newcommand{\Fnl}{\F_{\nlsf}}
\newcommand{\lipx}{L_{l, \st}}
\newcommand{\liplf}{L_{l, \param}}
\newcommand{\lipf}{L_{\f}}
\newcommand{\xbnd}{c_x}
\newcommand{\lipnl}{L_{\nlsf}}
\newcommand{\tmix}{\mixt}
\newcommand{\param}{\theta}
\newcommand{\FTnl}{\Theta_{\nlsf}}
\newcommand{\pdim}{d_{\param}}
\newcommand{\stp}{\lambda_{\param}}
\newcommand{\pbnd}{c_\param}

\makeatletter
\long\def\@makecaption#1#2{
        \vskip 0.8ex
        \setbox\@tempboxa\hbox{\small {\bf #1:} #2}
        \parindent 1.5em  %% How can we use the global value of this???
        \dimen0=\hsize
        \advance\dimen0 by -3em
        \ifdim \wd\@tempboxa >\dimen0
                \hbox to \hsize{
                        \parindent 0em
                        \hfil
                        \parbox{\dimen0}{\def\baselinestretch{0.96}\small
                                {\bf #1.} #2
                                }
                        \hfil}
        \else \hbox to \hsize{\hfil \box\@tempboxa \hfil}
        \fi
        }
\makeatother

\newcommand{\ff}{f}
\newcommand{\ball}{\mathbb{B}}
\newcommand{\sig}{\sigma}
\newcommand{\lin}{\textsf{Lin}}

\begin{document}

\begin{center}
{\bf {\LARGE{Online learning with dynamics: \\A minimax perspective}}}
\vspace*{.2in}

\large{
\begin{tabular}{cc}
Kush Bhatia$^\dagger$ & Karthik Sridharan$^\ddagger$
\end{tabular}
}
\vspace*{.2in}

\begin{tabular}{c}
Department of Electrical Engineering and Computer Sciences, UC
Berkeley$^\dagger$ \\
Department of Computer Science, Cornell University$^\ddagger$
\end{tabular}
\vspace*{.2in}

\today
\vspace*{.2in}

\end{center}

%!TEX root = odysys.tex
\begin{abstract}
We study the problem of {\em online learning with dynamics}, where a learner interacts with a stateful environment over multiple rounds. In each round of the interaction, the learner selects a  policy to deploy and incurs a cost that depends on both the chosen policy and current state of the world. The state-evolution dynamics and the costs are allowed to be time-varying, in a possibly adversarial way. In this setting, we study the problem of minimizing policy regret and provide non-constructive upper bounds on the minimax rate for the problem.

%We consider the problem of {\em online learning with dynamics}, where a learner interacts with a stateful environment over multiple rounds. In each round of the interaction, the learner selects a  policy to deploy and incurs a cost that depends on both the chosen policy and current state of the world. The dynamics, which affect the state evolution, and the learner costs are allowed to be time-varying, in a possibly adversarial way. In this setting, we study the problem of minimizing policy regret and provide non-constructive upper bounds on the minimax rate for the problem.

%Nature also presents dynamics for that round, and the state for next round evolves according to this dynamics and the choice of policy by learner. Our general setting allows for arbitrary adversarial choice of costs and adversarially chosen and changing dynamics for evolution of states. In this setting, we study the problem of minimizing policy regret in this framework and provide non-constructive upper bounds on the minimax rate for the problem.

Our main results provide sufficient conditions for online learnability for this setup with corresponding rates. The rates are characterized by 1) a complexity term capturing the expressiveness of the underlying policy class under the dynamics of state change, and  2) a dynamics stability term measuring the deviation of the instantaneous loss from a certain counterfactual loss. Further, we provide matching lower bounds which show that both the complexity terms are indeed necessary.

Our approach provides a unifying analysis that recovers regret bounds for several well studied problems including online learning with memory, online control of linear quadratic regulators, online Markov decision processes, and tracking adversarial targets. In addition, we show how our tools help obtain tight regret bounds for a new problems (with non-linear dynamics and non-convex losses) for which such bounds were not known prior to our work.

\end{abstract}

%!TEX root = odysys.tex

\section{Introduction}\label{sec:intro}
Machine learning systems deployed in the real-world interact with people through their decision making. Such systems form a feedback loop with their environment: they learn to make decisions from real-world data and decisions made by these systems in turn affect the data that is collected. In addition, people often learn to adapt to such automated decision makers in an attempt to maximize their own utility rendering any assumption on the data generation process futile. Motivated by these aspects of decision making, we propose the problem of \emph{online learning with dynamics} which involves repeated interaction between a learner and an environment with an underlying state. The decisions made by the learner affect this state of the environment which evolves as a dynamical system. Further, we place no distributional assumptions on  the learning data and allow this to be adversarial.

Given such a setup, a natural question to ask is how does one measure the performance of the learner? Classical online learning studies one such notion of performance known as regret. This measure compares the performance of the learner to that of a fixed best policy in hindsight, when evaluated on the \emph{same} states which were observed by the learner. Such a measure of performance clearly does not work for the above setup: if we would have deployed a different policy, we would have observed different states of the environment. To overcome this, we study a counterfactual notion of regret, called \emph{Policy Regret}, where the comparator term is the performance of a policy on the states one would have observed if this policy was deployed from the beginning of time.

Such a notion of regret has been studied in the online learning literature for understanding memory based adversaries~\cite{merhav2002, anava2015, arora2012} and more recently, for the study of specific reinforcement learning models~\cite{even2009,abbasi2014,cohen2018}. However, a vast majority of these works have focused on known and fixed models of state evolution, often restricting the scope to linear dynamical systems. Further, these works have focused on simplistic policy classes as the comparators in their notion of policy regret. Contrast this with the vast literature on statistical learning~\cite{vapnik1971, bartlett2002} and classical online learning~\cite{rakhlin2010} which study the question of learnability in full generality; for arbitrary losses and general function classes.

Our work is a step towards addressing this gap. We study the problem of learnability for a class of online learning problems with underlying states evolving as a dynamical system in its full generality. \kbdelete{In particular, we consider the setup arbitrary costs at each round of interaction as well as an arbitrary sequence of dynamics functions, both of which possibly chosen adversarially and revealed to the learner on the fly.} Our main results provide sufficient conditions (along with non-asymptotic upper bounds) on when such problems are learnable, that is, can have vanishing policy regret. Our approach is non-constructive and provides a complexity term that provides upper bounds on the minimax rates for these problems. Further, we provide lower bounds showing that for a large class of problems, our upper bounds are tight up to constant factors. By studying the problem in full generality, we show how several well-studied problems in the literature comprising
online Markov decision processes~\cite{even2009}, online adversarial tracking~\cite{abbasi2014}, online linear quadratic regulator~\cite{cohen2018}, online control with adversarial noise~\cite{agarwal2019}, and online learning with memory~\cite{arora2012,anava2015} can be seen as specific examples of our general framework. We recover the best known rates for a majority of these problems, often times even generalizing these setups. We also provide examples where, to the best of our knowledge, previous techniques are not able to obtain useful bounds on regret; however using our minimax tools, we are able to provide tight bounds on the policy regret for these examples.

\kbdelete{Machine learning systems, making decisions in the real world often affect the state of the world they are deployed in. Reinforcement learning, linear-Quadratic regulators are all example of state-full systems where actions of learner or the controller, affect the evolution of states of the system. Classical online learning framework makes no assumptions about how instances are produced and aim to minimize the so called regret against the best policy from a class of policies chosen in hindsight. However, this framework overlooks the fact that the learning algorithm affects the state of the world its deployed in. If the benchmark policy we compare regret w.r.t. were actually to be deployed,  the state of the world would have evolved differently. Hence regret needs to be measured against benchmark policy according to the state evolution corresponding to that policy. Such notion of regret has been considered in past literature under the name policy regret. However most of the existing literature in this line of work typically assume a fixed, known, simple dynamics for states evolution. In fact, a vast majority of existing literature assume the dynamics to be linear. Further,  results are typically specific to simple class of policies like linear class of policies against which policy regret is measured.
In this paper we consider a very general set up with arbitrary class of policies, arbitrary costs per round of interaction and arbitrary sequence of state dynamics possibly adversarially chosen and revealed to the learner on the fly. We provide sufficient conditions (along with  upper bounds) on when such problems are learnable (that is can have vanishing policy regret). Our approach is non-constructive and provides a complexity term that provides upper bounds on the minimax rates for these problems. We further provide lower bounds showing that for a large class of problems, our upper bounds are tight to within constant factors. We finally provide concrete examples of previously studied problems including Online Markov decision processes~\cite{even2009}, Online Adversarial Tracking~\cite{abbasi2014}, Online Linear Quadratic Regulator~\cite{cohen2018}, Online Control with Adversarial Noise~\cite{Agarwal2018} and online learning with memory~\cite{arora2012}; for which we recover rates and sometimes generalize the set up. We also provide examples where, to our knowledge previous techniques are not able to obtain useful bounds on regret; however the tools we introduce provide tight bounds on the policy regret.}

Formally, we consider the setup where $\X$ denotes an arbitrary set of states, $\F$ an arbitrary class of policies and $\Z$ an arbitrary instance space. \kbdelete{We assume that the world begins in a state $\st_1 \in \X$ known in advance to the learner.} Given this, the interaction between the learner and nature can be expressed as a $\T$ round protocol where on each round $t \in [T]$, the learner picks a policy $\f_t \in \F$, the adversary simultaneously picks instance $(\z_t, \zd_t) \in \Z$. The learner suffers loss $\loss(\f_t,\st_t, \z_t)$ and the state of the system evolves\footnote{while we consider deterministic dynamics here, Section~\ref{sec:setup} considers general dynamics with stochastic noise} as  $\st_{t+1} \leftarrow \dyn(\st_t, \f_t, \zd_t)$, where $\dyn$ is known to the learner. The goal of the learner is to minimize policy regret
\begin{small}
\begin{equation*}
\polregT= \sum_{t=1}^\T \loss(\f_t, \st_t, \z_t) - \pinf_{\f\in \F} \sum_{t=1}^{\T}\loss(\f, \stf{t}{\f^{(t-1)}, \zd_{1:t-1}}, \z_t)~~,
\end{equation*}
\end{small}
where $\st_t$ are the states of the system based on learners choices of policies and $\stf{t}{\f^{(t-1)}, \zd_{1:t-1}}$ represents the state of the system at time $t$ if the policy $\f$ was used the previous $t-1$ rounds. We refer to the loss {\small$\loss(\f, \stf{t}{\f^{(t-1)}, \zd_{1:t-1}}, \z_t)$} as the \emph{counterfactual loss} of policy~$\f$.  \kbdelete{A problem instance in this online learning with dynamics problem is specified by tuple $(\F, \Z, \dyn, \loss)$.}
Notice that dynamics $\dyn$ being fixed or known in advance to the learner is not really restrictive since an adversary can encode arbitrary state dynamics mapping in $\zd_t$’s and $\dyn$ can just be seen as an applicator of these mapping. \kbdelete{Also note that instance $z_t$ are used for both in losses $\loss$ and in dynamics of state. Making the instance space rich enough, one can capture arbitrary dynamics and losses. We also assume that at the end of every round $t$, the learner gets to observe instances $(\z_t, \zd_t)$. In this section to keep things simple we present the state dynamics as described above, but in the next section  we present an extra option of adding known stochastic noise into the dynamics which we will formally introduce in that section.} \kbdelete{Our very general setting encompasses various previously studied problems mentioned earlier but more importantly extend to much more general settings like non-linear dynamics, non-convex losses and arbitrary policy classes.}

\paragraph{Our contributions.}We are interested in the following question: for a given problem instance $(\F,\Z,\dyn,\loss)$, is the problem learnable, that is, does there exists a learning algorithm such that policy regret is such that $\polregT = o(T)$.  \kbdelete{To this ends, on similar lines as done in Rakhlin et al.~\cite{rakhlin2015a} for classical online learning framework, our starting point is the so called minimax value of the problem. Using an application of  von Neumann's minimax theorem repeatedly, we see that the minimax value of this learning problem is equal to that of a dual game which we analyze to obtain our final bound.} Below we highlight some of the key contributions of this paper.
\begin{enumerate}
\item We show that the minimax policy regret for any problem specified by $(\F,\Z,\dyn,\loss)$ can be upper bounded by sum of two terms: i) a sequential Rademacher complexity like term for the class of counterfactual losses of the policy class, and ii) a term we refer to as dynamic stability term for the Empirical Risk Minimizer (ERM) (or regularized ERM) algorithm.
\item We analyze the problem in the dual game. While in most cases ERM does not even have low classical regret let alone policy regret, we show that ERM like strategy in the dual game can lead to the two term decomposition of minimax policy regret we mention above.
\item Ours is the first work that studies arbitrary online dynamical systems, and provides an analysis for general policy classes and loss functions (possibly non-convex).
\item We provide lower bounds that show that our sufficient conditions are tight for a large class of problem instances showing that both the terms in our upper bounds are indeed necessary.
\item We delineate a number of previously studied problems including online linear quadratic regulator, and online learning with memory for which we recover rates. More importantly, we provide examples of new non-convex and general online learning with dynamics problems and obtain tight regret bounds. For these examples, none of the previous methods are able to obtain any non-degenerate regret bounds.
\end{enumerate}

%
%
%1. Online adversarial Control
%2. Adversarial tracking
%3. H-infinity control
%4. Online Learning and competing with strategies
%5. Online Learning with memory
%6. Reinforcement learning

%!TEX root = odysys.tex
\section{Related work}\label{sec:rw}
\paragraph{Online learning and Sequential Complexities.} The classical online learning setup~\cite{cesa2006} considers a repeated interactive game between a learner and an environment without any notion of underlying dynamics. \kbdelete{It can be seen as a special case of our setup where the loss function is independent of the underlying states.}
Sequential complexity measures were introduced in~\cite{rakhlin2010} to get tight characterization of minimax regret rates for the classical online learning setting. They showed that for the class of online supervised learning problems, one can upper and lower bound minimax rate in terms of a sequential Rademacher complexity of the predictor class. The works \cite{littlestone1988,ben2009} provided an analog of VC theory for online classification and the sequential complexity measures in work \cite{rakhlin2010} provided such a theory for general supervised online learning. This paper can be seen as deriving such characterization of learnability and tight rates for the problem of online learning with dynamics. \kbdelete{Indeed, in the case when the loss function is independent of the system state, our main result recovers the classical sequential Rademacher complexity as special case.}In the more general setting we consider, while the main mathematical tools introduced in \cite{rakhlin2010} are useful, they are not by themselves sufficient because of the complexities of policy regret and the state dynamics. This is evident from our upper bound which consists of two terms (both of which we show are necessary) and only one of them is a sequential Rademacher complexity type term.

\paragraph{Optimal Control.} Another line of work closes related ours is that on the theory of optimal control (see~\cite{kirk2004} for a review). \kbdelete{Learning and control problems involving states that evolve based on learner or controller's action introduce new complexities to the learning problem.} Linear dynamical systems with simple zero mean noise models like Gaussian noise for state dynamics have been extensively studied (see the surveys \cite{ljung1999} and \cite{hardt2018} for an extensive review). While majority of the work in control have focused on linear dynamics with fixed noise models, $H_\infty$ control (and more generally robust control) literature has aimed at extending the setting to worst case perturbations (see \cite{stengel1994}). However these works focus on cumulative costs and are often not practical for machine learning scenarios where such algorithms tend to be overly conservative.

\paragraph{Online Control and Linear Dynamics.}
There has been recent work dealing with adversarial costs and linear dynamics with either stochastic or adversarial noise. Online Markov decision processes~\cite{even2009}, Online Adversarial Tracking~\cite{abbasi2014} and Online Linear Quadratic Regulator~\cite{cohen2018, simchowitz2020} are all examples of such work that deal with specific form of possibly adversarially chosen cost functions, albeit the loss functions in these problems are very specific and the dynamics are basically linear with either fixed stochastic noise or no noise. Perhaps the closest comparison to our work is the work by Agarwal et. al~\cite{agarwal2019}
 (and also \cite{agarwal2019b, foster2020}) where adversarial but convex costs, linear policies and linear dynamics with an adversarial component are considered. In contrast, we consider arbitrary class of policies, both adversarially chosen costs (possibly non-convex) and dynamics that are presented on the fly and arbitrary state space. Indeed, in Section~\ref{sec:examples} we work out how our analysis recovers some of the above mentioned results.

%!TEX root = odysys.tex
\section{Online learning with dynamics}\label{sec:setup}
We now formally define the online learning with dynamics problem. We let $\X$ represent the state space, $\F$ denote the set of learner polices, and $\Z = \Zl\times \Zdyn$ denote the space of adversary's moves. \kscomment{I am wondering if we even need to define action space and policy as mapping from state space to action space. Seems like we never need it anywhere in formal set up and in the examples we anyway only need to finally define $\F$.}\kbcomment{Resolved.}

\subsection{Problem setup}
The problem of online learning with dynamics proceeds as a repeated game between a learner and an adversary\kbdelete{\footnote{we interchangeably use the terms adversary and environment}} played over $\T$ rounds. The state of the system at time $t$, denoted by $\st_t \in \X$, evolves according to a stochastic dynamical system as $\st_{t+1} = \dyn(\st_t, \f_t, \zd_t) + \w_t$, where $\dyn: \X \times \F \times \Zdyn \to \X$ is the transition function and $\w_t \sim \Dnoise$ is a zero-mean additive noise. The transition function $\Phi$ is allowed to depend on adversary's action $\zd_t$ allowing the dynamics to change across time steps. We assume that the dynamics function $\Phi$ and distribution $\Dnoise$ are fixed apriori and are known to the learner before the game begins.

Given these dynamics, the repeated online game between the learner and the adversary starts at an initial state $\st_1$ proceeds via the following interactive protocol:\\\vspace{-2mm}

\noindent On round $t = 1, \ldots, \T,$\vspace{-2mm}
\begin{itemize}
  %\item the adversary reveals a context $\zs_t \in \Zs_t$ to the learner\vspace{-1mm}
  \item the learner picks policy $\f_t \in \F$, adversary simultaneously selects instance  $(\z_t, \zd_t) \in \Z$ \vspace{-1mm}
  %\item the adversary reveals $\z_t \in \Z$ \vspace{-1mm}
  \item the learner receives payoff (loss) signal $\loss(\f_t, \st_t, \z_t)$ \vspace{-1mm}
  \item the state of the system transitions to $\st_{t+1} = \dyn(\st_t, \f_t, \zd_t) + \w_t$
\end{itemize}
We consider the \emph{full information} version of the above game: the learner gets to observe the instances $(\z_t, \zd_t)$ at time $t$. The objective of the learner is to minimize, in expectation, the \emph{policy regret}
\begin{small}
\begin{equation}\label{eq:pol_reg}
\polregT = \sum_{t=1}^\T \Eu{w}[\loss(\f_t, \stf{t}{\f_{1:t-1}, \w_{1:t-1}, \zd_{1:t-1}}, \z_t)] - \pinf_{\f\in \F} \Eu{w}\left[\sum_{t=1}^{\T}\loss(\f, \stf{t}{\f^{(t-1)}, \w_{1:t-1}, \zd_{1:t-1}}, \z_t)\right]
\end{equation}
\end{small}
with respect to a policy class $\F$ and dynamics model $\dyn$. In the above definition, the notation $\stf{t}{\f_{1:t-1}, \w_{1:t-1}, \zd_{1:t-1}}$ makes the dependence of the state $\st_t$ explicit on the previous policies, noise and adversarial actions. For notational convenience, we will often drop dependencies on the noise variables $\w$ and adversarial actions $\zd$ when it is clear from context.

Observe that in the definition of policy regret, the loss depends on the state of the system at that instance which can be potentially very different for the learner and a comparator policy $\f$. This lends an additional source of complexity to the interactive game, and can make the problem \emph{much} harder than its counterpart without a dynamics.
It is worth highlighting that the policy regret defined above is akin to the notion of pseudo-regret in online learning: the infimum with respect to the comparator is taken with respect to the expected cost. The alternative, with the infimum and the expectation swapped, can be in general hard to deal with because of failure of uniform laws for general stationary ergodic processes~\citep{nobel1995}.

\ksdelete{We now turn to describe the adversary. There are two kinds of adversaries, \emph{oblivious} and \emph{adaptive}, which have been studied in the online learning literature. Primarily, these differ in the knowledge available to them while selecting their moves $\z_t$. An adversary is said to be adaptive if these moves are allowed to adapt to the randomness of the previous timesteps and oblivious otherwise. As we show in section~\ref{}, an adaptive adversary is quite strong for our online learning with dynamics problem and can force any learner to incur $O(T)$ regret. Going forward, we focus on the problem with an oblivious adversary -- one which selects its actions $\z_t$ before the game begins.\vspace{-2mm}}

\kbdelete{\paragraph{Special Case: Online learning.}}The problem of online learning with dynamics generalizes the online learning problem where the loss functions $\loss(\f, \st, \z) = \losst(\f, \z)$ are independent of the underlying state variables. Indeed, our notion of policy regret in equation~\eqref{eq:pol_reg} reduces to the notion of external regret studied in the online learning literature. Also, the problem of online learning with memory involves adversaries which have bounded memory of length $m$ and thus the loss incurred by the learner at any time is a function of its past $m$ moves. By setting the state variable $\st_t = [\f_{t-m}, \ldots, \f_{t-1}]$, the dynamics function $\dyn(\st_t, \f_t) = [\f_{t-m+1}, \ldots, \f_t]$, and the noise disturbances $\w_t = 0$, we can see that the bounded memory adversaries can be seen as a special case of our problem with dynamics.

\subsection{Minimax Policy Regret}
Given the setup of the previous section, we study the online learning with dynamics game between the learner and the adversary through a minimax perspective. \ksdelete{ We represent any learner learner by an algorithm $\alg$ mapping sequences of past adversarial actions to the space of policies as \mbox{$\alg: \bigcup_{t=1}^{\infty} \Z^t \mapsto \Delta(\F)$}. We then define the the minimax value of the game as
\begin{equation}\label{eq:valgame}
  \val_\T(\F, \Z, \dyn, \loss) \defn \pinf_{\alg} \sup_{\z_{1:\T}} \En\left[\sum_{t=1}^\T \Eu{w}[\loss(\f_t, \stf{t}{\f_{1:t-1}, w}, \z_t)] - \pinf_{\f\in \F} \Eu{w}\left[\sum_{t=1}^{\T}\loss(\f, \stf{t}{\f^{(t-1)}, w}, \z_t)\right]\right],
\end{equation}
where the outer expectation is over the randomness present in the algorithm $\alg$. Observe that this is the minimax value of the policy regret for a class of policies $\F$, adversarial choices $\Z$, and dynamics function $\dyn$. The value encodes the interactive protocol of the previous section with an oblivious adversary which cannot adapt to the random bits of both the learner and the stochastic dynamics.}
Studying this minimax value allows one to understand the limits of learnability for a tuple $(\F, \Z, \dyn, \loss)$: upper bounds on this value imply existence of algorithms with corresponding rates while lower bounds on this values represent the information-theoretic limits of learnability.  \ksdelete{Formally, the problem of online learnability with dynamics can now be reduced to studying this value as implied by the following definition.
\begin{definition}\label{def:learnability}
A policy class $\F$ is said to be \emph{online learnable} with respect to a given dynamics function $\dyn$  and noise model $w$ if,
\begin{equation*}
  \limsup_{\T \to \infty} \frac{\val_\T(\F, \Z, \dyn)}{T} = 0\;.
\end{equation*}
\end{definition}}
In the following lemma, we formally define the value $\val_T(\F, \Z, \dyn, \loss)$ of the minimax policy regret for a given problem which informally is the policy regret of best learning algorithm against the worst case adversary. \kbdelete{We also show that this value is equal to the value of a corresponding \emph{Dual Game} -- one in which the adversary reveals to the learner a distribution over its instances $(\z_t, \zd_t)$ and the learner responds by selecting a policy $\f_t \in \F$.}

\begin{proposition}[Value $\to$ Dual Game]\label{prop:dual_game}
Let $\distl$ and $\dista$ denote the sets of probability distributions over the policy class $\F$ and the adversarial actions $\Z$ respectively, satisfying the necessary conditions for the minimax theorem to hold. Then, we have that\footnote{$\multiminimax{\ldots}_{t=1}^{\T}$ denotes interleaved application of the sequence of operator inside. For example, for $T=2$,  \mbox{$\multiminimax{\sup_{p_t} \inf_{q_t}}_{t=1}^{2}[\cdot] = \sup_{p_1} \inf_{q_1} \sup_{p_2} \inf_{q_2}[\cdot] $} }
\begin{small}
\begin{equation}\label{eq:val_game}
\val_\T(\F, \Z, \dyn, \loss)  := \multiminimax{\pinf_{q_t \in \distl} \psup_{(\z_t, \zd_t) \in \Z}\Eu{\f_t \sim q_t}}_{t=1}^{\T}\left[\polregT\right] =   \multiminimax{\psup_{\ad_t \in \dista}\pinf_{\f_t}\Eu{(\z_t, \zd_t) \sim \ad_t}}_{t=1}^{\T}\left[\polregT\right].
%\ksdelete{\begin{align*}
%  \val_\T(\F, \Z, \dyn, \loss) & := \multiminimax{\pinf_{q_t \in \distl} \psup_{\z_t \in \Z}\Eu{\f_t \sim q_t}}_{t=1}^{\T}\left[\Eu{\{w_t\}}\left[\sum_{t=1}^\T \loss(\f_t, \stf{t}{\f_{1:t-1}}, \z_t)\right] - \pinf_{\f\in \F} \Eu{\{w_t\}}\left[\sum_{t=1}^{\T}\loss(\f, \stf{t}{\f^{(t-1)}}, \z_t)\right]\right]\\
%&=   \multiminimax{\psup_{\ad_t \in \dista}\pinf_{\f_t}\Eu{\z_t \sim \ad_t}}_{t=1}^{\T}\left[\Eu{\{w_t\}}\left[\sum_{t=1}^\T \loss(\f_t, \stf{t}{\f_{1:t-1}}, \z_t)\right] - \pinf_{\f\in \F} \Eu{\{w_t\}}\left[\sum_{t=1}^{\T}\loss(\f, \stf{t}{\f^{(t-1)}}, \z_t)\right]\right].}
\end{equation}
\end{small}
\end{proposition}
The proof of the proposition is deferred to Appendix~\ref{app:setup}. The proof proceeds via a repeated application of von Neumann's minimax theorem (for instance see~\cite[Appendix A]{rakhlin2015a}). Notice that the minimax theorem changes the order of the online sequential game defined in the setup above: at every time step $t$, the adversary proceeds first and outputs a distribution $\ad_t$ over instances and the learner responds back with $\f_t$ \emph{after} having observed the distribution. The actual loss instance $(\z_t, \zd_t)$ is then sampled from the revealed distribution $\ad_t$. On the other hand, the comparator remains the same as before: the best policy $\f \in \F$ in hindsight. This reversed game, termed the \emph{Dual Game}, forms the basis of our analysis and allows us to study the complexity of the online learning with dynamics problem.

\section{Upper bounds on value of the game}\label{sec:main}
Our main result in this section concerns an upper bound on the value of the sequential game $\val_\T(\F, \Z, \dyn, \loss)$ relating it to the study of certain stability properties of empirical minimizers and stochastic processes associated with them. Before we proceed to describe the main result, we revisit some preliminaries and setup notation which would be helpful in describing the main result.\nocite{suggala2019, hazan2016, shalev2012}

\kbdelete{\subsection{Sequential Rademacher Complexity}}
\paragraph{Sequential Rademacher complexity.}The notion of Sequential Rademacher Complexity, introduced in~\cite{rakhlin2010},  is a natural generalization of the Rademacher complexity for online learning. However, observe that the loss of the comparator term in the definition of policy regret in equation~\eqref{eq:pol_reg} depends on the adversarial actions $\zd_{1:t-1}$ through the dynamics and $\z_t$ through the loss function $\loss$. We define the following version of sequential Rademacher complexity for such dynamics based losses.
\begin{definition}
  The Sequential Rademacher Complexity of a policy class $\F$ with respect to loss function $\loss:\F\times \X \times \Zl \mapsto \real$ and dynamics $\dyn: \X \times \F \times \Zdyn \to \X$ is defined as
  \begin{equation*}
    \Rseq_\T(\loss\com\F) \defn \sup_{(\mathbf{z}, \bm{\zd})}\En_{\rad}\left[\sup_{\f \in \F} \sum_{t=1}^\T \rad_t \loss(\f, \stf{t}{\bm{\zd}_1(\rad), \ldots, \bm{\zd}_{t-1}(\rad)}, \mathbf{\z}_t(\rad)) \right]\;,
  \end{equation*}
  where the outer supremum is taken over $\Z = \Zl\times\Zdyn$-valued trees
  \footnote{A $\Z$-valued tree $\mathbf{\z}$ of depth $d$ is defined as a sequence
$(\mathbf{\z}_1, \ldots, \mathbf{\z}_d)$ of mappings $\mathbf{\z}_t:\{\pm 1\}^{t-1} \mapsto \Z $ (see~\cite{rakhlin2014})}
 of depth $\T$  and $\rad = (\rad_1, \ldots, \rad_\T)$ is a sequence of i.i.d. Rademacher random variables.
\end{definition}
A similar definition was also used by Han et al.~\cite{han2013} in the context of online learning with strategies where the notion of regret was defined w.r.t. a set of strategies rather than a fixed action.  As compared with the classical online learning problem, the above comprises problems where the loss at time $t$ depends on the complete history $(\zd_1, \ldots, \zd_{t-1})$ of the adversarial choices along with $\z_t$. As noted by~\cite{han2013}, such dependence on the the adversary's history can often make the online learning problem harder to learn compared with the online learning problem.

\kbdelete{For the online learning problem, Rakhlin et al.~\cite{rakhlin2010} established that such a  notion of complexity provides a way to study the online learnability of a class of functions. The proof for this goes via a symmetrization argument for non-i.i.d. sequences and crucially relies on the fact that the loss function at time $t$, say $\loss_t(\cdot) \defn \loss(\cdot, \z_t)$ is the same for both the learner and the comparator. This is not the case in the online learning with dynamics problem: the loss at time $t$ depends on the corresponding state of the learner and the comparator. Since these states can be very different, the symmetrization argument no longer works for our problem setup. Our main results work by defining a suitable notion of \emph{counterfactual loss} which allows us to side-steps this difficulty.}

\kbdelete{\subsection{Empirical Risk Minimization}}
\vspace{-2mm}\paragraph{Empirical Risk Minimization (ERM).} Given a sequence of loss functions ${\loss}_t : \Fgen \mapsto \real$ for $t \in [T]$, the ERM with respect to a function class $\Fgen$ is defined to be the minimizer of the cumulative loss with $f_{\ERM, \T} \in \argmin_{f \in \Fgen}\sum_{t=1}^\T {\loss}_t(f)$.
In the statistical learning setup, the problems of supervised classification and regression are known to be learnable with respect to a function class $\Fgen$ \emph{if and only if} the empirical risks uniformly converge over this class $\Fgen$ to the population risks. In contrast, our results provide \emph{sufficient} conditions for learnability in terms of certain stability properties of such empirical risk minimizers.

%Recall that Lemma~\ref{lem:dual_game} gives us a way to study the value of the game through a dual game. Our upper bounds will depend on certain stability properties of a sequence of such empirical risk minimizers in the dual game.

\vspace{-2mm}\paragraph{Dynamic stability.} We introduce the notion of dynamic stability which captures the stability of an algorithm's interaction with the underlying dynamics $\dyn$. In order to do so, we define a notion of counterfactual loss $\losscf_t$ of a policy $\f$ as the loss incurred by a learner which selects $\f$ for time $1:t$. \kbcomment{Recall that the state reached by repeating a policy $\f$ up to time $t$ is denoted by $\stf{t}{\f^{(t-1)}, \w_{1:t-1}, \zd_{1:t-1}}$.}

\begin{definition}[Counterfactual Losses]\label{def:loss_cf}
Given a sequence of adversarial actions $\zd_{1:t-1}, \z_t$,  dynamics function $\dyn$, and noise distribution $\Dnoise$, the counterfactual loss of a policy $\f$ at time $t$ is
\begin{equation*}
  \losscf_t(\f, \zd_{1:t-1}, \z_t) \defn \Eu{\w_s \sim \Dnoise}\left[\loss(\f, \stf{t}{\f^{(t-1)}, \w_{1:t-1}, \zd_{1:t-1}}, \z_t)\right]\;.
\end{equation*}
\end{definition}
With this definition, observe that the comparator term in the value $\val_{\T}$ in equation~\eqref{eq:val_game} is in fact a cumulative sum of counterfactual losses for a policy $\f$. Any algorithm $\alg$ that plays a sequence of policies $\{\f_t\}$ in the online game incurs an instantaneous loss $\loss(\f_t, \stf{t}{\f_{1:t-1}, \zd_{1:t-1}}, \z_t)$ at time $t$.
In comparison, the counterfactual loss $\losscf(\f_t, \zd_{1:t-1}, \z_t)$ represents a scenario where the algorithm commits to the policy $\f_t$ from the beginning of the game. Our notion of dynamic stability of an algorithm is precisely the deviation between these two types of losses: instantaneous and counterfactual.\kbdelete{ We make this precise in the following definition.}

\begin{definition}[Dynamic Stability]\label{def:dyn_stab}
An algorithm $\alg$ is said to be $\{\stable_t\}$-dynamically stable if for all sequences of adversarial actions $[(\z_1, \zd_1), \ldots, (\z_T, \zd_T)]$ and time instances $t\in [T]$
\begin{small}
\begin{equation*}
  \left|\En_{\w_{1:t-1}}[\loss(\f_t, \stf{t}{\f_{1:t-1}, \w_{1:t-1}, \zd_{1:t-1}}, \z_t)] - \losscf(\f_t, \zd_{1:t-1},\z_t) \right| \leq \stable_t \quad \text{where} \quad \f_t = \alg((\z_{1:t-1}, \zd_{1:t-1})).
\end{equation*}
\end{small}
\end{definition}
It is interesting to note that if that loss functions are independent of the underlying states, that is $\loss(\f, \st, \z) = \losst(\f, \z)$, then \emph{any} algorithm is dynamically stable in a trivial manner with the stability parameters $\stable_t = 0$ for all time instances $t$. \kbcomment{add a couple of lines on how these differ from notions of stability in control theory and learning theory.}

\kbcomment{\begin{remark}\label{rem:primal_upper}
  Given this definition of dynamic stability, we can see that for any algorithm $\alg$, the policy regret in equation~\eqref{eq:pol_reg} can be upper bounded as
  \begin{equation*}
    \polregT(\alg) \leq  \sum_{t=1}^\T \stable_{\alg, t} + [\sum_{t=1}^T \losscf_t(\f_t, \zd_{1:t-1}, \z_t) - \inf_{\f \in \F}\sum_{t=1}^\T\losscf(\f, \zd_{1:t-1}, \z_t)]\;.
  \end{equation*}
  The above bound brings out the trade-off from an algorithmic perspective: on one hand, the algorithm would like to obtain a low regret with respect to losses $\losscf$ and at the same time ensure that the stability terms are small. While the relaxation based algorithms in Han et al.~\cite{han2013} indeed minimize the counterfactual regret, however it is not clear if these are dynamically stable with non-trivial rates.\hfill$\clubsuit$
\end{remark}}

With these definitions, we now proceed to describe our main result. Recall that Proposition~\ref{prop:dual_game} translates the problem of studying the value of the game $\val_\T(\F, \Z, \dyn, \loss)$ to that of studying the policy regret in a dual game. In this dual game, the learner has access to the set of adversaries distribution $\{\ad_s \}_{s=1}^t$ at time $t$ and the policy $\f_t$ can be a function of these. For a regularization function $\regf:\F \mapsto \real_+$, we denote the regularized ERMs with respect to function class $\F$ and counterfactual losses~$\losscf$ by
\begin{equation}\label{eq:reg_erm}
\f_{\RERM, t} \in \argmin_{\f\in \F} \sum_{s=1}^t \Eu{ \z_s}\left[ \losscf(\f, \zd_{1:s-1}, \z_s)\right] + \regp\cdot \regf(\f)\;,
\end{equation}
 where $\regp \geq 0$ is the regularization parameter. The following theorem provides an upper bound on the value $\val_\T$ in terms of the dynamic stability parameters of the regularized ERMs above as well a sequential Rademacher complexity of the \emph{effective} loss class $\losscf\com\F \defn \{\losscf(\f, \cdot)\;:\;\f \in \F \}$.

 \begin{theorem}[Upper bound on value]\label{thm:main}
   For any online learning with dynamics instance $(\F, \Z, \dyn, \loss)$, consider the set of regularized ERMs given by eq.~\eqref{eq:reg_erm} with regularization function $\regf$ and parameter $\regp \geq0$ having dynamic stability parameters $\{\stable_{\RERM, t} \}_{t=1}^\T$. Then, we have that the value of the game
   \begin{equation}\label{eq:main_upper}
\val_{\T}(\F, \Z, \dyn, \loss) \leq \sum_{t=1}^\T \stable_{\RERM, t} + 2\Rseq_{T}(\losscf\com\F) + 2\regp\cdot\sup_{\f\in \F}\regf(\f)\;.
   \end{equation}
 \end{theorem}

 The complete proof of the above theorem can be found in Appendix~\ref{app:main}.
 A few comments on Theorem~\ref{thm:main} are in order. The theorem provides sufficient conditions to ensure learnability of the online learning with dynamics problem. In particular, the two terms \mbox{Term (I) = $\sum_{t=1}^\T \stable_{\RERM, t}$} and Term (II) =  $\Rseq_{T}(\losscf\com\F)$
\kbdelete{\begin{equation*}
\underbrace{\sum_{t=1}^\T \stable_{\RERM, t}}_{\text{Term (I)}}  \quad \text{and} \quad \underbrace{\vphantom{\sum_{t=1}^\T}\Rseq_{T}(\losscf\com\F)}_{\text{Term (II)}}\;,
\end{equation*}}
contain the main essence of the upper bound. Term (I) concerns the dynamic mixability property of the regularized ERM in the dual game. If there exist \emph{approximate} minimizers (regularized) of the sequence of counterfactual losses within the policy class $\F$ such that $\f_{\RERM, t}$ is uniformly close to $\f_{\RERM, t+1}$ the dynamic stability parameters can be made to be small. Term (II) comprises of the sequential Rademacher complexity of the loss class $\losscf\com\F$ which involves the underlying policy class $\F$ as well as the counterfactual loss $\losscf$. This measure of complexity can be seen as one which corresponds to an  effective online game where the the loss at time $t$ depends on the adversarial actions up to time $t$. Compare this to the instantaneous loss $\loss(\f_t, \stf{t}{\f_{1:t-1}, \zd_{1:t-1}}, \z_t)$ which depends on both the policies as well as the adversarial actions up to time $t$.
Observe that for the classical online learning setup without dynamics, the dynamic stability parameters $\stable_{\RERM, t} \equiv 0$. On setting the value of regularization parameter $\regp = 0$, we recover back the learnability result of Rakhlin et al.~\cite{rakhlin2010}.

We would like to highlight that the complexity-based learnability guarantees of Theorem~\ref{thm:main} are non-constructive in nature. In particular, the theorem says that any non-trivial upper bounds on the stability and sequential complexity terms would guarantee the \emph{existence} of an online learning algorithm with the corresponding policy regret. Our minimax perspective on the problem allows us to study the problem in full generality without making assumptions with respect to the policy class $\F$, adversarial actions $\Z$ and the underlying (possibly adversarial) dynamics $\dyn$, and provide sufficient conditions for learnability. \kbdelete{Contrast this with existing algorithmic approaches in literature which study specific instances of the problem and provide policy regret rates for the same. To exhibit the full potential of our minimax approach, we show an example problem in Section~\ref{sec:examples} for which our result recovers the optimal regret bound for which \emph{no} algorithmic approaches are known to exist.}
\kbdelete{
\begin{remark}
While we statement of Theorem~\ref{thm:main} considers a fixed regularization function $\regf$ for all time instances, our proof in Appendix~\ref{app:main} considers time varying regularizers and proves a more general result. \kbdelete{This affects only the last term in equation~\eqref{eq:main_upper} and changes it to a sum of $\T$ terms, each containing a difference of regularization function.} In particular, let $\regf_t: \F \mapsto \real_+$ denote the regularization function used at time $t$, then the upper bound of equation~\eqref{eq:main_upper} holds with
{\small\begin{align}\label{eq:upper_timevarying}
  \val_{\T}(\F, \Z, \dyn, \loss) &\leq \sum_{t=1}^\T \stable_{\RERM, t} + 2\Rseq_{T}(\loss_{\textsf{eff}}) + \regp\cdot\sup_{\f\in \F} [\regf_\T(\f) -\regf_\T(\f_{\RERM, T})]\nonumber\\
  &\quad + \regp\cdot \sum_{t=1}^{\T-1} \regf_t(\f_{\RERM, t+1}) - \regf_t(\f_{\RERM, t}) \;.
\end{align}}
\end{remark}}

Given the upper bound on the value $\val_\T(\F, \Z, \dyn, \loss)$, one can observe that there is a possible tension between the two complexity terms: while dynamic stability term promotes using policies which are ``similar" across time steps, the regularized complexity term seeks policies which are minimizers of cumulative losses and might vary across time steps. In order to balance similar trade-offs, a natural \emph{Mini-Batching Algorithm} has been proposed in various works on online learning with memory~\cite{arora2012} and online learning with switching costs~\cite{chen2019}. The key idea is that the learner divides the time $\T$ into intervals of length $\tau > 0$ and commits to playing the same strategy over this time period.

Let us denote any such mini-batching  algorithm by $\alg_\tau$ and the corresponding minimax value restricted to this class of algorithms by $\val_{\T, \tau}(\F, \Z, \dyn, \loss)$ where the infimum in equation~\ref{eq:val_game} is taken over all mini-batching algorithms $\alg_\tau$. Similar to the regularized ERM of equation~\eqref{eq:reg_erm}, we define the following mini-batched ERMs:
\begin{equation}\label{eq:mini_erm}
  \f^{{\tau}}_{\ERM, t} = \begin{cases}
\f_{\ERM}(t)\quad &\text{for } t\equiv 0\bmod \tau\\
\f_{\ERM}(\tau\lfloor\frac{t}{\tau} \rfloor) \quad &\text{otherwise}
\end{cases}\;,
\end{equation}
where we have used the notation $\f_{\ERM}(t) \defn \f_{\ERM, t}$. In the following proposition, we prove an upper bound analogous to that of Theorem~\ref{thm:main} for this class of mini-batching algorithms\footnote{For this class of mini-batching algorithms, we consider an oblivious adversary which cannot adapt to the randomness of the learner.}.

\begin{proposition}[Mini-batching algorithms.]\label{prop:upper_mini}
  For any online learning with dynamics game $(\F, \Z, \dyn, \loss)$, consider the set of mini-batch ERMs given by equation~\eqref{eq:mini_erm} having dynamic stability parameters $\{\stable^\tau_{\ERM, t} \}_{t=1}^\T$. Then, we have that the value of the game
  {\small\begin{equation}\label{eq:mini_upper}
    \val_{\T}(\F, \Z, \dyn, \loss) \leq \inf_{\tau > 0}\val_{\T, \tau}(\F, \Z, \dyn, \loss) \leq \inf_{\tau > 0}\left(\sum_{t=1}^\T \stable^\tau_{\ERM, t} + 2\tau\cdot\sup_{s \in [\tau]}\Rseq_{\T/\tau}(\losscf_s\com\F)\right)\;\;,
  \end{equation}}
  where $\losscf_s$ is the counterfactual loss for the $s^{th}$ batch.
\end{proposition}
We defer the proof of the above proposition to Appendix~\ref{app:main}. In comparison with the upper bound of Theorem~\ref{thm:main}, this bound concerns the dynamic stability of the mini-batched ERMS as compared to their regularized counterparts. Often times, obtaining bounds on the stability parameters $\{\stable^\tau_{\ERM, t} \}_{t=1}^\T$ can be much easier than the ones for regularized ERMS. For instance, it is easy to see that for the problem of online learning with memory with adversaries having memory $m$, one can bound $\sum_{t=1}^\T \stable^\tau_{\ERM, t} = O(\frac{m\T}{\tau})$ whenever the losses are bounded, providing a natural trade-off between the two complexity terms.

\section{Lower bounds on value of the game}\label{sec:lower}
Having established sufficient conditions for the learnability of the online learning with dynamics problem in the previous section, we now turn to address the optimality of these conditions. In particular, we are interested in the question whether both the sequential complexity and dynamic mixability terms are necessary for learnability?
Recall that Theorem~\ref{thm:main} and Proposition~\ref{prop:upper_mini} established upper bounds on the value $\val_\T(\F, \Z, \dyn, \loss)$ for instances of our problem. %in terms the dynamics stability coefficients of certain ERMs and the sequential Rademacher complexity of the loss class $\loss_{\textsf{eff}}$.
The following theorem shows that both the upper bounds of equations~\eqref{eq:main_upper} and~\eqref{eq:mini_upper} are indeed tight upto constant factors.

\begin{theorem}[Lower Bound]\label{thm:lower}
  For the online learning with dynamics problem, there exist problem instances $\{(\F, \Z, \dyn, \loss_i)\}_{i=1}^3$, a regularization function $\regf$ and a universal constant $c>0$ such that
    \begin{subequations}\label{eq:lower}
\begin{align}
  \val_\T(\F, \Z, \dyn, \loss_1) &\geq  c \cdot\Rseq_{T}(\losscf_1\com\F)\label{eq:lower-a}\\
  \val_\T(\F, \Z, \dyn, \loss_2) &\geq  c \cdot \inf_{\regp > 0}\left(\sum_{t=1}^\T\stable_{\RERM, t} + \regp\cdot\sup_{\f\in \F}\regf(\f) \right)\label{eq:lower-b} \\
  \val_\T(\F, \Z, \dyn, \loss_3) &\geq  c \cdot \inf_{\tau > 0}\left(\sum_{t=1}^\T \stable^\tau_{\ERM, t} + 2\tau\Rseq_{\T/\tau}(\losscf_3\com\F)\right)\;,\label{eq:lower-c}
\end{align}
\end{subequations}
where $\stable_{\RERM, t}$ and $\stable_{\ERM, t}^\tau$ are the dynamic mixability parameters of the regularized ERM w.r.t. $\loss_2$ (eq.~\eqref{eq:reg_erm}) and mini-batching ERM w.r.t. $\loss_3$ (eq.~\eqref{eq:mini_erm}) respectively.
\end{theorem}
A few comments on Theorem~\ref{thm:lower} are in order. The theorem exhibits that the sufficiency conditions from Theorem~\ref{thm:main} and Proposition~\ref{prop:upper_mini} are indeed necessary by exhibiting instances whose value is lower bounded by these terms. In particular, equation~\eqref{eq:lower-a} shows that the sequential Rademacher term is necessary, \eqref{eq:lower-b} establishes necessity for the dynamic stability of the regularized ERM, while~\eqref{eq:lower-c} shows that the mini-batching upper bound is also tight. It is worth noting that these lower bounds are not instance dependent but rather construct specific examples to demonstrate the tightness of our upper bound from the previous section. We next present the key idea for the proof of the theorem and defer the complete details to Appendix~\ref{app:lower}.

\paragraph{Proof sketch.} We now describe the example instances which form the crux of the proof for Theorem~\ref{thm:lower}. Consider the online learning with dynamics game between a learner and an adversary with the state space  $\X =\{\st \in \real^d\; | \; \|\st\|_2 \leq 1 \}$ and the set of adversarial actions $\Z_{\loss}^\lin = \{\z \in \real^d\; |\; \|z\|_2 \leq 1\}$. Further, we consider the constant policy class $\F_\lin = \{\f_\ff\; |\; \f(\st) = \ff \text{ for all states } \st \text{ with }\ff \in  \ball_d(1)\}$,
consisting of policies $\f_\ff$ which select the same action $\ff$ at each state $\st$. With a slight abuse of notation, we represent the policy $\f_t$ played by the learner at time by the corresponding $d$-dimensional vector $\ff_t$. Further, we let the dynamics function $\dyn_\lin(\st_t, \ff_t, \zd_t) =  \ff_t$. \kbdelete{Observe that the dynamics simply remembers the last action played by the learner and sets the next state as $\st_{t+1} = \ff_t$ in a deterministic way.} We now define the loss function which consists of two parts, a linear loss and a $L$-Lipschitz loss involving the dynamics:
\begin{equation}\label{eq:lower_loss}
  \loss_L(\ff_t, \st_t, \z_t) = \inner{\ff_t}{\z_t} + \sig(f_t, x_t)\quad \text{where} \quad
  \sig(f_t, x_t) = \begin{cases}
  L \|f_t - x_t\|_2 \quad &\text{for  } \|f_t - x_t\|_2 \leq \frac{1}{L}\\
  1 \quad & \text{otherwise}
  \end{cases}.
\end{equation}
Observe that this example constructs a family of instances one for each value of the Lipschitz constant of $L$ of the function $\sigma$. For this family of instances, we establish that the value
{\small \begin{equation*}
  \val_\T(\F_\lin, \Z, \dyn_\lin, \loss_L) \geq \begin{cases}
 \sqrt{T} \quad &\text{for } 0 < L < 1\\
 \sqrt{LT} \quad &\text{for } 1 \leq  L \leq (4\T)^{\frac{1}{3}}\\
 2^{\frac{1}{3}}\T^{\frac{2}{3}} \quad &\text{for } L > (4\T)^{\frac{1}{3}}
\end{cases}\;.
\end{equation*}}
The proof finally connects these lower bounds to the bounds of Theorem~\ref{thm:main} and Proposition~\ref{prop:upper_mini}.\hfill$\clubsuit$
%Note that for this setup, we have that the counterfactual losses $\losscf(\ff_t, \z_{1:t}; \dyn_\lin) = \inner{\ff_t}{\z_t}$ for time $t > 1$ since the dynamics loss is $0$.\hfill $\clubsuit$

With the lower bounds given in Theorem~\ref{thm:lower}, it is natural to ask whether the sufficient conditions in Theorem~\ref{thm:main} and Proposition~\ref{prop:upper_mini} are indeed necessary for every instance of the online learning with dynamics problem. The answer to this question is unsurprisingly \emph{No} given the generality in which we study this problem. Consider the following simple instance of the problem:
\begin{equation*}
  \F = \X,\quad\loss(\f, \st, \z) = \losst(\f, \z) + \ind[\f = \st], \quad \text{and} \quad \st_{t+1} = \f_t\;,
\end{equation*}
for any non-negative bounded loss $\losst(\f, \z) \in [0,1]$ for all $\f \in \F, \z \in \Zl$. Consider any policy class for which $\Rseq_{T}(\losscf\com\F) > 0$. Both bounds~\eqref{eq:main_upper} and ~\eqref{eq:mini_upper} suggest that the problem is learnable with rate at least $\Rseq_{T}(\losscf\com\F)$. However, observe that the indicator term in the loss is quite severe on the comparator; it ensures that the comparator term is at least $T$. Thus, \emph{any} algorithm which selects a policy from $\F$ at every instance can ensure that the policy regret is at most $0$!
While the above example establishes that the sufficient conditions are not necessary in an instance dependent manner, our next proposition establishes that they are indeed tight for large class of problems instances. %Specifically, it shows that for any online learning problem without dynamics $(\Fgen, \Z, \losst)$, we can construct an instance of the online learning with dynamics problem where the learnability conditions given by Proposition~\ref{prop:upper_mini} are indeed tight.

\begin{proposition}[Instance-dependent lower bound]\label{prop:lower_gen}
a) Given any online learning problem $(\Fgen, \Zl, \loss)$ with a bounded loss function  $\loss:\Fgen\times\Zl \mapsto [-1,1]$, there exists an online learning with dynamics problem \mbox{$(\F_{\Fgen}, \Zl\times\{-1, 1\}, \dyn, \losst)$} and a universal constant $c>0$ such that
{\small\begin{equation*}
  \val_\T(\F_{\Fgen}, \Zl\times\{-1, 1\}, \dyn, \losst) \geq  c \cdot \inf_{\tau > 0}\left(\sum_{t=1}^\T \stable^\tau_{\ERM, t} + 2\tau\Rseq_{\T/\tau}(\losscf\com\F)\right)\;,
\end{equation*}}
where  $\stable_{\ERM, t}^\tau$ are the dynamic mixability parameters of the mini-batching ERM w.r.t. $\loss$ (eq.~\eqref{eq:mini_erm}).\vspace{1mm}\\
b) Given a policy class $\F$ and dynamics function $\dyn$, there exists an online learning with dynamics problem $(\F, \Z, \dyn, \loss)$ and a universal constant $c>0$ such that
{\small\begin{equation*}
  \val_\T(\F, \Z, \dyn, \loss) \geq  c \cdot\Rseq_{T}(\losscf\com\F).
\end{equation*}}
%\begin{equation*}
%  \X = \U = \Fgen, \;\; \F_{\Fgen} \defn \{\f_\ff\; |\; \f(\st) = \ff \text{ for all } \st\in \X\},\;\; \Z = \widetilde{\Z}\times \{-1, +1\}, \;\; \st_{t+1} = \dyn(\st_t, \f_\ff) = \ff\;\;,
%\end{equation*}
\end{proposition}
We defer the proof of the proposition to Appendix~\ref{app:lower}. This proposition can be seen as a strengthening of the lower bounds~\eqref{eq:lower-a} and \eqref{eq:lower-c} showing that for a very large class of problems, the upper bound given by the mini-batching algorithm and the sequential complexity terms are in fact necessary.

\section{Examples}\label{sec:examples}
In this section, we look at specific examples of the online learning with dynamics problem and obtain learnability guarantees for these instances using our upper bounds from Theorem~\ref{thm:main}. \kbdelete{Using our complexity-theoretic tools, we establish learnability for a problems for which no efficient algorithmic approaches are known and also recover regret bounds for several well studied problems in the literature.} For clarity of exposition, our focus in this section on the scaling of the value $\val_\T(\F, \Z, \dyn, \loss)$ with the time horizon $\T$. \kbdelete{and we use the $\Ot$ notation to suppress polynomials factors of problem dependent parameters as well as logarithmic factors of $\T$.} The proofs in Appendix~\ref{app:examples} explicitly detail out all the problem dependent parameters.

\subsection{Online Isotron with dynamics} Single Index Models (SIM) are  class of semi-parametric models widely studied in the econometric and operations research community. Kalai and Sastry~\cite{kalai2009} introduced the Isotron algorithm for learning SIMs and Rakhlin et al~\cite{rakhlin2015a} established that the online version of this problem is learnable.  In this example, we introduce a version of this problem with a state variable that requires a component of the model to vary slowly across time.

We consider a real-valued state space with $\X = \real$. The policy class $\F$ is based on a function class $\Fgen$ consisting of a $1$-Lipschitz function along with a $d+1$ unit dimensional vector and is given as
\begin{gather*}
  \Fgen = \{\ff = (\sig, \wvec = (w_1, w)) \; | \; \sig:[-1, 1] \mapsto [-1, 1]\; 1\text{-Lipschitz},\; \wvec \in \real^{d+1}\; |w_1| \leq 1\; \|w\|_2\leq 1 \},\\
  \F_\Fgen = \{\f_{f}\; | \; \f \in \Fgen, \; \f_\ff(\st) = \ff\; \text{for all } \st \in \X\}.
\end{gather*}
The adversary selects instances in the space $\Z = [-1,1]^{d+1}\times [-1,1]$ and we represent each instance \mbox{$\z = (\z_1, \xcov, y)$}. Given this setup, we now formalize the online learning protocol, starting from initial state $\st_1 = 0$.\\\vspace{-2mm}

\noindent On round $t = 1, \ldots, \T,$\vspace{-1mm}
\begin{itemize}
  \item the learner selects a policy $\f_t \in \F_\Fgen$ and the adversary selects $\z_t \in \Z$\vspace{-1mm}
  \item the learner receives loss $\loss(\f_t, \st_t, \z_t) = (y_t - \sig(\inner{\xcov_t}{w_t}))^2 +(\z_{t,1} - w_{t,1})^2 + (\st_t - w_{t,1})^2$\vspace{-1mm}
  \item the state of the system transitions to $\st_{t+1} = w_{t, 1}$
\end{itemize}
Given this setup, the next corollary provides a bound on the value of this game $\val_{\mathsf{Iso}, \T}(\F_\Fgen, \Z, \dyn, \loss)$.

  %\begin{small}
%\begin{gather}\label{eq:isotron_setup}
%  \X = \real,\; \Fgen = \{\ff = (\sig, w) \; | \; \sig:[-1, 1] \mapsto [-1, 1]\; 1\text{-Lipschitz},\; w \in \real^d\; \|w\|_2\leq 1 \}\;,\nonumber\\
%\F_\Fgen = \{\f_{f}\; | \; \f \in \Fgen, \; \f_\ff(\st) = \ff\; \text{for all } \st \in \X\},\; \Zl = \real^d\times \real,\; \Zdyn = \phi,\;\nonumber\\
%\loss(\f_\ff, \st, \z = (\xcov, y)) = (y - \sig(\inner{\xcov}{w}))^2 +(\st - w_1)^2,\; \dyn(\st, \f_{(\sig, w)}, \zd) = w_1, \; \Dnoise = 0.
%\end{gather}
%\end{small}
%The following corollary establishes the learnability of the online Isotron problem with dynamics. %is learnable at a rate of $\Ot(\sqrt{T})$.
\begin{corollary}\label{cor:ex1}
For the online Isotron problem with dynamics given by $(\F_\Fgen, \Z, \dyn, \loss)$, we have that the minimax value
\begin{align*}\val_{\mathsf{Iso}, \T}(\F_\Fgen, \Z, \dyn, \loss) = \Ot(\sqrt{T}).\end{align*}
\end{corollary}
It is worth recalling that the above game is an dynamical extension of the online Isotron problem instance studied by~\cite{rakhlin2015a}. We are not aware of any primal algorithm which can get a rate of $\sqrt{T}$ for both the online learning version as well the dynamical version of this game. Our non-constructive analysis on the other hand proved a way to guarantee learnability at this rate for the Isotron problem. \hfill$\clubsuit$

\kbdelete{
\paragraph{Example 2: Online Adversarial Tracking~\cite{abbasi2014}.} The problem of online tracking of adversarial targets in Linear Quadratic Regulators was first posed in Abbasi-Yadkori et al.~\cite{abbasi2014}. With our notation, this setup consists of
\begin{small}
\begin{gather}\label{eq:track_setup}
\X = \real^d,\; \Z = \{\z \in \real^d\; |\; \|z_t\|_2\leq 1 \},\; \dyn(\st, \f) = A\st + B\f(\st)\nonumber\;, \Dnoise = 0\\
\Ftrack = \{\f = (\K, \bias) \in \real^{k \times d} \times \real^d\;|\; \|A+BK\|_2 \leq \rholq < 1\;  \|\K\|_2\leq \Kbnd\; \|\bias\|_2\leq \cbnd\}\; \text{ such that } \f(\st) = \K\st + \bias,\nonumber\\
\loss(\f, \st, \z ) = (\st - \z)^\top \lqrQ (\st - \z) + \f(\st)^\top \f(\st) \; \text{ \text{with }\lqrQ \text{ positive definite}}\; .
\end{gather}
\end{small}
The following corollary now establishes the learnability guarantees for this setup.
\begin{corollary}\label{cor:ex2}
For the online adversarial tracking problem given by $(\Ftrack, \Z, \dyn, \loss)$ in equation~\eqref{eq:track_setup}, we have that the minimax value
\begin{equation*}
  \val^{\mathsf{track}}_{\T}(\Ftrack, \Z, \dyn, \loss) = \Ot(\sqrt{T})\;.
\end{equation*}
\end{corollary}
While the above corollary is only able to show a rate scaling as $\sqrt{T}$, Abbasi et al.~\cite{abbasi2014} provided an algorithm which had policy regret scaling as $\O(\log^2\T)$. As mentioned in Section~\ref{sec:conc}, achieving such fast rates within our general framework is an interesting open problem.}

\subsection{Online Markov decision processes} This example considers the  problem of Online Markov Decision Processes (MDPs) studied in Even-Dar et al.~\cite{even2009}.
The setup consists of a finite state space such that $|\X| = \ssize$ and a finite action space with $|\U| = \asize$. %\footnote{$\ssize$ and $\asize$ have been chosen instead of $\sdim$ and $\adim$ to be consistent with~\cite{even2009}}.
The policy class $\F$ consists of all stationary policies, that is,
\begin{equation*}
  \Fmdp = \{\f \; | \; \f: \X \mapsto \Delta(\U)\},
\end{equation*}
where $\Delta(\U)$ represents the set of all probability distributions over the action space. In addition, the transitions are drawn according to a known function $\trans : \X \times \U \mapsto \Delta(\U)$. The sequential game then proceeds as follows, starting from some state $\st_1 \sim \mdpd$:\\\vspace{-2mm}

\noindent On round $t = 1, \ldots, \T,$\vspace{-1mm}
\begin{itemize}
  \item the learner selects a policy $\f_t \in \Fmdp$ and the adversary selects $\z_t \in \Z = [0,1]^{\ssize\times\asize}$\vspace{-1mm}
  \item the learner receives loss $\loss(\f_t, \st_t, \z_t) = \z_t(\st_t, \f_t(\st_t))$\vspace{-1mm}
  \item the state of the system transitions to $\st_{t+1} \sim \trans(\st_t, \ac_t)$
\end{itemize}
For every stationary policy $\f$, we let $\trans^f$ denote the transition function induced by $\f$, that is,
\begin{equation*}
  \trans^f(\st, \st') \defn \sum_{\ac \in \U}\f^\ac(\st)\trans^{\st'}(\st, \ac),
\end{equation*}
where we have used superscript to denote the relevant coordinate of the vector. As in~\cite{even2009}, we make the following mixability assumptions about the underlying MDP.
\begin{assumption}[MDP Unichain]\label{ass:mdp_unichain}
We assume that the underlying MDP given by the transition function $\trans$ is uni-chain. Further, there exists $\mixt \geq 1$ such that for all policies $\f$ and distributions $\mdpd, \mdpd' \in \Delta(\U)$ we have
\begin{equation*}
  \|\mdpd\trans^\f - \mdpd'\trans^f \|_1 \leq e^{-1/\mixt}\|\mdpd - \mdpd' \|_1.
\end{equation*}
\end{assumption}
The parameter $\mixt$ is often referred to as the mixing time of the MDP. Since the MDP is assumed to be uni-chain, every policy $\f$ has a well defined unique stationary distribution $\mdpd_\f$. %with the stationary loss given by $\losss(\f, \z) = \En_{\st \sim \mdpd_\f}\En_{\ac\sim \f(\st)}\z(\st, \ac)$.
Given this setup, we can obtain an upper bound on the value $\val_{\mathsf{MDP}, \T}$ as follows:
\begin{corollary}\label{cor:ex2}
For the online Markov Decision Process sequential game satisfying Assumption~\ref{ass:mdp_unichain}, the minimax value $\val_{\mathsf{MDP}, \T}$ is bounded as
  \begin{equation*}
  \val_{\mathsf{MDP}, \T}(\Fmdp, \Z, \dyn, \loss) =  \O(\sqrt{T}).
\end{equation*}
\end{corollary}
The above corollary helps one recover the same $\O(\sqrt{T})$ regret bound that was obtained by~\cite{even2009}. Note that while the setting studied by~\cite{even2009} consisted of the weaker oblivious adversary, we consider the stronger adaptive adversary which can adapt to the learners strategy.
\kbdelete{
The setup consists of a finite state space  $|\X| = \ssize$, a finite action space  $|\U| = \asize$, and
\begin{small}
\begin{gather}\label{eq:mdp_setup}
\Zl = \{\z \;|\; \z \in [0,1]^{\ssize\times\asize}\},\;  \Zdyn = \phi,\;
\Fmdp = \{\f \; | \; \f: \X \mapsto \Delta(\U)\},\; \loss(\f, \st, \z ) = \z(\st, \f(\st)),\; \nonumber\\
\dyn \text{ given by } \trans : \X \times \U \mapsto \Delta(\U) \text{ with } \st' \sim \trans(\st, \f(\st)).
\end{gather}
\end{small}
With this setup, we now provide a bound on the minimax value $\val^\mathsf{MDP}_\T$ , assuming, as in~\cite{even2009}, that the underlying MDP is uni-chain and satisfies a mixability assumption (see Appendix~\ref{app:examples} for details).
\begin{corollary}\label{cor:ex2}
For the online MDP problem given by $(\Fmdp, \Z, \dyn, \loss)$ in equation~\eqref{eq:mdp_setup}, we have that the minimax value $\val_{\T}(\Fmdp, \Z, \dyn, \loss) = \O(\sqrt{T})$.
\end{corollary}
The above corollary helps one recover the same $\O(\sqrt{T})$ regret bound that was obtained by~\cite{even2009}.\hfill$\clubsuit$
}

\subsection{Online Linear Quadratic Regulator}\label{ex:olqr-main}
The online Linear Quadratic Regulator (LQR) setup studied in this section was first studied in ~\cite{cohen2018}. The setup consists of a LQ system - with linear dynamics and quadratic costs - where the cost functions can be adversarial in nature. The comparator class $\Flq$ comprises a subset of linear policies $\K$ which satisfy the following strong stability property.
\begin{definition}[Strongly Stable Policy]
A policy $\K$ is $(\stk, \stg)$-strongly stable (for $\stk > 0$ and $0 < \stg < 1$) if $\|\K \|_2\leq \stk$, and there exists matrices $\lqL$ and $\lqH$ such that $\lqA+\lqB\K = \lqH\lqL\lqH^{-1}$, with $\|\lqL\|_2 \leq 1-\stg$ and $\|\lqH\|_2\|\lqH^{-1}\|_2\leq \stk$.
\end{definition}
The policy class $\Flq$ is then defined as $\Flq = \{\K\;|\;\K \text{ is } (\stk, \stg)-\text{strongly stable} \}$. Given this policy class, the sequential protocol for this game proceeds as follows, starting from state $\st_0 = 0$.\\\vspace{-2mm}

\noindent On round $t = 1, \ldots, \T,$\vspace{-1mm}
\begin{itemize}
  \item the learner selects a policy $\K_t \in \Flq$ and the adversary selects instance $\z_t \in \Z = (\lqQ_t, \lqR_t)$ such that $\lqQ_t \succeq 0, \lqR_t \succeq 0$ and $\trace(Q_t), \trace(R_t) \leq C$\vspace{-1mm}
  \item the learner receives loss $\loss(\f_t, \st_t, \z_t) = \st_t^\top \lqrQ_t \st_t + \ac_t^\top \lqR_t \ac_t$\vspace{-1mm}
  \item the state of the system transitions to $\st_{t+1} = \lqrA \st_t + \lqrB \ac_t + \w_t$
\end{itemize}
where we assume that the stochastic noise $\w_t \sim \N(0, \W)$ with $\|\W\|_2 \leq \Wbnd$, $\trace(\W) \leq \trw$ and $\W \succeq \sigl I$. The transition matrices $\lqA$ and $\lqB$, as well as the noise covariance matrix $\W$ are assumed to be known to both the learner and the adversary in advance. %Given this setup, the stationary loss is given by
%\begin{equation}\label{eq:losss_lq}
%\begin{gathered}
%  \losss(\K, \z) = \inner{\lqQ + \K^\top\lqR\K }{\cov_\K} = \trace[(\lqQ + \K^\top\lqR\K)\cov_\K]\;, \\
%  \text{where}\quad \cov_\K = (\lqA + \lqB\K)\cov_K(\lqA + \lqB\K)^\top + \W\;.
%\end{gathered}
%\end{equation}
With this setup, the following corollary obtains an upper bound on the minimax value $\val_{\lqr, \T}$ for the above LQR problem.

\begin{corollary}\label{cor:ex3}
For the online LQR sequential game, the value $\val_{\lqr, \T}$ is bounded as
\begin{equation*}
  \val_{\lqr, \T}(\Flq, \Z, \dyn, \loss) \leq \O\left(\sqrt{\T\log(\T)} \right).
\end{equation*}
\end{corollary}

%The online Linear Quadratic Regulator (LQR) setup studied in this section was first studied in Cohen et al.~\cite{cohen2018}. The setup consists of a LQ system - with linear dynamics and quadratic costs - where the cost functions are chosen adversarially. The comparator class consists of a set of strongly stable linear policies (see Appendix~\ref{app:examples}).
%\begin{small}
%\begin{gather}\label{eq:lqr_setup}
%\X = \real^d,\; \Zl = \{(Q, R)\; |\; Q, R\succ 0, \; \trace(Q) , \trace(R)\leq C\},\;
%\Flq = \{\K\in \real^{k \times d}\;|\;\K \text{ is } (\stk, \stg)-\text{strongly stable} \},\; \nonumber\\
%\Zdyn = \phi,\; \loss(\f, \st, \z ) = \st^\top Q\st + (\K\st)^\top R (\K\st),\; \dyn(\st, \K) = (A + BK)\st,\; \Dnoise = \mathcal{N}(0, I).
%\end{gather}
%\end{small}
%With this setup, we now establish the learnability of this problem in the following corollary.
%For the online MDP problem given by $(\Flq, \Z, \dyn, \loss)$ in equation~\eqref{eq:lqr_setup}, we have that the minimax value $\val^{\mathsf{LQR}}_{\T}(\Flq, \Z, \dyn, \loss) = \Ot(\sqrt{T})$.
%\end{corollary}
Note that~\cite{cohen2018} obtained a similar policy regret bound of $\O(\sqrt{T})$ but their analysis only worked for an oblivious adversary whereas the guarantee of Corollary~\ref{cor:ex3} holds for an adaptive adversary.\hfill$\clubsuit$

\subsection{Online non-linear control}
In this section, we look at a non-linear control problem: one formed by extending the LQR problem above to have non-linear deterministic dynamics. We parameterize the dynamics using a non-linear function $\signl : \real^\sdim \mapsto \X$ as follows:
\begin{equation*}
  \st_{t+1} = \signl[\lqA\st_t + \lqB\ac_t]\;,
\end{equation*}
We assume that the function $\signl$ is $1$-Lipschitz and $\|\signl(\st)\|\leq \xbnd$ for some $\xbnd > 0$. This is done to ensure that the dynamics satisfy the ergodicity assumption. We now proceed to define the associated policy class $\Fnl$ as
\begin{equation*}
  \Fnl = \{\f_\param \; | \; \param \in \real^{\pdim}, \|\param\|_2\leq \pbnd, \|[\lqA\st + \lqB\f_\param(\st)] - [\lqA\st' + \lqB\f_\param(\st')]\|_2 \leq (1-\stg)\|\st - \st' \|_2\}\;,
\end{equation*}
where the last condition on the function class establishes a stability condition.  In addition, we assume that the function class $\Fnl$ satisfies a Lipschitz property:
\begin{equation*}
\|\f_\param(\st) - \f_{\param'}(\st)\|_2 \leq \lipf \|\param - \param'\|_2 \quad \text{for all } \quad \st \in \X\;.
\end{equation*}
The above condition implies that if two parameters $\param, \param'$ are close in the parameter space, then the policies parameterized by them are uniformly close for all states. We next outline the learning protocol, with the game starting with $\st_1 = 0$.\\\vspace{-2mm}

\noindent On round $t = 1, \ldots, \T,$\vspace{-1mm}
\begin{itemize}
  \item the learner selects policy $\f_t \in \Fnl$ and the adversary selects $\z_t \in \Z$\vspace{-1mm}
  \item the learner receives loss $\loss(\f_t, \st_t, \z_t) \in [0,1]$\vspace{-1mm}
  \item the state of the system transitions to $\st_{t+1} = \signl[\lqrA \st_t + \lqrB \ac_t]$
\end{itemize}
%For the above setup, we establish that the stationary loss for any policy $\f$ is given by
%\begin{equation}\label{eq:losss_nl}
%  \losss(\f, \z) = \loss(\f, \st^\f_*, \z) \quad \text{where}\quad \st^\f_* = \signl[\lqA\st^\f_*+ \lqB\f(\st^\f_*)]\;,
%\end{equation}
%where the existence of the fixed point is guaranteed by the stability assumption on the function class in conjunction with the Brouwer fixed-point theorem.
With this setup, our next result provides an upper bound on the minimax value $\val_{\nlsf, \T}$ for the online non-linear control problem.

\begin{corollary}\label{cor:ex4}
  For the online non-linear control problem described above, we have that the minimax value
\begin{equation*}
    \val_{\nlsf, \T}(\Fnl, \Z, \dyn) \leq \O\left(\sqrt{T\log(\T)} \right)\;.
  \end{equation*}
\end{corollary}
Notice that the above corollary establishes an upper bound of $\Ot(\sqrt{T})$ for the value $\val_{\nlsf, \T}$. Thus, despite the fact that the  setup does not have the nice structure of the LQR problem, we are able to establish the learnability of the class $\Fnl$ in the online learning with dynamics framework.
\hfill$\clubsuit$

\subsection{Online LQR with adversarial disturbances}
In this section, we consider the example of an online learning with dynamics problem where the adversary is allowed to perturb the dynamics in addition to the adversarial losses at each time step. We will focus on the Linear-Quadratic setup where the dynamics function is linear and the costs quadratic in the state $\st_t$ and action $\ac_t$. Agarwal et al.~\cite{agarwal2019} studied a general version of this problem where they considered the convex cost functions with linear dynamics.

As in the Online LQR example in Section~\ref{ex:olqr-main}, we consider the class of linear policies $\Flq$ which are $(\kappa, \gamma)$-strongly stable. Given this policy class, the online learning with dynamics game proceeds as follows, starting from state $\st_0 = 0$\\

\noindent On round $t = 1, \ldots, \T,$\vspace{-2mm}
\begin{itemize}
  \item the learner selects a policy $\K_t \in \Flq$ and the adversary selects instance $\z_t = (\lqQ_t, \lqR_t)$ such that $\lqQ_t \succeq 0, \lqR_t \succeq 0$ and $\trace(Q_t), \trace(R_t) \leq C$ and $\zd_t$ such that $\|\zd_t\|_2 \leq W$
  \item the learner receives loss $\loss(\f_t, \st_t, \z_t) = \st_t^\top \lqrQ_t \st_t + \ac_t^\top \lqR_t \ac_t$ where action $\ac_t = \K_t \st_t$
  \item the state of the system transitions to $\st_{t+1} = \lqrA \st_t + \lqrB \ac_t + \zd_t$
\end{itemize}
\noindent where we assume that the transition matrices $A$ and $B$ are known to both the learner and adversary in advance. Observe that in this case, a stationary loss $\losss$ does not exist because of the adversarial perturbations $\zd_t$ in the dynamic; indeed, if a learner repeatedly plays the same policy $\K \in \Flq$, the state of the system is not guaranteed to converge to a unique stationary state. We now proceed to obtain an upper bound on the value $\val_{\sf{adv}, \T}$ in the following corollary, by directly controlling the dynamic stability parameters $\{\stable_{\RERM, t}\}$ for this policy class $\Flq$ with a similar FTPL based regularized ERM as used in the proof of Corollary~\ref{cor:lqr}.

\begin{corollary}For the online LQR with adversarial disturbances problem, the value $\val_{\sf{adv}, \T}$ is bounded as
  \begin{equation*}
    \val_{\sf{adv}, \T} \leq \mathcal{O}(\sqrt{T\log(\T)}).
  \end{equation*}
\end{corollary}
The above corollary recovers the $\widetilde{\mathcal{O}}(\sqrt{T})$ rate obtained by  Agarwal et al.~\cite{agarwal2019}, albeit with quadratic costs compared to the general convex losses studied there.
\hfill$\clubsuit$
%In addition to these examples, in Appendix~\ref{app:examples} we consider additional example of the framework including online adversarial tracking~\cite{abbasi2014}

\section*{Acknowledgments}
We would like to thank Dylan Foster, Mehryar Mohri and Ayush Sekhari for helpful discussions. KB is supported by a JP Morgan AI Fellowship. KS would like to acknowledge NSF CAREER Award 1750575 and Sloan Research Fellowship.

\newpage
\appendix
\appendixpage
%!TEX root = odysys.tex
\section{Proof of Proposition~\ref{prop:dual_game}}\label{app:setup}
%\begin{proposition}[Value $\to$ Dual Game]
%Let $\distl$ and $\dista$ denote the sets of probability distributions over the policy class $\F$ and the adversarial actions $\Z$ respectively, satisfying the necessary conditions for the minimax theorem to hold. Then, we have that\footnote{$\multiminimax{\ldots}_{t=1}^{\T}$ denotes interleaved application of the sequence of operator inside. For example, for $T=2$,  \mbox{$\multiminimax{\sup_{p_t} \inf_{q_t}}_{t=1}^{2}[\cdot] = \sup_{p_1} \inf_{q_1} \sup_{p_2} \inf_{q_2}[\cdot] $} }
%\begin{small}
%\begin{equation}\label{eq:val_game}
%\val_\T(\F, \Z, \dyn, \loss)  := \multiminimax{\pinf_{q_t \in \distl} \psup_{(\z_t, \zd_t) \in \Z}\Eu{\f_t \sim q_t}}_{t=1}^{\T}\left[\polregT\right] =   \multiminimax{\psup_{\ad_t \in \dista}\pinf_{\f_t}\Eu{(\z_t, \zd_t) \sim \ad_t}}_{t=1}^{\T}\left[\polregT\right].
%\end{equation}
%\end{small}
%\end{proposition}
The minimax value of the policy regret for the online learning with dynamics protocol is achieved when at every time $t$, the learner picks the best distribution $\ld_t$, the adversary picks the worst-case $\z_t$ and a sample of policy $\f_t$ is then drawn from $\ld_t$. This can be succinctly represented as a sequence of infimum, supremum and expectations as
\begin{align*}
  \val_\T(\F, \Z, \dyn, \loss)  &= \multiminimax{\pinf_{q_t \in \distl} \psup_{(\z_t, \zd_t) \in \Z}\Eu{\f_t \sim q_t}}_{t=1}^{\T-1}\pinf_{q_\T \in \distl} \psup_{(\z_\T, \zd_\T) \in \Z}\Eu{\f_\T \sim q_\T}\left[\polregT\right]\\
  &\stackrel{\1}{=} \multiminimax{\pinf_{q_t \in \distl} \psup_{(\z_t, \zd_t) \in \Z}\Eu{\f_t \sim q_t}}_{t=1}^{\T-1}\psup_{\ad_\T \in \dista}\pinf_{\f_\T}\Eu{(\z_\T, \zd_\T) \sim \ad_\T}\left[\polregT\right]\\
  &\stackrel{\2}=\multiminimax{\psup_{\ad_t \in \dista}\pinf_{\f_t}\Eu{(\z_t, \zd_t) \sim \ad_t}}_{t=1}^{\T}\left[\polregT\right],
\end{align*}
where $\1$ follows from an application of the von Neumann's minimax theorem for the distributions $\distl$ and $\dista$~(see~\cite[Appendix A]{rakhlin2015a}) and $\2$ follows from repeatedly performing the same step for $t = \{1, \ldots, \T-1\}$. This establishes the desired claim.
\qed

%!TEX root = odysys.tex
\section{Proofs of upper bounds}\label{app:main}
\subsection{Proof of Theorem~\ref{thm:main}}

Recall from equation~\eqref{eq:reg_erm} that the dual regularized ERM for a regularization function $\regf$ and parameter $\regp \geq 0$ is
\begin{equation*}
\f_{\RERM, t} \in \argmin_{\f\in \F} \sum_{s=1}^t \Eu{ \z_s}\left[ \losscf(\f, \zd_{1:s-1}, \z_s)\right] + \regp\cdot \regf(\f).
\end{equation*}
The proof of our main result relies on the following intermediate result which relates the performance of the above RERM with that of any policy $\f \in \F$ when compared on the counterfactual losses $\losscf$.
\begin{lemma}\label{lem:rerm_induc}
For any policy $\f \in \F$ and any sequence of distributions $\{\ad_t\}_{t=1}^\T$ over instance space $\Z$, we have
{\small
\begin{equation}\label{eq:rerm_induc}
\sum_{t=1}^\T \Eu{\z_t}\left[ \losscf(\f_{\RERM, t}, \zd_{1:t-1}, \z_t)\right] \leq \sum_{t=1}^\T\Eu{\z_t}\left[ \losscf(\f, \zd_{1:t-1}, \z_t)\right]  +\regp\cdot(\regf(\f) - \regf(\f_{\RERM, 1})).
\end{equation}
}
\end{lemma}
Taking this lemma as given, let us proceed to the proof of the theorem statement. For the purpose of this proof, we will use the notation $\hf_{t} \defn \f_{\RERM, t}$. Let us begin by considering the value of the game and its equivalence to the dual game established by Proposition~\ref{prop:dual_game} as\footnote{we suppress the dependence of the state $\stf{t}{{\f_{1:t-1}, \zd_{1:t-1}}}$ on the random noise $\w_{1:t-1}$.}
{\small
\begin{align}\label{eq:upper_decom}
  \hspace{-4mm}\val_\T(\F, \Z, \dyn, \loss) &= \multiminimax{\psup_{\ad_t \in \dista}\pinf_{\f_t}\Eu{(\z_t, \zd_t) \sim \ad_t}}_{t=1}^{\T} \left(\sum_{t=1}^\T \Eu{w}[\loss(\f_t, \stf{t}{\f_{1:t-1}, \zd_{1:t-1}}, \z_t)] \right. \left. - \pinf_{\f\in \F} \Eu{w}\left[\sum_{t=1}^{\T}\loss(\f, \stf{t}{\f^{(t-1)}, \zd_{1:t-1}}, \z_t)\right]\right)\nonumber\\
  &\stackrel{\1}\leq \multiminimax{\psup_{\ad_t \in \dista}\Eu{(\z_t, \zd_t) \sim \ad_t}}_{t=1}^{\T} \left(\sum_{t=1}^\T \Eu{w}[\loss(\hf_t, \stf{t}{\hf_{1:t-1}, \zd_{1:t-1}}, \z_t)] - \Eu{\zd_{1:t-1}, \z_t}[\losscf(\hf_t, \zd_{1:t-1}, \z_t)]\right.\nonumber\\
  &\quad\left.\sum_{t =1}^\T \Eu{\zd_{1:t-1}, \z_t}[ \losscf(\hf_t, \zd_{1:t-1}, \z_t)]- \pinf_{\f\in \F} \Eu{w}\left[\sum_{t=1}^{\T}\loss(\f, \stf{t}{\f^{(t-1)}, \zd_{1:t-1}}, \z_t)\right]\right)\nonumber\\
  &\stackrel{\2}{\leq}\multiminimax{\psup_{\ad_t \in \dista}\Eu{(\z_t, \zd_t) \sim \ad_t}}_{t=1}^{\T} \left(\sum_{t=1}^\T \Eu{w}[\loss(\hf_t, \stf{t}{\hf_{1:t-1}, \zd_{1:t-1}}, \z_t)] - \Eu{\zd_{1:t-1}, \z_t}[\losscf(\hf_t, \zd_{1:t-1}, \z_t)]\right) \hspace{2mm}\hfill{\text{[Term (I)]}}\nonumber\\
  &\quad + \multiminimax{\psup_{\ad_t \in \dista}\Eu{(\z_t, \zd_t) \sim \ad_t}}_{t=1}^{\T}\left(\sum_{t =1}^\T \Eu{\zd_{1:t-1}, \z_t}[ \losscf(\hf_t, \zd_{1:t-1}, \z_t)]- \pinf_{\f\in \F} \Eu{w}\left[\sum_{t=1}^{\T}\loss(\f, \stf{t}{\f^{(t-1)}, \zd_{1:t-1}}, \z_t)\right]\right) \hspace{2mm}{[\text{Term (II)}]},
\end{align}}
where $\1$ follows from upper bounding the infimum over the policies $\f_t$ by the choice of $\f_t = \hf_t$ and $\2$ follows from the linearity of the expectation and sub-additivity of the supremum function.

Focusing on the first term in the above decomposition,
{\small
\begin{align}\label{eq:dyn_stable}
  \text{Term (I)} &= \multiminimax{\psup_{\ad_t \in \dista}\Eu{(\z_t, \zd_t) \sim \ad_t}}_{t=1}^{\T} \left(\sum_{t=1}^\T \Eu{w}[\loss(\hf_t, \stf{t}{\hf_{1:t-1}, \zd_{1:t-1}}, \z_t)] - \Eu{\z_t}[\losscf(\hf_t, \zd_{1:t-1}, \z_t)]\right) \nonumber\\
  &{=} \multiminimax{\psup_{\ad_t \in \dista}\Eu{\zd_t}}_{t=1}^{\T} \left(\sum_{t=1}^\T \Eu{w, \z_t}[\loss(\hf_t, \stf{t}{\hf_{1:t-1}, \zd_{1:t-1}}, \z_t)] - \Eu{ \z_t}[\losscf(\hf_t, \zd_{1:t-1}, \z_t)]\right)\nonumber\\
  &\stackrel{\1}{=} \multiminimax{\psup_{\ad_t \in \dista}\Eu{\zd_t}}_{t=1}^{\T} \left(\sum_{t=1}^\T \psup_{\z_t}|\Eu{w}[\loss(\hf_t, \stf{t}{\hf_{1:t-1}, \zd_{1:t-1}}, \z_t)] -[\losscf(\hf_t, \zd_{1:t-1}, \z_t)]| \right)\nonumber\\
  & \stackrel{\2}{\leq} \sum_{t=1}^\T\stable_{\RERM, t},
\end{align}
}
where $\1$ follows from an application of H\"older's inequality and $\2$ follows from the definition of the dynamic stability of the regularized ERM algorithm.

Having established the upper bound on the first term, we now proceed to the second term of equation~\eqref{eq:upper_decom}.
{\small
\begin{align*}
\text{Term (II)} &\stackrel{\1}{=} \multiminimax{\psup_{\ad_t \in \dista}\Eu{(\z_t, \zd_t) \sim \ad_t}}_{t=1}^{\T}\psup_{\f\in \F}\left(\sum_{t =1}^\T \Eu{\z_t}[ \losscf(\hf_t, \zd_{1:t-1}, \z_t)]-  \sum_{t=1}^{\T}\losscf(\f, \zd_{1:t-1}, \z_t)\right)\\
&\stackrel{\2}{\leq} \multiminimax{\psup_{\ad_t \in \dista}\Eu{(\z_t, \zd_t), \z_t'}}_{t=1}^{\T}\psup_{\f\in \F}\left(\sum_{t =1}^\T \losscf(\f, \zd_{1:t-1}, \z_t')-  \sum_{t=1}^{\T}\losscf(\f, \zd_{1:t-1} \z_t)\right) + \regp\psup_{\f \in \F}\regf(\f)\\
&\stackrel{\3}{\leq} \multiminimax{\psup_{\ad_t \in \dista}\Eu{(\z_t, \zd_t), \z_t'}}_{t=1}^{\T}\Eu{\rad_{1:T}}\psup_{\f\in \F}\left(\sum_{t =1}^\T \rad_t(\losscf(\f, \zd_{1:t-1}, \z_t')-  \losscf(\f, \zd_{1:t-1} \z_t))\right) + \regp\psup_{\f \in \F}\regf(\f)\\
&\phantom{(}\leq 2 \multiminimax{\psup_{\ad_t \in \dista}\Eu{(\z_t, \zd_t)}}_{t=1}^{\T}\Eu{\rad_{1:T}}\psup_{\f\in \F}\left(\sum_{t =1}^\T \rad_t\losscf(\f, \zd_{1:t-1}, \z_t)\right) + \regp\psup_{\f \in \F}\regf(\f),
\end{align*}
}
where $\1$ follows from rewriting the comparator in terms of the counterfactual loss $\losscf$, $\2$ follows from Lemma~\ref{lem:rerm_induc}, and in $\3$ we introduce the Rademacher variables $\rad_t$. Using Jensen's inequality, we can obtain a further upper bound on Term (II) as
\begin{align}\label{eq:seq_rad}
  \text{Term (II)} &\leq 2 \multiminimax{\psup_{\ad_t \in \dista}\Eu{(\z_t, \zd_t), \rad_t}}_{t=1}^{\T}\psup_{\f\in \F}\left(\sum_{t =1}^\T\rad_t\losscf(\f, \zd_{1:t-1}, \z_t)\right) + \regp\psup_{\f \in \F}\regf(\f)\nonumber\\
  &{\leq} 2\sup_{\mathbf{\z}, \bm{\zd}}\Eu{\rad_{1:\T}}\psup_{\f\in \F}\left(\sum_{t =1}^\T\rad_t\losscf(\f, [\bm{\zd}_1(\rad), \ldots, \bm{\zd}_{t-1}(\rad)], \mathbf{\z}_t(\rad))\right) + \regp\psup_{\f \in \F}\regf(\f)\;
\end{align}
where in the last line, we have replaced the worst case joint distributions over the $\Z$ space by the corresponding worst case $\Z$-valued trees (see~\cite{han2013,rakhlin2015a} for more details). The upper bound on the value $\val_\T(\F, \Z, \dyn, \loss)$ now follows from combining the bounds obtained in equations~\eqref{eq:dyn_stable} and~\eqref{eq:seq_rad}.
\qed

\paragraph{Proof of Lemma~\ref{lem:rerm_induc}.}
For the purpose of this proof, we will use the short hand $\hf_t \defn \f_{\RERM, t}$. We will prove the statement of the lemma via an inductive argument on the number of time steps $t$.

\textit{Base Case:} For time step $t = 1$, we have that $\f_{\RERM, 1}$ is the minimizer of the regularized loss implying
\begin{equation*}
  \Eu{\z_1}\left[ \losscf(\hf_{1}, \zd_{\phi}, \z_1)\right] \leq \Eu{\z_1}\left[ \losscf(\f, \zd_{\phi}, \z_1)\right] + \regp(\regf(\f) - \regf(\hf_1))
\end{equation*}
for any $\f \in \F$.

\textit{Inductive Step:} Assume that the equation~\eqref{eq:rerm_induc} holds for some time step $s$ and consider the cumulative loss at time step $s+1$
\begin{align*}
\sum_{t=1}^{s+1}\Eu{\z_t}\left[ \losscf(\hf_{t}, \zd_{1:t-1}, \z_t)\right] &= \sum_{t = 1}^s\Eu{\z_t}\left[ \losscf(\hf_{ t}, \zd_{1:t-1}, \z_t)\right] + \Eu{ \z_{s+1}}\left[ \losscf(\hf_{s+1}, \zd_{1:s}, \z_{s+1})\right]\\
&\stackrel{\1}{\leq} \sum_{t=1}^{s} \Eu{\z_t}\left[ \losscf(\hf_{s+1}, \zd_{1:t-1}, \z_t)\right] +  \regp(\regf(\hf_{s+1}) - \regf(\hf_1))\\
&\quad + \Eu{\z_{s+1}}\left[ \losscf(\hf_{s+1}, \zd_{1:s}, \z_{s+1})\right]\\
&\stackrel{\2}{\leq} \sum_{t=1}^{s+1}\Eu{\z_t}\left[ \losscf(\f, \zd_{1:t-1}, \z_t)\right] + \regp(\regf(\f) - \regf(\hf_1)),
\end{align*}
where $\1$ follows from the induction hypothesis for time $s$ and applying it for $\f = \hf_{s+1}$, and $\2$ follows from the fact that $\hf_{s+1}$ is the minimizer of the regularized objective at time $s+1$. This concludes the proof of the lemma.
\qed

\subsection{Proof of Proposition~\ref{prop:upper_mini}}
For the purpose of this proof, we restrict our attention to an oblivious adversary wherein the adversary selects instances $\{\z_t\}_{t=1}^\T$ before the game begins. Several recent works~\cite{arora2012,chen2019} have studied specific versions of a mini-batching algorithms under such an oblivious adversary.

For any mini-batching algorithm with parameter $\batch$, we consider denote by $\Tb = \nicefrac{\T}{\batch}$ as the effective time horizon\footnote{We assume $\T/\batch$ to be an integer; if not, it affects the bound by an additive factor of $\batch$.} of the game. We now look at the mini-batched value of the game %and define the minibatch loss
%\begin{equation*}
%  \Loss(\f_t, \stf{t}{\f_{1:t-1}, \zd_{1:t-1}}, \z_t, \zd_t) = \Eu{w}\left[\sum_{s= \batch(t-1)+1}^{\batch t} \loss(\f_t, \stf{t,s}{\f_{1:t-1}^{(\batch)}, \f_t^{(s-1)}, \zd_{1:t-1}, \zd_{t, 1:s-1}}, \z_{t,s})\right],
%\end{equation*}
%where with a slight abuse of notation, we have overloaded the adversarial instances to denote $\z_t \in \Zl^\batch$ and $\zd_t \in \Zdyn^\batch$ with $\z_{t,s}, \zd_{t,s}$ denoting their $s^{th}$ coordinate. With this notation setup, the mini-batched value of the game
{\small
\begin{equation*}
  \val_{\T, \batch}(\F, \Z, \dyn, \loss) \leq \multiminimax{\pinf_{\ld_t}\multiminimax{\psup_{\z_{t, s}, \zd_{t,s}}}_{s=1}^\batch  \Eu{\f_{t,s}\sim \ld_t}}_{t=1}^{\Tb}\En_w[\polregT] %\left(\sum_{t = 1}^{\Tb}\Loss(\f_t, \stf{t}{\f_{1:t-1}, \zd_{1:t-1}}, \z_t, \zd_t) - \inf_{\f \in \F}\sum_{t=1}^{\Tb}\Loss(\f, \stf{t}{\f^(t-1), \zd_{1:t-1}}, \z_t, \zd_t)\right)
\end{equation*}
}
represents the minimax policy regret for any such mini-batching algorithm $\alg_\batch$ in the presence of an oblivious adversary. Let us denote the comparator term by
\begin{equation*}
   \psi(\zd_{1:T}, \z_{1:T}) \defn \pinf_{\f\in \F} \left[\Eu{w}\sum_{t=1}^{\T}\loss(\f, \stf{t}{\f^{(t-1)}, \zd_{1:t-1}}, \z_t)\right].
\end{equation*}
Following a repeated application of von Neumann's minimax theorem similar to the proof of Proposition~\ref{prop:dual_game}, we upper bound the value
{\small
\begin{equation}\label{eq:mb_dual}
  \val_{\T, \batch}(\F, \Z, \dyn, \loss) \leq \multiminimax{\psup_{\ldb_t}\pinf_{\f_t}\Eu{(\zb_t, \zdb_t) \sim \ldb_t}}_{t=1}^{\Tb}\left[\Eu{w}\sum_{t=1}^{\Tb}\sum_{s=1}^\batch \loss(\f_t, \stf{t}{\f_{1:t-1}^{(\batch)},\f_t^{(s-1)} , \zdb_{1:t-1}, \zdb_t^{1:s-1}}, \zb_t^s)
  - \psi(\zdb_{1:\Tb}, \zb_{1:\Tb})\right],
\end{equation}}
where the distribution $\ldb \in \mathcal{P}^\batch$ is a joint distribution over instances $(\zb, \zdb) \in \Z^\batch$ and we have explicitly indicated the dependence of the state variable on the past sequence of policies and adversarial instances. Define the mini-batched loss at time $t$
\begin{equation*}
  \Loss_\batch(\f_t, \zdb_{1:t-1}, \zb_t, \zdb_t;\f_{1:t-1}) \defn \Eu{w}\sum_{s=1}^\batch \loss(\f_t, \stf{t}{\f_{1:t-1}^{(\batch)},\f_t^{(s-1)} , \zdb_{1:t-1}, \zdb_t^{1:s-1}}, \zb_t^s),
\end{equation*}
and the corresponding mini-batched counterfactual loss
\begin{equation*}
  \Losscf_\batch(\f_t, \zdb_{1:t-1}, \zb_t, \zdb_t) \defn \Eu{w}\sum_{s=1}^\batch \loss(\f_t, \stf{t}{\f_{t}^{(\batch)},\f_t^{(s-1)} , \zdb_{1:t-1}, \zdb_t^{1:s-1}}, \zb_t^s).
\end{equation*}
Given these definitions, we can rewrite equation~\eqref{eq:mb_dual} as
\begin{equation}\label{eq:minibatch_val}
\val_{\T, \batch}(\F, \Z, \dyn, \loss) \leq \multiminimax{\psup_{\ldb_t}\pinf_{\f_t}\Eu{(\zb_t, \zdb_t) \sim \ldb_t}}_{t=1}^{\Tb}\left[ \sum_{t=1}^{\Tb} \Loss_\batch(\f_t, \zdb_{1:t-1}, \zb_t, \zdb_t;\f_{1:t-1}) - \psi(\zdb_{1:\Tb}, \zb_{1:\Tb})\right]\;.
\end{equation}
The expression on the right can be seen as a dual game between a learner and an adversary of $\Tb$ rounds. At each round, the adversary reveals a joint distribution $\ldb_t$ over the instances and the learner selects a policy $\ld_t$. The learner then receives the loss $\Loss_\tau$ for that round. We further bound the value by selecting the mini-batched dual ERM strategies for the learner, given by
\begin{equation*}
  \hf_t^\batch \defn \f_{\ERM, t}^\batch = \argmin_{\f} \left(\sum_{s=1}^{t} \Eu{\zb_s, \zdb_s} [\Losscf_\batch(\f, \zdb_{1:s-1}, \zb_s, \zdb_s)]\right).
\end{equation*}
Substituting the above mini-batched policies in equation~\eqref{eq:minibatch_val} and following a similar set of steps as in proof of Theorem~\ref{thm:main}, we get,
\begin{align*}
  \val_{\T, \batch}(\F, \Z, \dyn, \loss) &\leq \multiminimax{\psup_{\ldb_t}\Eu{(\zb_t, \zdb_t) \sim \ldb_t}}_{t=1}^{\Tb} \left[ \sum_{t=1}^{\Tb} \Loss_\batch(\hf^\batch_t, \zdb_{1:t-1}, \zb_t, \zdb_t;\hf^\batch_{1:t-1}) - \psi(\zdb_{1:\Tb}, \zb_{1:\Tb})\right]\\
  &\leq\multiminimax{\psup_{\ldb_t}\Eu{(\zb_t, \zdb_t) \sim \ldb_t}}_{t=1}^{\Tb} \left[\sum_{t=1}^{\Tb}\Loss_\batch(\f_t, \zdb_{1:t-1}, \zb_t, \zdb_t;\f_{1:t-1}) -  \Losscf_\batch(\f_t, \zdb_{1:t-1}, \zb_t, \zdb_t)\right]\\
  &\quad + \multiminimax{\psup_{\ldb_t}\Eu{(\zb_t, \zdb_t) \sim \ldb_t}}_{t=1}^{\Tb} \left[\sum_{t=1}^{\Tb}\Eu{\zb_t, \zdb_t}[\Losscf_\batch(\f_t, \zdb_{1:t-1}, \zb_t, \zdb_t)] -\psi(\zdb_{1:\Tb}, \zb_{1:\Tb}) \right]\\
  &\stackrel{\1}{\leq} \sum_{t=1}^\T \stable_{\ERM, t}^\batch + \multiminimax{\psup_{\ldb_t}\Eu{(\zb_t, \zdb_t) \sim \ldb_t}}_{t=1}^{\Tb} \left[\sum_{t=1}^{\Tb}\Eu{\zb_t, \zdb_t}[\Losscf_\batch(\f_t, \zdb_{1:t-1}, \zb_t, \zdb_t)] -\psi(\zdb_{1:\Tb}, \zb_{1:\Tb}) \right],
\end{align*}
where step $\1$ follows from upper bounding the sequence of joint distributions by the worst-case sequence of the adversary instances $\zb_t, \zdb_t$. The second term in the above expression can be upper bounded by using an induction argument, similar to that used in Lemma~\ref{lem:rerm_induc}. The resulting bound is given by
{\small
\begin{equation*}
\val_{\T, \batch}(\F, \Z, \dyn, \loss) \leq \sum_{t=1}^\T \stable_{\ERM, t}^\batch + \multiminimax{\psup_{\ldb_t}\Eu{(\zb_t, \zdb_t)}}_{t=1}^{\Tb}\sup_{\f \in \F}\left[\sum_{t=1}^{\Tb}\Eu{\zb_t, \zdb_t}[\Losscf_\batch(\f, \zdb_{1:t-1}, \zb_t, \zdb_t)] - \Losscf_\batch(\f, \zdb_{1:t-1}, \zb_t, \zdb_t) \right].
\end{equation*}}
Symmetrizing the above expression and introducing Rademacher variables, we get,
{\small
\begin{align*}
\val_{\T, \batch}(\F, \Z, \dyn, \loss) &\leq \sum_{t=1}^\T \stable_{\ERM, t}^\batch + 2\multiminimax{\psup_{\ldb_t}\Eu{(\zb_t, \zdb_t)}\Eu{\rad_t}}_{t=1}^{\Tb}\sup_{\f \in \F}\left[\sum_{t=1}^{\Tb} \rad_t\Losscf_\batch(\f, \zdb_{1:t-1}, \zb_t, \zdb_t) \right]\\
&\stackrel{\1}{\leq}  \sum_{t=1}^\T \stable_{\ERM, t}^\batch +2\multiminimax{\psup_{\ldb_t}\Eu{(\zb_t, \zdb_t)}\Eu{\rad_t}}_{t=1}^{\Tb}\sup_{\f \in \F}\sum_{s=1}^\batch\left[\left|\sum_{t=1}^{\Tb} \rad_t\losscf_{s, t}(\f, \zdb_{1:t-1}, \zb_t, \zdb_t) \right|\right]\\
&\leq \sum_{t=1}^\T \stable_{\ERM, t}^\batch +2\sum_{s=1}^\batch\multiminimax{\psup_{\ldb_t}\Eu{(\zb_t, \zdb_t)}\Eu{\rad_t}}_{t=1}^{\Tb}\sup_{\f \in \F}\left[\left|\sum_{t=1}^{\Tb} \rad_t\losscf_{s, t}(\f, \zdb_{1:t-1}, \zb_t, \zdb_t) \right|\right]\\
&\leq \sum_{t=1}^\T \stable_{\ERM, t}^\batch + 2\batch \cdot\sup_{s \in [\batch]}\Rseq_{T/\batch}(\losscf_s\com \F),
\end{align*}}
where step $\1$ follows from swapping the supremum with the summation and in the last step we have used the definition of sequential Rademacher complexity with an absolute value. This establishes the desired claim.
%Let us simplify the expression for the value $\val_{\T, \batch}$ by focussing on the last minibatch given by $t = \Tb$.
%Applying von Neumann's minimax theorem within each mini-batch similar to Proposition~\ref{prop:dual_game}, we get
%\begin{equation*}
%\pinf_{\ld_{\Tb}}\multiminimax{\psup_{\z_{\Tb, s}, \zd_{\Tb,s}}\Eu{\f_{\Tb,s}}}_{s=1}^\batch \sum_{s=1}^\batch \loss(\f_{{\Tb,s}}, \stf{\Tb,s}{})
%\end{equation*}
\qed

%!TEX root = odysys.tex
\section{Proofs of lower bounds}\label{app:lower}
\subsection{Proof of Theorem~\ref{thm:lower}}
We begin by recalling the example instance described in the proof sketch of Theorem~\ref{thm:lower}. The online learning game between learner and adversary is given comprises of the state space $\X =\{\st \in \real^d\; | \; \|\st\|_2 \leq 1 \}$ and the set of adversarial actions $\Z_{\loss}^\lin = \{\z \in \real^d\; |\; \|z\|_2 \leq 1\}$ for some dimension $d \geq 3$. In our setup, the adversarial instance space for the dynamics is empty. Given this state space, our policy class $\F_\lin $ is a constant class of policies
\begin{equation*}
  \F_\lin = \{\f_\ff\; |\; \f(\st) = \ff \text{ for all states } \st \text{ with }\ff \in  \ball_d(1)\},
\end{equation*}
consisting of policies $\f_\ff$ which select the same action $\ff$ at each state $\st$. With a slight abuse of notation, we represent the policy $\f_t$ played by the learner at time by the corresponding $d$-dimensional vector $\ff_t$. Further, we let the dynamics function $\dyn_\lin(\st_t, \ff_t, \zd_t) =  \ff_t$ with the noise distribution $\Dnoise = 0$. Observe that the dynamics simply remembers the last action played by the learner and sets the next state as $\st_{t+1} = \ff_t$ in a deterministic way with the starting state $\st_1 = 0$. We now define the loss function which consists of two parts, a linear loss and a $L$-Lipschitz loss involving the dynamics:
\begin{equation}\label{eq:lower_loss}
\loss_L(\ff_t, \st_t, \z_t) = \inner{\ff_t}{\z_t} + \sig(f_t, x_t)\quad \text{where} \quad
\sig(f_t, x_t) = \begin{cases}
L \|f_t - x_t\|_2 \quad &\text{for  } \|f_t - x_t\|_2 \leq \frac{1}{L}\\
1 \quad & \text{otherwise}
\end{cases}.
\end{equation}
Observe that this example constructs a family of instances one for each value of the Lipschitz constant of $L$ of the function $\sigma$. For this setup, the loss function $\losscf$ for any time $t>1$ is just the linear part of the loss
\begin{equation*}
  \losscf(\ff, \st[f^{t-1}], \z) = \inner{\ff}{\z} + \underbrace{\sig(\ff, \st)}_{= 0} = \inner{\ff}{\z}\;.
\end{equation*}
Let us now break down the lower bound analysis into two cases: that of the Lipschitz constant $L\leq 1$ and $L > 1$.
\paragraph{Case 1: $L \leq 1$.} For the case when $L \leq 1$, we lower bound the value of the game by ignoring the dynamics loss $\sig$.
\begin{align*}
  \val_\T(\F_\lin, \Z^\lin, \dyn_\lin, \loss) &= \multiminimax{\pinf_{\ld_t}\psup_{\z_t} \Eu{\ff_t}}_{t=1}^\T \sum_{t=1}^\T \left(\inner{\ff_t}{\z_t} + \sig(\ff_t, \ff_{t-1}) - \inf_{\ff\in \F_{\lin}}\sum_{t=1}^\T\inner{\ff}{\z_t} + \sig(\ff, \st_1)\right)\\
  &\stackrel{\1}{\geq} \multiminimax{\pinf_{\ld_t}\psup_{\z_t} \Eu{\ff_t}}_{t=1}^\T \left(\sum_{t=1}^\T\inner{\ff_t}{\z_t} - \inf_{\ff\in \F_{\lin}}\left(\sum_{t=1}^\T\inner{\ff}{\z_t}\right)\right) + 1,
\end{align*}
where $\1$ follows by noting that $\sig(\ff_t, \ff_{t-1}) \in [0,1]$. The above lower bound reduces the value to that of a online linear game between a learner and an adversary. A lower bound on the value of this game can be shown to be $\sqrt{T}/2$ (see~\cite{rakhlin2014}) and thus for the case when $L < 1$, we have that the value $\val_\T(\F_\lin, \Z^\lin, \dyn_\lin, \loss) \geq c\sqrt{T}$ for some $c = 0.5    $.

\paragraph{Case 2: $L > 1$.} We now proceed to the case when the Lipschitz constant\footnote{Assume $L$ to be an integer; if not, redefine $L = \left\lfloor L \right\rfloor$. } $L > 1$. In order to prove the requisite lower bound, we will describe the adversaries choice of action $\z_t$. The adversaries strategy is to stick to some action $\z$ and only \emph{switch} to a new action when one of events E1 or E2 happen.
\begin{itemize}
  \item[E1] The time $t = \lambda L$ for $\lambda = \{1, \ldots, \T/L \}$.
  \item[E2] Let $t_0$ denote the last time the adversary had switched and denote the expected deviation from the previous move by $\delta_t \defn \En_{\ff_t, \ff_{t-1}}\|f_t - \ff_{t-1}\|$. Further, let $\Delta_{t_0, t} = \sum_{s=t_0}^t \delta_t$ denote the cumulative deviation of the moves from time $t_0$ upto time $t$. The adversary switches at time $t$ whenever $\Delta_{t_0, t} > \frac{1}{L}$.
\end{itemize}
Given the above events, we now define the adversarial action when it switches. Let $t$ be a time instance when one of E1 or E2 happens. Then the adversary selects $\z_t$ such that
\begin{equation*}
  \|\z_t\|_2 = 1, \quad \inner{Z_{t-1}}{ \z_t} = 0. \quad \Eu{\ff_t \sim \ld_t}\inner{\ff_t}{\z_t} = 0,
\end{equation*}
where $Z_{t-1} = \sum_{s=1}^{t-1} \z_s$ is the cumulative sum of the adversary's past actions. Note that our choice of dimensions $d\geq 3$ ensures that such a $\z_t$ will always exist.

In order to understand the performance of any algorithm, let us partition the time interval into $T/L$ blocks each of length $L$ and denote each such bock $I_i \defn [L(i-1)+1, Li ]$. Let $k_i$ denote the number of times the learner causes event E2 to occur in the interval $I_i$. Observe that the cumulative loss within an interval $I_i$
\begin{equation}\label{eq:learn_lower}
\multiminimax{\Eu{\f_t}}_{t\in I_i}  \sum_{t\in I_i} \inner{\ff_t}{\z_t} + \sig(\ff_t, \ff_{t-1}) \quad \begin{cases}
= 0 \quad &\text{if } k_i = 0\\
\geq k_i - 1\quad &\text{if } k_i \geq 1
\end{cases}
\end{equation}
where the lower bound for the case $k_i \geq 1$ follows since at each round the learner can only obtain inner product $\En_{\ff_t}\inner{\ff_t}{\z_t} \geq -1/L$ at each round of the interval. As soon as $\En_{\ff_t}\inner{\ff_t}{\z_t} < -1/L$, the adversary switches and ensures that $\En_{\ff_t}\inner{\ff_t}{\z_t} = 0$ for that time. The lower bound of $k_i - 1$ follows since the total length of the interval is $L$ and each time event E2 occurs, the learner pays a cumulative cost of $1$. Note that the case for $k_i = 0$ is equivalent to the case $k_i = 1$ and hence going forward, we assume each $k_i \geq 1$.

Let $K = \sum_{i = 1}^{T/L}k_i$ denote the total number of times an algorithm causes event E2 to happen and let $K = \beta \T/L$ for some $\beta \in [1, L]$. Then, for any sequence of learner distributions $[\ld_1, \ldots, \ld_T]$, we have that the policy regret is lower bounded as
\begin{align}\label{eq:polreg_lower}
  \polregT(\alg) &= \multiminimax{\Eu{\ff_t \sim \ld_t}}_{t=1}^\T \left(\inner{\ff_t}{\z_t} + \sig(\ff_t, \ff_{t-1}) \inf_{\ff\in \F_\lin}\sum_{t=1}^\T\inner{\ff}{\z_t} \right) \stackrel{\1}{\geq} K - \frac{T}{L} + \|Z_T\|_2,
\end{align}
where $\1$ follows from an application of the Cauchy-Schwarz inequality and from the bound in \eqref{eq:learn_lower}. In order to lower bound the term $\|Z_\T\|_2$, we define the set of times $T_{\sf sw} = [\hat{t}_1, \ldots, \hat{t}_{K_{\sf ad}}]$  where $K_{\sf ad} \leq T/L + K$ is the total number of switches that the adversary makes. Further, let us denote by $\hat{z}_i$ the choice of the adversary at time $\hat{t}_i$ and by $\gamma_i\defn \hat{t}_{i+1} - \hat{t}_{i}$ as the length of the interval for which the adversary played $\hat{z}_i$. Then, the squared norm
\begin{equation*}
\|Z_\T\|_2^2 = \| \sum_{i = 1}^{K_{\sf ad}} \gamma_i \hat{z}_i\|_2^2 = \sum_{i = 1}^{K_{\sf ad}} \gamma_i^2\|\hat{z}_i\|_2^2 + 2\sum_{j=2}^{K_{\sf ad}}\gamma_j\inner{\sum_{i = 1}^{j-1}\gamma_i\hat{z}_i }{\hat{z}_j} = \sum_{i = 1}^{K_{\sf ad}} \gamma_i^2,
\end{equation*}
where the last inequality follows from the choice of adversary ensuring that $\inner{Z_{i-1}}{\z_i} = 0$ and noting that $\|\z_t\| = 1$ for all time $t$. We can now obtain a lower bound on $\|Z_\T\|_2$ by an application of the Cauchy-Schwarz inequality as
\begin{equation*}
  \|Z_\T\|_2 = \sqrt{\sum_{i = 1}^{K_{\sf ad}} \gamma_i^2} \geq \frac{\sum_{i=1}^{K_{\sf ad} }\gamma_i}{\sqrt{K_{\sf ad}}} =\frac{T}{\sqrt{K_{\sf ad}}}.
\end{equation*}
Substituting the above value in equation~\eqref{eq:polreg_lower} and taking an infimum over all algorithms, we have that the minimax value
\begin{equation*}
  \val_\T(\F_\lin, \Z^\lin, \dyn_\lin, \loss) \geq \inf_{\beta \in [1, L]} \left((\beta-1)\frac{T}{L} + \frac{\sqrt{LT}}{\sqrt{\beta+1}}\right)\;,
\end{equation*}
where the inequality above follows from setting $K = \beta \T/L$ and the fact that $K_{\sf ad} \leq T/L + K$. Optimizing for the value of $\beta$, we get that the minimax value
{\begin{equation}\label{eq:val_lower_proof}
  \val_\T(\F_\lin, \Z, \dyn_\lin, \loss_L) \geq \begin{cases}
 \frac{\sqrt{T}}{2}\quad &\text{for } 0 < L < 1\vspace{1mm}\\
 \frac{\sqrt{LT}}{\sqrt{2}} \quad &\text{for } 1 \leq  L \leq (32\T)^{\frac{1}{3}}\vspace{1mm}\\
 2^{\frac{1}{3}}\T^{\frac{2}{3}} \quad &\text{for } L > (32\T)^{\frac{1}{3}}
\end{cases}\;.
\end{equation}}
Thus, we have that the value is lower bounded by these three different terms each corresponding to different ranges of the Lipschitz constant $L$. In order to obtain the requisite lower bounds, we now evaluate each term on the right hand side of equations~\eqref{eq:lower-a}-~\eqref{eq:lower-c}.

\paragraph{Bound~\eqref{eq:lower-a}.} This corresponds to the sequential Rademacher complexity of the class $\F$ which corresponds to the unit Euclidean ball with respect to the linear loss. Following the calculations in Rakhlin and Sridharan~(see\cite[Chapter 10]{rakhlin2014}), we have that
\begin{equation}\label{eq:lower-a-proof}
  \Rseq_{T}(\losscf\com\Fgen) \leq \sqrt{T}.
\end{equation}

\paragraph{Bound~\eqref{eq:lower-b}.} In order to establish an upper bound on the dynamic stability parameters, we consider the regularization given by the squared loss as $\regf(\ff) = \frac{\|\ff\|_2^2}{2}$ with some regularization parameter $\regp \geq 0$. Given that the form of the counterfactual loss $\losscf$, the regularized ERM
\begin{equation*}
  \ff_{\RERM, t} = \text{Proj}_{\mathbb{B}_{d}(1)}\left(\frac{1}{\regp}\sum_{s=1}^t\Eu{\z_s \sim \ad_s} [\z_s]\right)
\end{equation*}
for the dual game and the adversarial distributions given by $\{\ad_t\}$. Consequently, the stability parameters \begin{equation*}\beta_{\RERM, t} = \sig(\ff_{\RERM, t}, \ff_{\RERM, t-1}) \leq L\|\ff_{\RERM, t} - \ff_{\RERM, t-1} \|_2 \leq \frac{L}{\regp}.
\end{equation*}
Finally, the bound of equation~\eqref{eq:lower-b} can now be evaluated as
\begin{equation}\label{eq:lower-b-proof}
  \inf_{\regp > 0}\left(\sum_{t=1}^\T\stable_{\RERM, t} + \regp\cdot\sup_{\f\in \F}\regf(\f) \right) \leq \inf_{\regp\geq0} \left(\frac{LT}{\regp} + \frac{\regp}{2} \right) = \sqrt{\frac{LT}{2}}.
\end{equation}

\paragraph{Bound~\eqref{eq:lower-c}.} We now proceed to the bound given by the mini-batching ERMs with parameter $\tau > 0$. The stability parameters for the mini-batching ERM can be upper bounded as
\begin{equation*}
  \stable_{\ERM, t}^\tau : \begin{cases}
  \leq 2 \quad &\text{for } t\equiv 0\bmod \tau\\
  =0 \quad &\text{otherwise}
\end{cases},
\end{equation*}
where the first case follows trivially from the fact that two unit norm vectors can have a distance at most $2$ and the second case is a consequence of the fact that $\loss = \losscf$ anytime an algorithm repeats the past two policies. Combining this with the sequential Rademacher bound of equation~\eqref{eq:lower-a-proof} we have
\begin{equation}\label{eq:lower-c-proof}
  \inf_{\tau > 0}\left(\sum_{t=1}^\T \stable^\tau_{\ERM, t} + 2\tau\Rseq_{\T/\tau}(\losscf\com\F) \right) \leq \inf_{\tau>0} \frac{2T}{\tau} + 2\tau\sqrt{\frac{T}{\tau}} = 2\T^{\frac{2}{3}}.
\end{equation}
Comparing equations~\eqref{eq:lower-a-proof},~\eqref{eq:lower-b-proof} and~\eqref{eq:lower-c-proof} with the lower bounds on the value $\val_{\T}$ in equation~\eqref{eq:val_lower_proof}, we see that the sequential Rademacher bound is tight up to constant factors in the regime $L \leq1$, the dynamic stability bounds are tight for the regime $1 < L < (32T)^{\frac{1}{3}}$ and the mini-batching bounds are tight for the range $(32T)^{\frac{1}{3}} \leq L \leq T$. This establishes the desired claim.
\qed

\subsection{Proof of Proposition~\ref{prop:lower_gen}}
We establish both parts of the proposition separately. For both the sub-parts, we lower bound the value $\val_T$ be first describing a problem instance $(\F, \Z, \dyn, \loss)$ and compute the value for a specific choice of adversarial actions. We assume that the loss function $|\loss(\ff, \z)| \leq 1$ for all $\ff \in \Fgen$ and $\z \in \Zl$. The bounds for larger loss values can be obtained by a corresponding scaling.

\subsubsection{Proof of part (a)}
We denote by $K = \T/\tau$ the number of times a mini-batching algorithm changes its policy.

\paragraph{Constructing online learning with dynamics instance.} Given an instance of the online learning problem $(\Fgen, \Zl, \loss)$, we construct the online learning with dynamics instance with state space $\X = \Fgen$ and policy class
\begin{equation*}
  \F_\Fgen = \{\f_\ff\; |\; \ff \in \Fgen, \f_\ff(\st) = \ff \text{ for all } \st \in \X \},
\end{equation*}
which plays the same action $\ff$ for all states $\st \in \X$. Going forward, with a slight abuse of notation we use the action $\ff$ and the constant policy $\f_\ff$ interchangeably.

The adversary's loss instance space is given by $\tilde{\Zl} = \Zl\times\{-1, +1 \}$ with the actions $\z_t \in \Zl$ and $\rad_t \in \{-1, +1 \}$. The dynamics function $\dyn(\st, \f_\ff, \zd) = \ff$ represent the deterministic dynamics which remembers the last action played by the learner and is not affect by the adversary. The instantaneous loss $\losst(\ff_t, \st_t, (\z_t, \rad_t))$ is given as
\begin{equation*}
  \losst(\ff_t, \st_t, (\z_t, \rad_t)) = \rad_t\loss(\ff_t, \z_t) + \ind[\ff_t \neq \st_t].
\end{equation*}
With the above loss function, notice that the counterfactual loss $\losscf(\ff_t, (\z_t, \rad_t)) = \rad_t \loss(\ff_t, \z_t)$ for all time $t > 1$ and the dynamic stability parameters for any algorithm $\stable_t = \En_{\alg}[\ind[\ff_t \neq \ff_{t-1}]]$.

\paragraph{Specifying the adversary.} Given the online learning with dynamics problem above, we now specify an adversary for this setup. Let $\Kst = T/\tau^*$ denote the optimal number of switches given by
\begin{equation*}
  \Kst = \argmin_{K} \left(K + 2\frac{\T}{K}\Rseq_{K}(\loss\com\F_\Fgen)  \right).
\end{equation*}
Note that such a value of $\Kst$ is an equalizer of the two terms and ensures that $\Kst$ and $\frac{2T}{\Kst}\Rseq_{\Kst}$ are equal. Now, consider the worst case $\Zl$-valued tree $\mathbf{\z}_\T$ of depth $\T$ corresponding to the online learning problem $(\Fgen, \Zl, \loss)$
\begin{equation*}
  \ztr_\T = \argsup_{\ztr} \En_{\rad}\left[\psup_{\ff\in \Fgen}\sum_{t=1}^\T\rad_t\loss(\ff, \ztr(\rad)) \right].
\end{equation*}
The adversary computes the tree $\ztr_{2\Kst}$ produces instances $(\z_t, \rad_t)$ as
\begin{itemize}
  \item[Case 1.] Whenever $t = \lambda \tau^*/2$ for $\lambda = \{1, \ldots, 2T/K^*\}$, the adversary samples $\rad_t$ as a Rademacher random variable and sets $\z_t = \ztr_{2\Kst}(\rad_{1:2(t-1)/\tau^*})$.
  \item[Case 2.] For any time $t\neq \lambda\tau^*$, the adversary computes the probability of switch $\psw_t = \En_{\alg}\ind[\ff_t \neq \ff_{t-1}]$ and selects instance $(\z_t, \rad_t)$ as %For any time $t \neq \lambda\tau^*$, denote by $\Gamma(t) = \sup_{s < t} \{ s\; |\; \text{Adversary sampled a fresh copy of } \rad_s \}$. The adversary selects the instance $(\z_t, \rad_t)$ as
  \begin{equation*}
    (\z_t, \rad_t) = \begin{cases}
    (\z_{t-1}, \rad_t \sim \text{Rad}) &\quad \text{if } \psw_t > \frac{1}{2}\\
    (\z_{t-1}, \rad_{t-1}) &\quad \text{otherwise}
  \end{cases}.
  \end{equation*}
\end{itemize}

\paragraph{Lower bound on the value.} For any algorithm $\alg$ producing distributions $\ld_1, \ldots, \ld_\T$, the expected policy regret is
\begin{align*}
  \En_{\alg, \rad}[\polregT] &\stackrel{\1}{\geq}  \sum_{t =1}^\T \ind[\psw_t > 0.5] - \En_\rad \pinf_{\ff \in \Fgen}\sum_{t=1}^T \losscf(\ff, \z_t, \rad_t)\\
  &= \sum_{t =1}^\T \ind[\psw_t > 0.5] + \En_\rad \psup_{\ff \in \Fgen}\sum_{t=1}^T \rad_t\loss(\ff, \z_t)
\end{align*}
where inequality $\1$ follows from fact that whenever $\psw_t > 0.5$, the adversary samples a new Rademacher variable $\rad_t$. For any algorithm, let $\Ksw = \sum_{t} \ind[\psw_t > 0.5]$ denote the number of time periods for which the switching probability is greater than half. We break the lower bound in two separate cases depending on the value of $\Ksw$.

\textbf{Case 1: $\Ksw \geq \Kst$.} For this case, the policy regret for any algorithm can be lower bounded as
\begin{equation}\label{eq:largeK}
  \En_{\alg, \rad}[\polregT] \geq \Kst \stackrel{\1}{=} \frac{1}{2}\left( \Kst + 2\frac{\T}{\Kst}\Rseq_{\Kst}(\loss\com\F_\Fgen\right)\;,
\end{equation}
where $\1$ follows from our previous observation that $\Kst = \frac{2T}{\Kst}\Rseq_{\Kst}$.

\textbf{Case 2: $\Ksw < \Kst$.} For this case, not that the complete time horizon can be divided into at most $3\Kst$ intervals wherein the adversary selects the same instances $(\z, \rad)$, each of length at most $T/2\Kst$. By the pigeonhole principle, we must have at least $\Kst$ intervals having length $T/2\Kst$ beginning at time $t = \lambda\tau^*/2$ for some integral $\lambda$. Denote the collection of times in these intervals by $\mathcal{I}$. We can now lower bound the policy regret as
\begin{align}\label{eq:smallK}
  \En_{\alg, \rad}[\polregT] &\geq \Eu{\rad}\sup_{\ff \in \Fgen} \sum_{t=1}^\T \rad_t\loss(\f, \ztr_{2\Kst}(\rad))\nonumber \\
  &\stackrel{\1}{\geq} \Eu{\rad_t:t \in \mathcal{I}}\left[ \psup_{\ff\in \Fgen}\sum_{t\in \mathcal{I}}\rad_t\loss(\f, \ztr_{2\Kst}(\rad))\right]\nonumber\\
  &\stackrel{\2}{\geq} \frac{T}{2\Kst}\cdot \Rseq_{\Kst}(\loss\com\Fgen)\nonumber\\
  &= \frac{1}{4}\left(\Kst + 2\frac{\T}{\Kst}\Rseq_{\Kst}(\loss\com\F_\Fgen \right)
\end{align}
where $\1$ follows from the an application of Jensen's inequality and the fact that the resampled $\rad_t$ when adversary switched because of the learner are not used to parse the tree $\ztr_{2\Kst}$ and $\2$ follows from noting that each pair $(\z, \rad)$ was used exactly $T/2\Kst$ times.

Combining equations~\eqref{eq:largeK} and~\eqref{eq:smallK} along with the observation that the minimax value of the online learning with dynamics $\val_\T(\F, \Zl\times\{+1, -1 \}, \dyn, \losst)$ is the minimum policy regret for any algorithm establishes the desired claim.
\qed

\subsubsection{Proof of part (b)}
We will proof a slightly stronger version of the lower bound from which the desired statement will follow. We follow a strategy similar to the one used in the proof of part (a) above.

\paragraph{Constructing online learning with dynamics instance.}
Let the dynamics function be defined over states space $\X$ and adversary instance space $\Zdyn$. Consider any loss function $\losst: \F\times\X\times\tilde{\Zl} \mapsto\real$ for some instance space $\tilde{\Zl}$. We define the space of adversarial loss actions $\Zl = \tilde{\Zl}\times\{-1, +1\}$ and the corresponding loss $\loss(\f, \st, (\z, \rad)) = \rad\cdot \losst(\f, \st, \z)$. This defines an instance of the online learning with dynamics problem $(\F, \Z = \Zl\times\Zdyn, \dyn, \loss)$.

\paragraph{Specifying the adversary.} Consider the $\tilde{\Z}_l$ and $\Zdyn$ valued trees $\ztr_\T$ and $\zdtr_\T$ defined as
\begin{equation*}
  (\ztr_\T, \zdtr_\T) = \argsup_{\ztr, \zdtr} \En_{\rad}\left[ \psup_{\f \in \F} \sum_{t=1}^\T\rad_t\losscf(\f, \zdtr_{1:t-1}(\rad), \ztr(\rad))\right],
\end{equation*}
which correspond to the worst-case trees of the sequential Rademacher complexity of the class $\losscf\com \F$. At every time $t$, the adversary selects $(\z_t, \rad_t, \zd_t)$ by sampling a uniform Rademacher variable and traversing the two trees as
\begin{equation*}
  \rad_t \sim \text{Rad}, \quad \z_t = \ztr_\T(\rad_{1:t-1}) \quad \text{and} \quad \zd_t = \zdtr_\T(\rad_{1:t-1}).
\end{equation*}

\paragraph{Lower bound on the value.} For any algorithm $\alg$, the expected policy regret is given by
\begin{equation*}
  \En_{\alg, \rad}[\polregT] \stackrel{\1}{=} \En_{\rad}\left[\sup_{\f \in \F} \sum_{t=1}^\T\rad_t \losscf(\f, \zd_{1:t-1}, \z_t) \right] \stackrel{\2}{=}\Rseq_{T}(\losscf\com\F),
\end{equation*}
where $\1$ follows from noting that the loss at time $t$ is a zero-mean random variable and $\2$ is implied by the definition of the trees $\ztr_\T$ and $\zdtr_\T$.

Finally, observing that the minimax value is equal to the policy regret of the best algorithm completes the proof.
\qed

%!TEX root = odysys.tex
\section{Details of examples}\label{app:examples}
In this section, we work out the examples mentioned in Section~\ref{sec:examples} in detail and prove the rates for their respective value functions.

Before proceeding to the examples, we introduce some notation. Most of the examples that we consider have dynamics which are not affected by the adversary, that is, the instance space $\Zdyn$ is empty. We focus on this special case and derive a few results which will be helpful in deriving bounds for the examples.

Borrowing from the theory of stochastic processes, we next define ergodicity of the dynamics  which relates a sequence of instantaneous losses to a notion of stationary loss $\losss: \F\times\Z \mapsto \real$.
\begin{definition}[Ergodicity]\label{def:erg} We say that the dynamics $\dyn$ are ergodic with respect to the loss $\loss$ if for any policy $\f \in \F$ and adversarial action $\z \in \Zl$,  the expected loss converges to a stationary loss starting from any state $\st_1$ as
  \begin{equation*}
    \lim_{t \rightarrow \infty} \En_{\{w_t\}} \loss(\f, \stf{t}{\f^{(t-1)}}, \z) = \losss(\f, \z).
  \end{equation*}
\end{definition}
The loss function $\losss$ can be seen as the limit of the counterfactual losses $\losscf$ and as we shortly show, the losses and dynamics in most of our examples satisfy this ergodicity assumption. For setups where such a stationary loss exists, we define the ergodic stability parameters $\stables_t$ analogous to the dynamic stability parameters.

\begin{definition}[\Ergo Stability]\label{def:erg_stab}
An algorithm $\alg$ is said to be $\{\stables_t\}$-ergodic stable if for all sequences of adversarial actions $[\z_1, \ldots, \z_T]$ and time instances $t\in [T]$
\begin{small}
\begin{equation*}
  \left|\En_{\w_{1:t-1}}[\loss(\f_t, \stf{t}{\f_{1:t-1}, \w_{1:t-1}}, \z_t)] - \losss(\f_t, \z_t) \right| \leq \stables_t \quad \text{where} \quad \f_t = \alg(\z_{1:t-1}).
\end{equation*}
\end{small}
\end{definition}
Observe that the ergodic stability parameters are defined with respect to the stationary loss as compared to their dynamic stability counterparts which were defined with respect to the counterfactual losses. Next, we define the set of regularized ERMs $\fs_{\RERM}$ with respect to these stationary loss as
\begin{equation}\label{eq:reg_erms}
  {\fs}_{\RERM, t} = \argmin_{\f \in \F} \sum_{t=1}^T\Eu{\z_t\sim \ad_t}\left[ \losss(\f, \z_t)\right] + \regp\cdot \regf(\f)\;,
\end{equation}
for some regularization function $\regf$ and parameter $\regp \geq 0$. Given this notation, the following corollary upper bounds the value of the game $\val_\T(\F, \Z, \dyn, \loss)$ in terms of the sequential Rademacher complexity of the loss class $\losss\com\F$ and the ergodic stability of the RERMs $\fs_{\RERM}$.

\begin{corollary}\label{cor:val_ergo}
  For any online learning with dynamics instance $(\F, \Z, \dyn, \loss)$ with ergodic dynamics $\dyn$, consider the set of regularized ERMs given by eq.~\eqref{eq:reg_erms} with regularization function $\regf$ and parameter $\regp \geq0$ having ergodic stability parameters $\{\stables_{\RERM, t} \}_{t=1}^\T$. Then, we have that the value of the game
  \begin{equation}\label{eq:app_upper}
    \val_T(\F, \Z, \dyn, \loss) \leq \sum_{t=1}^\T \stables_{\RERM, t} + 2\Rseq_{T}(\losss\com \F) + 2\regp\sup_{\f\in \F}\regf(\f) + \underbrace{\sup_{\f \in \F} \sum_{t=1}^\T\left\lvert\losscf(\f, \z_t, t) - \losss(\f, \z)\right\rvert}_{\text{Mixing Gap}}.
  \end{equation}
\end{corollary}
Compared with the corresponding upper bound in Theorem~\ref{thm:main}, the above bound has an additional term: the worst case deviation of the counterfactual losses\footnote{since the dynamics are independent of the adversary, we have added an additional time index $t$ to make explicit the number of times policy $\f$ is run in the environment.} from the stationary losses. This term, which we call the \emph{Mixing Gap}, captures how quickly the dynamics mix to these stationary stationary losses when the same policy is repeatedly played over a period of time. The proof of the corollary is very similar to that of Theorem~\ref{thm:main} and we provide it below for completeness.
\begin{proof}[Proof of Corollary~\ref{cor:val_ergo}.]
  We begin by considering the value of the game and its equivalence to the dual game established by Proposition~\ref{prop:dual_game} as
  \begin{small}
  \begin{align*}
    \val_\T(\F, \Z, \dyn) &= \multiminimax{\psup_{\ad_t \in \dista}\pinf_{\f_t}\Eu{\z_t \sim \ad_t}}_{t=1}^{\T}\left[\Eu{w}\left[\sum_{t=1}^\T \loss(\f_t, \stf{t}{\f_{1:t-1}}, \z_t)\right] - \pinf_{\f\in \F} \Eu{w}\left[\sum_{t=1}^{\T}\loss(\f, \stf{t}{\f^{(t-1)}}, \z_t)\right]\right] \\
    &\stackrel{\1}{\leq} \multiminimax{\psup_{\ad_t \in \dista}\Eu{\z_t \sim \ad_t}}_{t=1}^{\T}\left[\Eu{w}\left[\sum_{t=1}^\T \loss(\f_{\RERM, t}, \stf{t}{\f_{\RERM, 1:t-1}}, \z_t)\right] - \pinf_{\f\in \F} \Eu{w}\left[\sum_{t=1}^{\T}\loss(\f, \stf{t}{\f^{(t-1)}}, \z_t)\right]\right]\\
    &\stackrel{\2}{\leq} \multiminimax{\psup_{\ad_t \in \dista}\Eu{\z_t \sim \ad_t}}_{t=1}^{\T}\left[\Eu{w}\left[\sum_{t=1}^\T \Eu{\z_t\sim \ad_t} \left[\loss(\fs_{\RERM, t}, \stf{t}{\fs_{\RERM, 1:t-1}}, \z_t)\right]\right] - \pinf_{\f\in \F} \left(\sum_{t=1}^{\T}\losss(\f,\z_t)\right)\right] \qquad{[\text{Term (I)}]}\\
    &\quad + \multiminimax{\psup_{\ad_t \in \dista}\Eu{\z_t \sim \ad_t}}_{t=1}^{\T}\left[\psup_{\f \in \F} \left( \sum_{t=1}^{\T}\losss(\f,\z_t) - \Eu{w}\left[\sum_{t=1}^{\T}\loss(\f, \stf{t}{\f^{(t-1)}}, \z_t)\right] \right)\right],
  \end{align*}
  \end{small}
  where $\1$ follows from replacing the $\pinf_{\f_t}$ at every time step with $\fs_{\RERM,t}$ and $\2$ follows from the subadditivity of the $\sup$ function and the fact that $\inf_{y}(g(y) + h(y)) \geq \inf_y g(y) + \inf_y h(y)$. The second term in the expression now corresponds to the worst-case deviation of the stationary loss from the counterfactual losses.

  Further, observe that Term (I) above is similar to the term obtained in equation~\ref{eq:upper_decom} and the desired upper bound can be obtained by following the same sequence of steps as in the proof of Theorem~\ref{thm:main}.
\end{proof}

Having established the above corollary, we proceed to studying the examples from Section~\ref{sec:examples} in detail. We reproduce the setup as well as key underlying assumptions from the main paper to help keep the section self-contained.

\subsection{Online Isotron with dynamics}
In this section, we look at the online Isotron with dynamics problem introduced in Section~\ref{sec:examples}. The setup consists of a real valued state space $\X = \real$. The policy class $\F$ is based on a function class $\Fgen$ consisting of a $1$-Lipschitz function along with a $d+1$ unit dimensional vector and is given as
\begin{gather*}
  \Fgen = \{\ff = (\sig, \wvec = (w_1, w)) \; | \; \sig:[-1, 1] \mapsto [-1, 1]\; 1\text{-Lipschitz},\; \wvec \in \real^{d+1}\; |w_1| \leq 1\; \|w\|_2\leq 1 \},\\
  \F_\Fgen = \{\f_{f}\; | \; \f \in \Fgen, \; \f_\ff(\st) = \ff\; \text{for all } \st \in \X\}.
\end{gather*}
The adversary selects instances in the space $\Z = [-1,1]^{d+1}\times [-1,1]$ and we represent each instance \mbox{$\z = (\z_1, \xcov, y)$}. Given this setup, we now formalize the online learning protocol, starting from initial state $\st_1 = 0$.

\noindent On round $t = 1, \ldots, \T,$\vspace{-2mm}
\begin{itemize}
  \item the learner selects a policy $\f_t \in \F_\Fgen$ and the adversary selects $\z_t \in \Z$
  \item the learner receives loss%\footnote{the main paper misses the squared loss term $(\z_{t,1} - w_{t,1})^2$ -- we will fix this in the revised version}
    $\loss(\f_t, \st_t, \z_t) = (y_t - \sig(\inner{\xcov_t}{w_t}))^2 +(\z_{t,1} - w_{t,1})^2 + (\st_t - w_{t,1})^2$
  \item the state of the system transitions to $\st_{t+1} = w_{t, 1}$
\end{itemize}
Given this setup, the next corollary provides a bound on the value of this game $\val_{\mathsf{Iso}, \T}(\F_\Fgen, \Z, \dyn, \loss)$.

\begin{corollary}[Online Isotron with dynamics]\label{cor:iso}
  For the online Isotron with dynamics problem, there exists a universal constant $c>0$ such that
  \begin{equation*}
    \val_{\mathsf{Iso}, \T}(\F_\Fgen, \Z, \dyn, \loss) \leq c \sqrt{\T}\log^{3/2}(\T).
  \end{equation*}
\end{corollary}

\begin{proof} We prove the above statement by bounding the mixing gap and the ergodic stability parameters for the appropriate regularized ERMs.

\paragraph{Bound on mixing gap.} Note that for any time $t>1$, the losses $\losss$ and $\losscf$ are identical since the state variable only depends on the policy at time $t-1$. Therefore, one can upper bound the loss by constant $c = 12$.

\paragraph{ERMs.} For the dual game, we consider the ERM at time $t$ given by
\begin{equation*}
  \ff_{\ERM, t} = (\sig_{t}, \wvec_t) = \argmin_{\sig, \wvec} \left\lbrace\sum_{s=1}^t \left(\Eu{\z_s\sim \ad_s} [(y_s - \sig(\inner{\xcov_s}{w}))^2 +(\z_{s,1} - w_{1})^2]\right) \right\rbrace\;,
\end{equation*}
 and  set $\f_t = \f_{\ff_{\ERM, t}}$.

\paragraph{Ergodic stability parameters.} Note that objective function in the above equation is strongly-convex with respect to the parameter $w_1$ and a simple calculation shows that $|w_{t, 1} - w_{t-1,1}| \leq \frac{2}{t}$. We can now bound the ergodic stability parameter as
\begin{equation}\label{eq:iso_mix_erm}
  \stables_{\RERM, t} = |\loss(\f_t, \stf{t}{\f_{1:t-1}}, \z_t) - \losss(\f_t, \z_t)| = |w_{t-1, 1} - w_{t, 1}|^2 \leq \frac{4}{t^2}.
\end{equation}

\paragraph{Bound on the value.} Having established bounds on the mixing gap and the ergodic stability parameters of the ERM, we now use Corollary~\ref{cor:val_ergo} to upper bound the value of the game as
{\small
\begin{align*}
  \val_{\mathsf{Iso}, \T}(\F_\Fgen, \Z, \dyn, \loss) &\stackrel{\1}{\leq} \sum_{t=1}^\T \stables_{\RERM, t} + 2\Rseq_{T}(\losss\com \F_\Fgen) + 16 \\
 &\overset{\text{Eq.}~\eqref{eq:iso_mix_erm}}{\leq} 8 + 2\Rseq_{T}(\losss\com \F_\Fgen) + 16\\
 &\stackrel{\2}{\leq} c \sqrt{\T}\log^{3/2}(\T)\;,
\end{align*}}
where $\1$ follows by the upper bound of $16$ on the mixing gap and $\2$ follows by the corresponding bound on the sequential Rademacher complexity $2\Rseq_{T}(\losss\com \F_\Fgen)$ from~\cite[Proposition 18]{rakhlin2015a}.
\end{proof}

\subsection{Online Markov decision processes}
In this section, we revisit the problem of Online Markov Decision Processes (MDPs) studied in~\cite{even2009}. The setup consists of a finite state space such that $|\X| = \ssize$ and a finite action space with $|\U| = \asize$. %\footnote{$\ssize$ and $\asize$ have been chosen instead of $\sdim$ and $\adim$ to be consistent with~\cite{even2009}}.
The policy class $\F$ consists of all stationary policies, that is,
\begin{equation*}
  \Fmdp = \{\f \; | \; \f: \X \mapsto \Delta(\U)\},
\end{equation*}
where $\Delta(\U)$ represents the set of all probability distributions over the action space. In addition, the transitions are drawn according to a known function $\trans : \X \times \U \mapsto \Delta(\U)$. The sequential game then proceeds as follows, starting from some state $\st_1 \sim \mdpd$:

\noindent On round $t = 1, \ldots, \T,$\vspace{-2mm}
\begin{itemize}
  \item the learner selects a policy $\f_t \in \Fmdp$ and the adversary selects $\z_t \in \Z = [0,1]^{\ssize\times\asize}$
  \item the learner receives loss $\loss(\f_t, \st_t, \z_t) = \z_t(\st_t, \f_t(\st_t))$
  \item the state of the system transitions to $\st_{t+1} \sim \trans(\st_t, \ac_t)$
\end{itemize}
For every stationary policy $\f$, we let $\trans^f$ denote the transition function induced by $\f$, that is,
\begin{equation*}
  \trans^f(\st, \st') \defn \sum_{\ac \in \U}\f^\ac(\st)\trans^{\st'}(\st, \ac),
\end{equation*}
where we have used superscript to denote the relevant coordinate of the vector. As in~\cite{even2009}, we make the following mixability assumptions about the underlying MDP.
\begin{assumption}[MDP Unichain]\label{ass:mdp_unichain}
We assume that the underlying MDP given by the transition function $\trans$ is uni-chain. Further, there exists $\mixt \geq 1$ such that for all policies $\f$ and distributions $\mdpd, \mdpd' \in \Delta(\U)$ we have
\begin{equation*}
  \|\mdpd\trans^\f - \mdpd'\trans^f \|_1 \leq e^{-1/\mixt}\|\mdpd - \mdpd' \|_1.
\end{equation*}
\end{assumption}
The parameter $\mixt$ is often referred to as the mixing time of the MDP. Since the MDP is assumed to be uni-chain, every policy $\f$ has a well defined unique stationary distribution $\mdpd_\f$ with the stationary loss given by $\losss(\f, \z) = \En_{\st \sim \mdpd_\f}\En_{\ac\sim \f(\st)}\z(\st, \ac)$. Given this setup, we can obtain an upper bound on the value $\val_{\mathsf{MDP}, \T}$ as follows:
\begin{corollary}[Online MDP]\label{cor:mdp}
For the online Markov Decision Process sequential game satisfying Assumption~\ref{ass:mdp_unichain}, the value $\val_{\mathsf{MDP}, \T}(\Fmdp, \Z, \dyn)$ is bounded by
  \begin{equation*}
  \val_{\mathsf{MDP}, \T}(\Fmdp, \Z, \dyn, \loss) \leq 4\mixt\sqrt{\T\ssize\log \asize}  + 2\mixt(1+ e^{1/\mixt}).
\end{equation*}
\end{corollary}
The above corollary helps one recover the same $\O(\sqrt{T})$ regret bound that was obtained by~\cite{even2009}. In terms of the dependence of problem specific parameters, while our bound above shows a $\sqrt{\ssize}$ dependence, their bound was independent of $\ssize$. However note that while the setting studied by~\cite{even2009} consisted of the weaker oblivious adversary, we consider the stronger adaptive adversary which can adapt to the learners strategy.

\begin{proof}[Proof of Corollary~\ref{cor:mdp}.] In order to establish the bound, we begin by bounding the ergodic stability parameters as well as the mixing gap for loss $\losscf$ and $\losss$.

  \paragraph{Bound on mixing gap.}
  Consider any policy $\f \in \Fmdp$ and the associated steady state distribution $\mdpd_\f$. The stationary loss for this problem is then
  \begin{equation*}
    \losss(\f, \z) = \Eu{\st \sim \mdpd_\f}\Eu{\ac \sim \f(\st)}\left[ \z(\st, \ac)\right].
  \end{equation*}
  Consider now the difference between the stationary loss and the counterfactual loss at any time $t$
  \begin{align}\label{eq:mdp_unif_mix}
    \left\lvert\losscf(\f, \z, t) - \losss(\f, \z)\right\rvert &= \left\lvert\Eu{\st_t^\f \sim \mdpd_\f^t}\Eu{u \sim \f(x_t^\f)}[\z(\st_t^\f, \ac)] - \Eu{\st \sim \mdpd_\f}\Eu{\ac \sim \f(\st)}\left[ \z(\st, \ac)\right]\right\rvert\nonumber\\
    &\stackrel{\1}{=} \left\lvert\Eu{\st \sim \mdpd_\f^t}[\tilde{\z}_{\f}(\st)] - \Eu{\st \sim \mdpd_\f}\left[ \tilde{\z}_{\f}(\st)\right]\right\rvert \nonumber\\
    &\stackrel{\2}{\leq} \|\tilde{z}_\f\|_\infty\cdot \|\mdpd^t_\f - \mdpd_\f \|_1\nonumber\\
    &\stackrel{\3}{\leq} 2e^{-(t-1)/\mixt},
  \end{align}
  where in $\1$, we use the redefined loss function $\tilde{\z}_\f(\st) \defn \En_{\ac \sim \f(\st)}\z(\st, \ac)$, $\2$ follows from an application of H\"older's inequality, and $\3$ follows from Assumption~\ref{ass:mdp_unichain} and the fact the $\|\mdpd_1 - \mdpd_\f \| \leq 2$.
  %Thus wehave established that the class of policies $\Fmdp$ satisfy uniform mixability with $\stable_{\mdp, t} = 2e^{-(t-1)/\mixt}$.

  \paragraph{Ergodic stability parameters.}
  For this setup, we will be using a regularized ERM and parameterize the policy $\f_{\RERM, t}$ as a distribution over the deterministic policies present in $\Fmdp$. Let us denote this subset of policies by $\Fmdp^{\mathsf{det}}$. Note that a distribution $\ld$ in  $\distl_{\mdp}^{\mathsf{det}}$ is  randomized policy in the class $\Fmdp$. We will work with the negative entropy function as the regularizer.
  \begin{equation*}
    \ld_{\RERM, t} \in \argmin_{\ld \in \distl_\mdp} \left(\Eu{\f \sim \ld}\left[\sum_{s=1}^t \Eu{\st \sim \mdpd_\f}\Eu{\ac \sim \f(\st)}\left[ \bar{\z}_2(\st, \ac)\right]\right] + \regp\cdot\sum_{i = 1}^{|\Fmdp^\mathsf{det}|}\ld_i \ln \ld_i \right)\;,
  \end{equation*}
  where we denote by $\bar{\z}_s = \En_{\z_s \sim \ad_t} \z_s$ the expected loss at time $s$. Now, we can encode the loss at time $s$ for every policy $\f \in \Fmdp^{\mathsf{det}}$ in a vector $\loss^{\mathsf{det}}_s \in [0,1]^{|\Fmdp^\mathsf{det}|}$ where the $\f^{th}$ coordinate $\loss_{s, \f}^\detm$ is the loss for policy $\f$. Given this, we can show that the distribution $\ld_{\RERM, t}$ is given by:
  \begin{equation*}
    (\ld_{\RERM, t})_{\f} = \frac{\exp\left(\frac{-1}{\regp}\sum_{s=1}^t \loss_{s, \f}^\detm \right)}{\sum_j\exp\left(\frac{-1}{\regp}\sum_{s=1}^t \loss_{s, j}^\detm \right)}\;.
  \end{equation*}
  Going forward, we drop the RERM term from the distribution $\ld_{\RERM, t}$ for ease of readability. In addition, the boundedness of the loss function $|\loss_{s, \f}^\detm| \leq 1$ ensures that the RERM solutions satisfy the following stability property:
  \begin{equation}\label{eq:stable_mdp}
    \|\ld_{ t} - \ld_{t+1}\|_1 \leq \frac{1}{\lambda}.
  \end{equation}
  Given the above stability, one can also obtain a bound on the action distribution between the randomized policy $\f_t = \En_{\f \sim \ld_t}[\f]$ and the corresponding $\f_{t+1}$:
  \begin{equation*}
    \| \f_{t}(\st) - \f_{t+1}(\st)\|_1 = \left\lVert \Eu{\f \sim \ld_t}[f(\st)] - \Eu{\f \sim \ld_{t+1}}[f(\st)] \right\rVert_1 = \|\ld_t - \ld_{t+1} \|_1 \leq \frac{1}{\regp}\;,
  \end{equation*}
  where the second equality follows from the fact that $\|\f(\st)\|_1 = 1$ since they are distributions over the action space $\U$. Now, following a similar calculation as Lemma 5.2 in \cite{even2009}, we can obtain a bound on the variation in state distributions while playing policies $\ld_{1:t-1}$ as compared to the steady state distribution $\mdpd_{\ld_t}$.
  \begin{equation*}\label{eq:mdp_mix_dist}
    \|\mdpd[\ld_{1:t-1}] - \mdpd_{\ld_t} \|_1 \leq \frac{2\mixt^2}{\regp} + 2e^{-t/\mixt}.
  \end{equation*}
  With this bound in place, we can now bound the ergodic stability parameters $\stables_{\RERM, t}$ for the ERM procedure as
  \begin{align}\label{eq:mdp_mix_erm}
  \stables_{\RERM, t} &= |\En\left[\loss(\f_t, \stf{t}{\f_{1:t-1}}, \z) \right] - \losss(\f_t, \z)|\nonumber\\
   & = | \Eu{\st \sim\mdpd[\ld_{1:t-1}] } \bar{z}_{\f_t}(\st) - \Eu{\st \sim\mdpd_{\ld_t} } \bar{z}_{\f_t}(\st)| \nonumber\\
  &\leq \frac{2\mixt^2}{\regp} + 2e^{-t/\mixt}.
  \end{align}

  \paragraph{Bound on the value.} Having established bounds on the mixing gap and the RERM ergodic stability parameters, we now proceed to obtain the requisite bound on the value $\val_{\mathsf{MDP}, \T}(\Fmdp, \Z, \dyn, \loss)$.
  {\small
  \begin{align*}
    \val_{\mathsf{MDP}, \T}(\Fmdp, \Z, \dyn, \loss) &\stackrel{\1}{\leq} \sum_{t=1}^\T \stables_{\RERM, t} + 2\Rseq_{T}(\losss\com \Fmdp) + \sup_{\f \in \F} \sum_{t=1}^\T\left\lvert\losscf(\f, \z_t, t) - \losss(\f, \z)\right\rvert + \regp\ssize \log \asize\\
   &\overset{\text{Eq.}~\eqref{eq:mdp_unif_mix}}{\leq} \sum_{t=1}^\T \stables_{\RERM, t}+ 2\Rseq_{T}(\losss\com \Fmdp) + 2\mixt e^{1/\mixt}+ \regp\ssize \log \asize\\
   &\overset{\text{Eq.}~\eqref{eq:mdp_mix_erm}}{\leq} \frac{2\mixt^2}{\regp}T + 2\Rseq_{T}(\losss\com \Fmdp) + 2\mixt(1+ e^{1/\mixt})+ \regp\ssize \log \asize\\
   &\stackrel{\2}{\leq} 2\mixt\sqrt{\T\ssize\log \asize} + 2\Rseq_{T}(\losss\com \Fmdp) + 2\mixt(1+ e^{1/\mixt})
 \end{align*}}
  where $\1$ follows since the entropy over the class $\Fmdp^\detm$ is upper bounded by $\log \Fmdp^{\detm}$, and $\2$ follows by setting $\regp = \mixt \sqrt{\frac{T}{\ssize \log \asize}}$. Finally, bounding the sequential Rademacher complexity of the finite loss class $\losss\com \Fmdp$ by $2\sqrt{ST\log(A)}$ completes the proof of the corollary.
\end{proof}

\subsection{Online linear quadratic regulator}\label{app:ex-olqr}
The online Linear Quadratic Regulator (LQR) setup studied in this section was first studied in ~\cite{cohen2018}. The setup consists of a LQ system - with linear dynamics and quadratic costs - where the cost functions can be adversarial in nature. The comparator class $\Flq$ comprises a subset of linear policies $\K$ which satisfy the following strong stability property.
\begin{definition}[Strongly Stable Policy]
A policy $\K$ is $(\stk, \stg)$-strongly stable (for $\stk > 0$ and $0 < \stg < 1$) if $\|\K \|_2\leq \stk$, and there exists matrices $\lqL$ and $\lqH$ such that $\lqA+\lqB\K = \lqH\lqL\lqH^{-1}$, with $\|\lqL\|_2 \leq 1-\stg$ and $\|\lqH\|_2\|\lqH^{-1}\|_2\leq \stk$.
\end{definition}
The policy class $\Flq$ is then defined as $\Flq = \{\K\;|\;\K \text{ is } (\stk, \stg)-\text{strongly stable} \}$. Given this policy class, the sequential protocol for this game proceeds as follows, starting from state $\st_0 = 0$\\

\noindent On round $t = 1, \ldots, \T,$\vspace{-2mm}
\begin{itemize}
  \item the learner selects a policy $\K_t \in \Flq$ and the adversary selects instance $\z_t \in \Z = (\lqQ_t, \lqR_t)$ such that $\lqQ_t \succeq 0, \lqR_t \succeq 0$ and $\trace(Q_t), \trace(R_t) \leq C$
  \item the learner receives loss $\loss(\f_t, \st_t, \z_t) = \st_t^\top \lqrQ_t \st_t + \ac_t^\top \lqR_t \ac_t$
  \item the state of the system transitions to $\st_{t+1} = \lqrA \st_t + \lqrB \ac_t + \w_t$
\end{itemize}
where we assume that the stochastic noise $\w_t \sim \N(0, \W)$ with $\|\W\|_2 \leq \Wbnd$, $\trace(\W) \leq \trw$ and $\W \succeq \sigl I$. The transition matrices $\lqA$ and $\lqB$, as well as the noise covariance matrix $\W$ are assumed to be known to both the learner and the adversary in advance. Given this setup, the stationary loss is given by
\begin{equation}\label{eq:losss_lq}
\begin{gathered}
  \losss(\K, \z) = \inner{\lqQ + \K^\top\lqR\K }{\cov_\K} = \trace[(\lqQ + \K^\top\lqR\K)\cov_\K]\;, \\
  \text{where}\quad \cov_\K = (\lqA + \lqB\K)\cov_K(\lqA + \lqB\K)^\top + \W\;.
\end{gathered}
\end{equation}
The following lemma establishes certain structural properties of the stationary loss, namely, boundedness over the policy class $\Flq$ and Lipschitzness with respect to the operator norm.
\begin{lemma}\label{lem:lq_lip}
  The loss function $\losss: \Flq \times \Z \mapsto \real_+$ described in equation~\eqref{eq:losss_lq} satisfies
  \begin{gather*}
     \losss(\K, \z) \leq B_{\sf{max}}\quad \text{for all } \K \in \Flq,\; \z \in \Z\\
     |\losss(\K_1, \z) - \losss(\K_2, \z)| \leq  \lip \|\K_1 - \K_2\|_2 \quad \text{for all }\K_1, \K_2 \in \Flq,\; \z \in \Z,
  \end{gather*}
  where ${B_{\sf{max}}} \defn {C(1+\kappa^2)\frac{\sig_w\kappa^2}{\gamma}}$ and ${\lip} \defn {4C(1+\kappa^2)\frac{\sigma_b\kappa^5\sigma_w}{\gamma^2}}$.
\end{lemma}
We defer the proof of the lemma to the end of section and now proceed to obtain an upper bound on the value $\val_{\lqr, \T}$ for the above problem.

\begin{corollary}[Online Linear Quadratic Regulator]\label{cor:lqr}
For the online LQR sequential game, the value $\val_{\lqr, \T}$ is bounded as
\begin{equation*}
  \val_{\lqr, \T}(\Flq, \Z, \dyn, \loss) \leq \O\left(\sqrt{\T\log(\T)} \right) \;,
\end{equation*}
where the $\O$ notation hides the dependence of the bound on problem-specific parameters (see equation~\eqref{eq:lq_exact} for the exact dependencies).
\end{corollary}

\begin{proof}
As before, our strategy is to establish upper bounds on the mixing gap and the RERM ergodic stability parameter for the LQR problem, and using these with Corollary~\ref{cor:val_ergo} to establish an upper bound on the value $\val_{\lqr, \T}$.

\paragraph{Existence of stationary loss.}
Consider any stable policy $\K \in \Flq$. It is well known that a repeated application of the policy $\K$ in the linear dynamics ensures that the state $\st_t$ converges to a steady-state distribution, that is, the distribution of $\st_t$ and $(\lqA + \lqB\K)\st_t + \w_t$ is the same. Since the noise $\w_t$ is assumed to be $\N(0, \W)$, the steady-state distribution will also be a normal distribution with mean $0$ and steady-state covariance $\cov_\K$ satisfying the following recurrence equation:
\begin{equation*}
  \cov_\K = (\lqA + \lqB\K)\cov_K(\lqA + \lqB\K)^\top + \W \quad \text{or equivalently} \quad \cov_\K = \sum_{s=0}^\infty (\lqA +\lqB\K)^sW(\lqA + \lqB\K)^s)^\top,
\end{equation*}
and the corresponding steady-state loss is given by:
\begin{equation*}
  \losss(\K, \z) = \inner{\lqQ + \K^\top\lqR\K }{\cov_\K} = \trace[(\lqQ + \K^\top\lqR\K)\cov_\K].
\end{equation*}

\paragraph{Bound on mixing gap.}
We now proceed to obtain upper bounds on the mixing gap for this problem instance. Going forward, we define $\cov_{\K, t}$ to be the state-covariance matrix at time $t$ when policy $\K$ has been used for all preceding timesteps. For the purpose of readability, we will drop the dependence of the covariance matrix on the underlying policy $\K$ when it is clear from the context. We begin by looking at the convergence of $\cov_t$ to the stationary matrix $\cov$:
\begin{align}
\|\cov_t - \cov\|_2 &= \left\lVert\sum_{s=0}^{t-1}(\lqA +\lqB\K)^s\W(\lqA + \lqB\K)^s)^\top - \sum_{s=0}^\infty (\lqA +\lqB\K)^sW(\lqA + \lqB\K)^s)^\top\right\rVert_2\nonumber\\
&= \left\lVert\sum_{s=t}^{\infty}(\lqA +\lqB\K)^s\W(\lqA + \lqB\K)^s)^\top \right\rVert_2\nonumber\\
&\stackrel{\1}{\leq} \Wbnd\sum_{s=t}^\infty \stk^2(1-\stg)^{2s}\nonumber\\
&\leq \frac{\Wbnd\stk^2(1-\stg)^{2t}}{\stg}\nonumber
\end{align}
where $\1$ follows from the fact that $\|\lqA +\lqB\K\|^s \leq \stk (1-\stg)^s$ from the strong-stability of $\K$. The above analysis shows that the covariance matrix $\cov_t$ converges to its stationary distribution exponentially fast. One can also obtain a bound similar to above on $\trace(\cov - \cov_t)$ with $\Wbnd$ replaced by $\trw$. Having established this convergence, we establish a bound on the mixing gap as
\begin{align}\label{eq:lqr_unif_mix}
  |\En[\losscf(\f,\z, t)] - \losss(\f, \z)| &= |\inner{\lqQ + \K^\top \lqR \K}{\cov_t - \cov}|\nonumber\\
  &\leq (\Qbnd+\stk^2\Rbnd)\cdot \trace(\cov - \cov_t)\nonumber\\
  &\leq (\Qbnd+\stk^2\Rbnd)\cdot \frac{\trw\stk^2(1-\stg)^{2t}}{\stg}.
\end{align}
Since the above bound is independent of the underlying policy $\K$, we have thus established a bound on the mixing gap for the policy class $\Flq$.

\paragraph{Regularized ERMs.}
We now define the class of RERM's we use for the function class $\Flq$. Instead of working with a fixed regularization function, we shall look at random perturbations as regularizations. Such an idea is popular in the study of online learning algorithms and is often termed as Follow the Perturbed Leader (FTPL); for a detailed study, see~\cite{shalev2012, hazan2016}. Thus, the regularized ERM solutions at time $t$ are given by:
\begin{equation*}
  \K_{t, \pert} = \argmin_{\K \in \Flq} \left(\sum_{s=1}^t\Eu{\z_s\sim \ad_s}\left[\inner{\lqQ_s + \K^\top\lqR\K}{\cov_\K} \right] - \inner{\pert}{\K}\right)\;,
\end{equation*}
where $\pert \in \real^{\adim \times \sdim}$ such that each coordinate of $\pert\sim \expd(\regp)$, the exponential distribution with parameter $\regp > 0$. It was established by~\cite{suggala2019} that if each of the loss function above is $\lip$-Lipschitz, the iterates produced by the FTPL strategy above satisfy:
\begin{equation*}
  \Eu{\pert}\left[\|\K_{t, \pert} - \K_{t+1, \pert}\|_1\right] \leq c\regp \cdot \lip(\adim\sdim)^2\stk \defn \stK\;,
\end{equation*}
where the norm above is defined element-wise. In Lemma~\ref{lem:lq_lip}, we establish that the losses given by $\losss(\f, \z)$ are indeed Lipschitz over the space of policies $\Flq$. With these set of regularized empirical minimizers, we proceed to now bound the ergodic stability parameters of these regularized ERM's, each one of which is strongly-stable.

\paragraph{Sequential strong-stability of solutions.}
We first establish that the set of RERM solutions produced by the algorithm satisfy the sequential strong-stability property (see~\cite{cohen2018} for details) with the appropriate parameters. Note that since each of the $\K_{t}$ (we drop the dependence on the random noise $\pert$) belongs to the class $\Flq$, we have that $\|\K_t \|_2 \leq \stk$.

Let $\cov_t \defn \cov_{\K_t}$ be the steady-state covariance of the $t^{th}$ solution and $\hat{\cov}_t$ denote the covariance of the state reached when policies $\{\K_1, \ldots, \K_{t-1} \}$ are applied at the first $t$ timesteps. Consider the following decomposition for $\lqA + \lqB\K_t$:
\begin{equation*}
  \lqA + \lqB\K_t = \lqH_t \lqL_t \lqH_t^{-1} \quad \text{where} \quad \lqL_t = \cov_t^{-1/2}(\lqA + \lqB\K_t)\cov_t^{-1/2}, \, \lqH_t = \cov_t^{1/2}.
\end{equation*}

\textbf{Bound on $\|\lqH_t\|_2$ and $\|\lqH_t^{-1}\|_2$.} Using the recursive definition of $\cov_t$, we have:
\begin{equation}\label{eq:cov_bnd_lq}
\|\cov_t\|_2 = \left\lVert\sum_{s=0}^{\infty} (\lqA + \lqB\K_t)^s\W((\lqA +\lqB\K)^s)^\top \right\rVert_2 \leq \frac{\Wbnd\stk^2}{\stg}
\end{equation}
The above equation allows us to bound $\|\lqH_t\|_2 \leq \stk\sqrt{{\Wbnd}/{\stg}} = \Hb$. Also, by the definition of the matrix $\cov_t$, we have that $\cov \succeq \W$ and hence $\|\lqH_t^{-1}\| \leq 1/\sqrt{\sigl} = 1/\Ha$. Define $\tilde{\stk} = \Hb/\Ha$ and note that $\tilde{\stk} \geq \stk$.

\textbf{Bound on $\|\lqL_t\|_2$.} Starting from the recursive definition of $\cov_t$, we have,
\begin{align*}
  I &= \cov_t^{-1/2}(\lqA + \lqB\K)\cov_t(\lqA + \lqB\K)^\top \cov_t^{-1/2}+ \cov_t^{-1/2}\W \cov_t^{-1/2}\\
   &\succeq \lqL_t\lqL_t^\top + \sigl\cov_t^{-1}\\
   &\succeq \lqL_t\lqL_t^\top + \frac{\sigl \stg}{\Wbnd\max(\stk^2, 1)} I\;,
\end{align*}
which implies that $\|\lqL_t\| \leq 1-\tilde{\stg}$ where $\tilde{\stg} = \frac{\sigl \stg}{2\Wbnd\max(\stk^2, 1)}$.

\textbf{Bound on $\|\cov_t - \cov_{t+1}\|$.} As before, we begin with the recursive definitions of $\cov_t$ and $\cov_{t+1}$ to get:
\begin{align*}
  \cov_{t+1} - \cov_{t} &= (\lqA + \lqB\K_{t+1})\cov_{t+1}(\lqA + \lqB\K_{t+1})^\top -  (\lqA + \lqB\K_t)\cov_t(\lqA + \lqB\K_t)^\top\\
  &= (\lqA + \lqB\K_{t+1})(\cov_{t+1} - \cov_t)(\lqA + \lqB\K_{t+1})^\top + \underbrace{\lqB\Delta_t\cov_t(\lqA + \lqB \K_{t+1})^\top}_{T_1} + \underbrace{(\lqA + \lqB\K_t)(\lqB\Delta_t)^\top}_{T_2}\\
  &= \sum_{s=0}^{\infty}(\lqA + \lqB\K_{t+1})^s(T_1 + T_2)((\lqA + \lqB\K_{t+1})^s)^\top\;,
\end{align*}
where $\Delta_t = \K_{t+1} - \K_t$. Taking norms on both sides, we get:
\begin{equation}\label{eq:lqr_cov_bnd}
  \|\cov_{t+1} - \cov_{t}\|_2 \leq \frac{2\Bbnd\stk^5\Wbnd}{\stg^2}\|\Delta_t\|_2.
\end{equation}

\textbf{Bound on $\|\lqH_{t+1}^{-1}\lqH_t\|_2$.} Recall that $\lqH_t = \cov_t^{1/2}$. In order to bound the required term, we proceed as follows:
\begin{align*}
  \En\|\cov_{t+1}^{-1/2}\cov_t^{1/2} \|_2^2 &= \En\|\cov_{t+1}^{-1/2}\cov_t\cov_{t+1}^{-1/2}\|\\
  &\leq \En\|\cov_{t+1}^{-1/2}\cov_{t+1}\cov_{t+1}^{-1/2} \|_2 + \En\|\cov_{t+1}^{-1/2}(\cov_{t+1} - \cov_t)\cov_{t+1}^{-1/2}\|\\
  &\leq 1 + \frac{\En\|\cov_{t+1} - \cov_{t}\|_2}{\sigl} \\
  &\leq 1 + \frac{2\Bbnd\stk^5\Wbnd}{\sigl\stg^2}\stK \\
  &\stackrel{\1}{\leq} 1 + {\tilde{\stg}}\;
\end{align*}
where we bound the term $\|\cov_{t+1} - \cov_t\|_2$ using Eq.~\eqref{eq:lqr_cov_bnd} and $\1$ follows by setting $\regp \leq \frac{\tilde{\stg}\stg^2\sigl}{c\Bbnd\Wbnd\lip\stk^6(\adim\sdim)^2}$. Finally, using the fact that $\sqrt{1+x} \leq 1+ x/2$ for $x \in [0,1]$, we have that $\En\|\lqH_{t+1}^{-1}\lqH_t\|_2 \leq 1 + \tilde{\stg}/2$.

\paragraph{Ergodic stability parameters.} We now proceed to obtain an upper bound on the ergodic stability parameters. Before doing so, we obtain some auxiliary results which will be useful in establishing the final bound.

\textbf{Bound on $\|\hat{\cov}_t - \cov_t \|_2$.} We will now obtain a bound on the difference between the observed covariance $\hat{\cov}_t$ when a sequence of ERMs are played and the steady-state covariance matrix $\cov_t$. Let us set some notation before we begin with bounding this.
\begin{equation*}
  \Delta_{x, t} \defn \lqH_t^{-1}(\hat{\cov}_t - \cov_t)(\lqH_t^{-1})^\top \quad \text{and} \quad \En_\pert\|\cov_t - \cov_{t+1}\|_2 \leq \tilde{\regp}.
\end{equation*}
We then have the following recursion for the term $\Delta_{x,t}$ with the expectation with respect to the sampling of the noise variable $\pert$:
\begin{align*}
  \En\|\Delta_{x, t+1}\|_2 &\leq \En\|(\lqH_{t+1}^{-1}\lqH_t\lqL_t)\Delta_{x, t}(\lqH_{t+1}^{-1}\lqH_t\lqL_t)^\top \|_2 + \En\|(\lqH_{t+1}^{-1})(\cov_t - \cov_{t+1})((\lqH_{t+1}^{-1})^\top \|_2\\
  &\leq \En\|\lqL_t\|_2^2\|\lqH_{t+1}^{-1}\lqH_{t}\|_2^2\|\Delta_{t,x}\|_2 + \frac{\tilde{\regp}}{\alpha^2}\\
  &\stackrel{\1}{\leq} \left(1-\frac{\tilde{\stg}}{2}\right)^2 \En\|\Delta_{x,t} \|_2 + \frac{\tilde{\regp}}{\alpha^2}\\
  &\leq e^{-\tilde{\stg}t}\|\Delta_{x,1}\|_2 + \frac{\tilde{\regp}}{\alpha^2\tilde{\stg}}\;,
\end{align*}
where $\1$ follows from the bound on $\|\lqL_t\| \leq 1-\tilde{\stg}$ and $\|\lqH^{-1}_{t+1}\lqH_t \|_2 \leq (1+\tilde{\stg}/2)$. Substituting the value for $\Delta_{x,t}$ in the above bound, we get that:
\begin{equation}\label{eq:lqr_cov_samp}
  \En\|\cov_{t+1} - \hat{\cov}_{t+1} \|_2 \leq \frac{\Hb^2}{\Ha^2}\left(e^{-\tilde{\stg}t}\En\|\hat{\cov}_1 - \cov_1\| + \frac{\tilde{\regp}}{\tilde{\stg}} \right).
\end{equation}

Let us now bound the ergodic stability parameters $\stables_{\RERM, t}$ as
\begin{small}
\begin{align}\label{eq:lqr_mix_erm}
  \stables_{\RERM, t} &= \left\lvert\Eu{\pert}\Eu{w}[\loss(\f_t, \stf{t}{\f_{1:t-1}},\z )] -\Eu{\pert}[\losss(\f_t, \z)]\right\rvert \nonumber\\
  &= \left\lvert\Eu{\pert}\left[\trace((\lqQ + \K_{t, \pert}^\top\lqR\K_{t, \pert}^\top)(\hat{\cov}_{t, \pert} - \cov_{t, \pert})) \right]\right\rvert\nonumber\\
  &\leq \sdim(\Qbnd+\stk^2\Rbnd)\En_{\pert}\|\hat{\cov}_{t, \pert} - \cov_{t, \pert} \|_2\nonumber\\
  &\stackrel{\text{Eq.~\eqref{eq:lqr_cov_samp}}}{\leq} \sdim(\Qbnd+\stk^2\Rbnd)\cdot\frac{\Hb^2}{\Ha^2}\left(e^{-\tilde{\stg}t}\En\|\hat{\cov}_1 - \cov_1\| + \frac{\tilde{\regp}}{\tilde{\stg}} \right)\nonumber\\
  &\leq \sdim(\Qbnd+\stk^2\Rbnd)\cdot\frac{\Hb^2}{\Ha^2}\left(e^{-\tilde{\stg}t}\cdot\frac{2\Wbnd\stk^2}{\stg} + \regp\cdot\frac{c\Bbnd\stk^6\Wbnd\sdim^2\adim^2\lip}{\tilde{\stg}\stg^2}\right)\;,
\end{align}
\end{small}
where $\regp > 0$ is a free parameter corresponding to the noise in the perturbation $\pert$.

\paragraph{Bound on the value.}
Having established upper bounds on the mixing gap and the ergodic stability parameters, we now bound the value $\val_{\lqr, \T}$ as
\begin{align}\label{eq:lq_exact}
  \val_{\lqr, \T}(\Flq, \Z, \dyn, \loss) &\stackrel{\1}{\leq} \sum_{t=1}^\T \stables_{\RERM, t} + 2\Rseq_{T}(\losss\com \Flq) + \sup_{\f \in \Flq} \sum_{t=1}^\T\left\lvert\losscf(\f, \z_t, t) - \losss(\f, \z)\right\rvert + \frac{\stk\adim\sdim}{\regp}\nonumber\\
 &\overset{\text{Eq.}~\eqref{eq:lqr_unif_mix}}{\leq} \sum_{t=1}^\T \stable_{\RERM, t}^{\lqr} + 2\Rseq_{\T}(\losss\com \Flq) + (\Qbnd+\stk^2\Rbnd)\cdot \frac{\trw\stk^2}{\stg^2} + \frac{\stk\adim\sdim}{\regp}\nonumber\\
 &\overset{\text{Eq.}~\eqref{eq:lqr_mix_erm}}{\leq} \sdim(\Qbnd+\stk^2\Rbnd)\cdot\frac{\Hb^2}{\Ha^2}\left(\frac{2\Wbnd\stk^2}{\stg^2} + \regp\T\cdot\frac{c\Bbnd\stk^6\Wbnd\sdim^2\adim^2\lip}{\tilde{\stg}\stg^2}\right) \nonumber\\
 &\quad +  2\Rseq_{\T}(\losss\com \Flq) + (\Qbnd+\stk^2\Rbnd)\cdot \frac{\trw\stk^2}{\stg^2} + \frac{\stk\adim\sdim}{\regp}
\end{align}
where $\1$ follows from the fact that $\En[\pert_i] = 1/\regp$.

To obtain a bound on the sequential Rademacher complexity of the class, observe the the matrices $\K \in \real^{k \times d}$. Also, by Lemma~\ref{lem:lq_lip}, we have that the loss $\losss$ is bounded by $B_{\sf max}$ and Lipschitz with respect to policies $\K$ with constant $\lip$. Using a standard covering number argument, one can get an $\epsilon$-net of the class $\Flq$ in the frobenius norm with at most $O(dk(\frac{1}{\epsilon})^{d k})$ elements. Given this cover, one can upper bound the complexity as
\begin{equation*}
  \Rseq_{\T}(\losss\com \Flq) \leq cB_{\sf max}\sqrt{ k d\cdot \T\log(k d \T\lip )}
\end{equation*}
for some universal constant $c>0$. Setting $\regp = O(1/\sqrt{T})$ concludes the proof of the corollary.
\end{proof}

\subsubsection{Proof of Lemma~\ref{lem:lq_lip}}
We establish both parts of the claim separately.

\paragraph{Boundedness of stationary loss.} Consider the loss $\losss$ given by
\begin{align*}
  \losss(\K, \z) &= \trace[(Q+\K^\top R \K)X_\K]\\
  &\stackrel{\1}{\leq} C(1+\kappa^2)\|X_\K\|_2\\
  &\stackrel{\2}{\leq} C(1+\kappa^2)\frac{\sig_w\kappa^2}{\gamma},
\end{align*}
where inequality $\1$ follows from an application of von Neumann's trace inequality and the trace bounds on the matrices $Q$ and $R$, and step $\2$ follows from equation~\eqref{eq:cov_bnd_lq}.

\paragraph{Lipschitzness of stationary loss.} For any two matrices $\K_1, \K_2 \in \Flq$ and instance $\z \in \Z$,  consider the difference between the stationary losses
\begin{align*}
  |\losss(\K_1, \z) - \losss(\K_2, \z)| &\leq |\trace[Q(X_{\K_1} - X_{\K_2}]| + |\trace[R(\K_1\cov_{\K_1}\K_1^\top - \K_2\cov_{\K_2}\K_2^\top]|\\
  &\leq C\left((1+\kappa^2)\|\cov_{\K_1} - \cov_{\K_2}\|_2 +  \frac{2\kappa^3\sigma_w}{\gamma}\|\K_1 - \K_2\|_2\right)\\
  &\stackrel{\1}{\leq} 4C(1+\kappa^2)\frac{\sigma_b\kappa^5\sigma_w}{\gamma^2} \|\K_1 - \K_2\|_2\;,
\end{align*}
where step $\1$ follows from equation~\eqref{eq:lqr_cov_bnd}. This concludes the proof. \qed

\subsection{Online adversarial tracking}
The problem of online tracking of adversarial targets in Linear Quadratic Regulators was first posed in Abbasi et al.~\cite{abbasi2014}. The problem setup involves a state space given by $\real^\sdim$ and a action space $\real^\adim$. The sequential game proceeds as follows starting from state $\st_1 = 0$

\noindent On round $t = 1, \ldots, \T,$\vspace{-2mm}
\begin{itemize}
  \item the learner selects a policy $\f_t \in \Ftrack$ and adversary selects $\z_t \in \Z = \real^d$ such that $\|\z_t\|_2 \leq \zbnd$
  \item the learner receives loss $\loss(\f_t, \st_t, \z_t) = (\st_t - \z_t)^\top \lqrQ (\st_t - \z_t) + \f_t(\st_t)^\top \f_t(\st_t)$
  \item the state of the system transitions to $\st_{t+1} = \lqrA \st_t + \lqrB \ac_t$
\end{itemize}
where the matrices $\lqrA, \lqrB, \lqrQ$ are known in advance to the learner and the adversary. In addition, the matrix $\lqrQ$ is positive definite, the pair $(\lqrA, \lqrB)$ is assumed to be controllable while the pair $(\lqrA, \lqrQ^{1/2})$ is assumed to be observable. The comparator policy class $\Ftrack$ is assumed to be the following restricted class of linear policies:
\begin{equation*}
  \Ftrack = \{\f = (\K, \bias)\;|\; \|A+BK\|_2 \leq \rholq; \|\K\|_2\leq \Kbnd; \|\bias\|_2\leq \cbnd\}\;,
\end{equation*}
such that the action is given by $\ac_t = \K_t \st_t + \bias_t$. For this setup, as we establish later, the stationary loss for any policy $\f = (\K, \bias)$ is given by:
\begin{equation*}
    \losss(\f, \z) = (\st^\f_* - z)^\top \lqrQ (\st^\f_* - z) + \|\K\st_*^\f + \bias \|_2^2, \quad \text{where} \quad \st^\f_* = {(I - (\lqrA +\lqrB\K))^{-1}\lqrB}\bias
  \end{equation*}
Given these preliminaries, we obtain a bound on the value $\val_{\mathsf{tar}, \T}$ through the following corollary.

\begin{corollary}[Online Tracking]\label{cor:track}
For the online adversarial tracking sequential game, the value $\val_{\mathsf{tar}, \T}$ is bounded by:
\begin{equation*}
  \val_{\mathsf{tar}, \T}(\Ftrack, \Z, \dyn) \leq \O\left(\sqrt{T\log(\T)}\right)\;,
\end{equation*}
where the $\O$ notation hides the dependence of the bound on problem-specific parameters (see equation~\eqref{eq:track_final} for the exact dependencies).
\end{corollary}

In contrast to the result obtained above, \cite{abbasi2014} provide an algorithm for which the regret for the above problem is bounded by $\O(\log^2 \T)$. Obtaining such fast rates in our general framework is an interesting open problem.

\begin{proof}[Proof of Corollary~\ref{cor:track}]
  Our general strategy is to obtain bounds on the the mixing gap and the ergodic stability parameters for certain regularized ERMs. We then use these upper bounds together with Corollary~\ref{cor:val_ergo} to establish the required upper bound.

  \paragraph{Bound on mixing gap.}
  Consider any policy $\f = (\K, \bias)$. We are interested in obtaining a bound on the mixability for this function as:
  \begin{equation*}
     \left\lvert\losscf(\f, \z, t) - \losss(\f, \z)\right\rvert \leq \stable_{\f,t}.
  \end{equation*}
  Let us abbreviate the state $\stf{t}{\f^{(t-1)}}$ by $\st_t^\f$. If we run any policy with the linear dynamics, a steady state $\st^\f_*$ is reached with
  \begin{equation*}
    \st^\f_* = (\lqrA + \lqrB\K)\st^{\f}_* + \lqrB\bias \quad \text{ and therefore } \quad \st^\f_* = \underbrace{(I - (\lqrA +\lqrB\K))^{-1}\lqrB}_{\defn\Meff_\K} \bias = \Meff_\K\bias.
  \end{equation*}
Then, the corresponding loss at this stationary point is given as
  \begin{equation*}
    \losss(\f, \z) = (\st^\f_* - z)^\top \lqrQ (\st^\f_* - z) + \|\K\st_*^\f + \bias \|_2^2.
  \end{equation*}
In order to obtain a bound on the mixing gap, we analyze the convergence of the state $\st_{t+1}^\f$ to the stationary state $\st^\f_*$.
  \begin{align*}
    \|\st_{t+1}^\f - \st^f_* \|_2 &= \|\sum_{s=0}^\infty (\lqrA - \lqrB\K)^s\lqrB\bias - \sum_{s=0}^{t-1}(\lqrA - \lqrB\K)^s\lqrB\bias \|_2\\
    &\stackrel{\1}{\leq} \rholq^t\|\Meff_\K\bias \|_2\;,
  \end{align*}
  where $\1$ follows from the assumption that $\|\lqrA+\lqrB\K\|_2 \leq \rholq$.

Next, we consider a bound on the norm of the state $\st^\f_t$ that is reached by any policy.
  \begin{align*}
    \|\st_{t+1}^\f\|_2 &= \|(\lqrA + \lqrB\K)\st_{t}^\f + \lqrB\bias\|_2\\
    &\stackrel{\1}{=} \|\sum_{s=1}^t (\lqrA + \lqrB \K)^{t-s} \lqrB\bias\|_2\\
   &\stackrel{\2}{\leq} \frac{\|\lqrB\|\cbnd}{1-\rholq} \defn c_\st,
  \end{align*}
  where $\1$ follows from recursively applying the definition of the state evolution and the fact that $\st_1 = 0$, and $\2$ follows from the assumption that $\|\lqrA+\lqrB\K\|_2 \leq \rholq$. Having established the above, we now proceed to obtain a bound on the mixing gap as
  \begin{small}
  \begin{align}\label{eq:track_mix}
  \left\lvert\losscf(\f, \z, t) - \losss(\f, \z)\right\rvert &= \left\lvert (\st_{t}^\f - \z)^\top \lqrQ(\st_{t}^\f - \z) + \|\K\st_t^\f + \bias\|_2^2 - (\st^\f_* - z)^\top \lqrQ (\st^\f_* - z) + \|\K\st_*^\f + \bias \|_2^2  \right\rvert\nonumber\\
  &\stackrel{\1}{\leq} |(\st_t^\f - \st_*^\f)^\top\lqrQ(\st_t^\f - z)| + |(\st_t^\f - \st_*^\f)^\top\lqrQ(\st_*^\f - z)| + \|\K(\st_t^\f - \st_*^\f)\|_2^2 \nonumber\\
  &\quad  + 2\inner{\K\st_*^\f +\bias}{\K(\st_t^\f - \st_*^\f)}\nonumber\\
  & \leq 2\|\lqrQ\|(c_\st + \zbnd)\cdot \|\st_t^\f - \st_*^\f \|_2 + \Kbnd^2\cdot \|\st_t^\f - \st_*^\f \|_2^2 + 2\Kbnd(\Kbnd c_\st + \cbnd)\cdot \|\st_t^\f - \st_*^\f \|_2\nonumber\\
  &\leq \rholq^{t-1}\cdot \underbrace{c_\st(2\|\lqrQ\|(c_\st + \zbnd) + 2\Kbnd(\Kbnd c_\st + \cbnd))}_{\stablea} + \rholq^{2(t-1)}\cdot \underbrace{c_\st^2\Kbnd^2}_{\stableb}\;,
  \end{align}
  \end{small}
  where $\1$ follows from adding and subtracting $\st_*^\f$ in both the terms followed by an application of triangle inequality. For ease of presentation, let us represent the above using constants $\stablea$ and $\stableb$ with the knowledge that these depend on the underlying problem parameters but independent of the underlying policy $\f$, that is,
  \begin{equation}\label{eq:track_unifmix}
  \left\lvert\losscf(\f,\z, t) - \losss(\f, \z)\right\rvert \leq \rholq^{t-1}\stablea + \rholq^{2(t-1)}\stableb
  \end{equation}

  \paragraph{Ergodic stability parameters.}
  For obtaining a bound on the ergodic stability parameters, we require a few structural results for the loss $\losss$ defined above. We present these next and defer their proofs to the end of the section.

  \begin{lemma}[Equivalence of Tracking Cost]\label{lem:track_equiv} Consider any policy $\f = (\K, \bias)$ and another stable matrix $\K'$. There exists an $\bias'$ such that we have $\losss((\K, \bias), \z) = \losss((\K', \bias'), \z)$ such that
    \begin{equation*}
      \|\bias'\|_2 \leq 2\cbnd\left(\frac{\|B\|\Kbnd}{(1-\rholq)} + 1\right).
    \end{equation*}
  \end{lemma}

  Thus going forward, we consider the ERM procedure on the  class of functions $\Ftrack'(K)$, parameterized for a fixed stable policy $\K$, where the bias
  \begin{equation*}\bias \leq \cbnd' \defn \max\left(2\cbnd\left(\frac{\|B\|\Kbnd}{1-\rholq} + 1\right), \cbnd\frac{\zbnd\|\lqrQ\|\|\lqrB\|}{\sigma_\lqrQ\sigma^2_{\lqrB}(1-\rholq)}\right),\end{equation*}
  with $\sigma_X$ denotes the smallest non-zero singular value of $X$.
  Note that the conclusions of Corollary~\ref{cor:val_ergo} are still valid with the mixability parameters $\stable_{\Ftrack'(\K), t}$.  The next lemma establishes the stability of the ERM solutions obtained in consecutive rounds.

  \begin{lemma}\label{lem:track_stable_bias}
  Fix any $\rholq$-stable policy $\K$. The ERM solutions $\f_{\ERM, t} = (\K, \bias_t)$ and \mbox{$\f_{\ERM, t+1} =(\K, \bias_{t+1})$} satisfy the following stability bound:
  \begin{equation*}
    \|\bias_t - \bias_{t+1}\| \leq \frac{2\zbnd}{t+1}\cdot \frac{\cbnd\|\lqrQ\|\|\lqrB\|}{\sigma_\lqrQ\sigma^2_{\lqrB}(1-\rholq)} \defn \frac{\psi_\bias}{t+1}
  \end{equation*}
  \end{lemma}

  Having established the stability bound above, one can proceed in a manner similar to \cite[Lemma 8]{abbasi2014}, one can establish that for $t >\rholq\log(T)/(1-\rholq)$
  \begin{equation}\label{eq:track_perturb}
  \left\lVert\stf{t}{\f_{\ERM, 1:t-1}} - \st_*^{\f_{\ERM, t}} \right\rVert_2 \leq \underbrace{\frac{\|\lqrB\|\psi_\bias}{1-\rholq}\cdot \frac{2\log t}{t - \log t} + \rholq^{t-1}\frac{\|\lqrB\|\cbnd'}{1-\rholq}}_{\psi_{\st, t}}.
    % \left\lVert\stf{t}{\f_{\ERM, 1:t-1}} - \st_*^{\f_{\ERM, t}} \right\rVert_2 \leq \frac{8\|\lqrB\|}{1-\rholq} \left[\frac{\cbnd'\log \T}{1-\rholq} + \psi_\bias \log^2 \T\right]
  \end{equation}
  Having established the above, we can now obtain a bound on the ergodic stability parameters for the ERM procedure for the online tracking problem. The calculation is similar to the one done for the mixing gap (see Eq.~\eqref{eq:track_mix}).
  \begin{align*}
  \left\lvert\loss(\f_{\ERM, t}, \stf{t}{\f_{\ERM, 1:t-1}}, \z_t) - \losss(\f_{\ERM, t}, \z_t)\right\rvert &\leq 2\|\lqrQ\|(c_\st + \zbnd)\cdot \|\stf{t}{\f_{\ERM, 1:t-1}} - \st_*^{\f_{\ERM, t}} \|_2 \\
  &\quad + \Kbnd^2\cdot \|\stf{t}{\f_{\ERM, 1:t-1}} - \st_*^{\f_{\ERM, t}} \|_2^2 \\
  &\quad + 2\Kbnd(\Kbnd c_\st + \cbnd')\cdot \|\stf{t}{\f_{\ERM, 1:t-1}} - \st_*^{\f_{\ERM, t}} \|_2\\
  &\leq \psi_{\st, t}(2\|\lqrQ\|(c_\st + \zbnd) + 2\Kbnd(\Kbnd c_\st + \cbnd')) + \psi_{\st, t}^2\Kbnd^2\;,
  \end{align*}
  where we have substituted the bound for $\|\stf{t}{\f_{\ERM, 1:t-1}} - \st_*^{\f_{\ERM, t}}\|$ from Eq.~\eqref{eq:track_perturb}. Thus, we that the ERM ergodic stability parameters are
  \begin{equation}\label{eq:tar_erm_mix}
    \stables_{\ERM, t}= \psi_{\st, t}\cdot(2\|\lqrQ\|(c_\st + \zbnd) + 2\Kbnd(\Kbnd c_\st + \cbnd)) + \psi_{\st, t}^2\Kbnd^2\;.
  \end{equation}

\paragraph{Bound on the value.}  We now proceed to obtain a  bound on the value $\val_{\textsf{tar}, \T}$, beginning from the statement of Corollary~\ref{cor:val_ergo}.
  \begin{align}\label{eq:track_final}
  \val_{\mathsf{tar}, \T}(\Ftrack, \Z, \dyn, \loss) &\leq \sum_{t=1}^\T \stable_{\ERM, t}^{\mathsf{tar}} + 2\Rseq_{T}(\losss\com \Ftrack'(\K)) + \sup_{\f \in \Flq} \sum_{t=1}^\T\left\lvert\losscf(\f, \z_t, t) - \losss(\f, \z)\right\rvert\nonumber\\
  &\overset{\text{Eq}.{~\eqref{eq:track_unifmix}}}{\leq} \sum_{t=1}^\T \stable_{\ERM, t}^{\mathsf{tar}} + 2\Rseq_{T}(\losss\com \Ftrack'(\K)) + \frac{\stablea}{1-\rholq} + \frac{\stableb}{ 1-\rholq^2}\nonumber\\
  &\overset{\text{Eq.}~\eqref{eq:tar_erm_mix}}{\leq} \frac{\rholq\log T}{1-\rholq}\left(\|\lqrQ\|(c_\st+\zbnd)^2 +(\Kbnd c_\st+ \cbnd')^2 \right) + 2\Rseq_{T}(\losss\com \Ftrack'(K)) + \frac{\stablea}{1-\rholq} + \frac{\stableb}{1-\rholq^2}\nonumber\\
  &\quad + \left(\frac{2\log^2 \T\|\lqrB\|\psi_\bias}{1-\rholq} + \frac{\|\lqrB\|\cbnd'}{(1-\rholq)^2}\right)\cdot(2\|\lqrQ\|(c_\st + \zbnd) + 4\Kbnd(\Kbnd c_\st + \cbnd))\;,
  \end{align}
  where in the last inequality, we have upper bounded the lower order term $\psi_{\st, t}^2$ by $\psi_{\st, t}$.

  Finally, one can obtain a bound on the sequential complexity by noting that the loss $\losss$ is bounded since the state $\|\st_*^\f\|_2 \leq c_x$ and is Lipschitz in the bias parameter $\eta$ with respect to the $ell_2$ norm. Using an argument similar to that from the proof of Corollary~\ref{cor:lqr}, we have
  \begin{equation*}
    \Rseq_{T}(\losss\com \Ftrack'(K)) \leq O\left(\sqrt{d\T\cdot \log(d\T)}\right).
  \end{equation*}
  Substituting this bound in equation~\eqref{eq:track_final} establishes the corollary.
\end{proof}

\subsubsection{Proof of Lemma~\ref{lem:track_equiv}}
  Let $\valf_{\f, \z}(\st, \ac)$ represent the value function for state-action pair $(\st, \ac)$ with respect to policy $\f$ and loss function $\z$. Following Lemma 12 from~\cite{abbasi2014} we have that:
  \begin{equation*}
    \losss(\f', \z) - \losss(\f, \z) = \valf_{\f', \z}(\st_*^{\f}, \ac_{\f'}) - \valf_{\f', \z}(\st_*^{\f}, \ac_{\f}).
  \end{equation*}
  The action taken by policy $\f$ is given by $\ac_{\f} = \K\st_*^\f + \bias$, while that taken by $\f'$ is given by $\K'\st_{*}^\f + \bias'$. If we set the value of $\bias'$ as:
  \begin{equation*}
    \bias' = (\K -\K')\st_*^\f + \bias \quad \Rightarrow \quad \losss(\f', \z) = \losss(\f, \z).
  \end{equation*}
  Also, note that one can obtain an upper bound on the norm of $\bias'$ as $\|\bias' \|_2 \leq 2c_\st\Kbnd + \cbnd$ using the bounds on the state $\st_*^\f$.
\qed

\subsubsection{Proof of Lemma~\ref{lem:track_stable_bias}}

  We begin by characterizing the  ERM solution ${\bias}_t$ as follows:
\begin{align*}
  {\bias}_t &= \argmin_{\bias} \left( \sum_{t=1}^\T \Eu{\z_t \sim \ad_t}\left[ (\st^\f_* - \z_t)^\top \lqrQ (\st^\f_* - \z_t) + \|\K\st_*^\f + \bias \|_2^2\right]\right)\\
  &= \argmin_{\bias} \left( \sum_{t=1}^\T \Eu{\z_t \sim \ad_t}\left[ (\st_*^\f)^\top[\lqrQ + \K^\top \K]x_*^\f -2\z_t\lqrQ\st_*^\f + \bias^\top \bias + \bias^\top\K\st_{*}^\f + (\K\st_{*}^\f)^\top \bias\right]\right)\\
  &= \argmin_{\bias} \left( \sum_{t=1}^\T \Eu{\z_t \sim \ad_t}\left[ \bias^\top(\underbrace{\Meff_\K^\top(\lqrQ + \K^\top\K)\Meff_\K + I + \K\Meff_\K + \Meff_\K^\top \K^\top }_{W})\bias - 2\z_t^\top \lqrQ\Meff_\K\bias \right]\right)\\
  &= W^{-1}\left(\frac{1}{t}\sum_{s=1}^t \Eu{\z_s\sim \ad_s}[\z_s] \right)\lqrQ\Meff_\K\;,
\end{align*}
where the last equality follows by minimizing the quadratic and the existence of the inverse because $W \succeq \lqrB^\top\lqrQ\lqrB$.\ and the fact that $\bias$ does not lie in the null space of $\lqrB$ (it is always better to set it to zero in that case). This ensures that $\|\bias_t\|_2 \leq \cbnd\frac{\zbnd\|\lqrQ\|\|\lqrB\|}{\sigma_\lqrQ\sigma^2_{\lqrB}(1-\rholq)}$ and hence the policy $\f_t = (\K, \bias_t) \in \Ftrack'$. We can now obtain the stability bounds as:
\begin{align*}
  \|\bias_t - \bias_{t+1}\|_2 &= \left\lVert W^{-1} \left(\frac{1}{t}\sum_{s=1}^t \Eu{\z_s\sim \ad_s}[\z_s]  - \frac{1}{t+1}\sum_{s=1}^{t+1} \Eu{\z_s\sim \ad_s}[\z_s]\right) \lqrQ\Meff_\K\right\rVert_2\\
  &\leq \frac{2\zbnd}{t+1}\cdot \frac{\cbnd\|\lqrQ\|\|\lqrB\|}{\sigma_\lqrQ\sigma^2_{\lqrB}(1-\rholq)}\;,
\end{align*}
where the final inequality follows from using the bound on $\|\Meff_\K\|_2$ as well as the fact that $\|\z\|_2 \leq \zbnd$. This concludes the proof of the lemma.
\qed

\subsection{Online non-linear control}
In this section, we look at a non-linear control problem: one formed by extending the LQR problem above to have non-linear deterministic dynamics. We parameterize the dynamics using a non-linear function $\signl : \real^\sdim \mapsto \X$ as follows:
\begin{equation*}
  \st_{t+1} = \signl[\lqA\st_t + \lqB\ac_t]\;,
\end{equation*}
We assume that the function $\signl$ is $1$-Lipschitz and $\|\signl(\st)\|\leq \xbnd$ for some $\xbnd > 0$. This is done to ensure that the dynamics satisfy the ergodicity assumption. We now proceed to define the associated policy class $\Fnl$ as
\begin{equation*}
  \Fnl = \{\f_\param \; | \; \param \in \real^{\pdim}, \|\param\|_2\leq \pbnd, \|[\lqA\st + \lqB\f_\param(\st)] - [\lqA\st' + \lqB\f_\param(\st')]\|_2 \leq (1-\stg)\|\st - \st' \|_2\}\;,
\end{equation*}
where the last condition on the function class establishes a stability condition.  In addition, we assume that the function class $\Fnl$ satisfies a Lipschitz property:
\begin{equation*}
\|\f_\param(\st) - \f_{\param'}(\st)\|_2 \leq \lipf \|\param - \param'\|_2 \quad \text{for all } \quad \st \in \X\;.
\end{equation*}
The above basically means that if two parameters $\param, \param'$ are close in the parameter space, then the policies parameterized by them are uniformly close for all states. Note that the class of linear policies $\Ftrack$ (without the bias term) defined for the adversarial tracking problem satisfies the above properties. We next outline the learning protocol, with the game starting with $\st_1 = 0$.

\noindent On round $t = 1, \ldots, \T,$\vspace{-2mm}
\begin{itemize}
  \item the learner selects policy $\f_t \in \Fnl$ and the adversary selects $\z_t \in \Z$.
  \item the learner receives loss $\loss(\f_t, \st_t, \z_t) \in [0,1]$
  \item the state of the system transitions to $\st_{t+1} = \signl[\lqrA \st_t + \lqrB \ac_t]$
\end{itemize}
For the above setup, we shortly establish that the stationary loss for any policy $\f$ is given by
\begin{equation}\label{eq:losss_nl}
  \losss(\f, \z) = \loss(\f, \st^\f_*, \z) \quad \text{where}\quad \st^\f_* = \signl[\lqA\st^\f_*+ \lqB\f(\st^\f_*)]\;,
\end{equation}
where the existence of the fixed point is guaranteed by the stability assumption on the function class in conjunction with the Brouwer fixed-point theorem. In the following lemma, we show that the loss function $\losss$ above is Lipschitz with respect to the parameter $\param$.
\begin{lemma}\label{lem:nl_lip}
The loss function $\losss$ given in equation~\eqref{eq:losss_nl} satisfies
\begin{equation*}
  |\losss(\f_{\param_1}, \z) - \losss(\f_{\param_2}, \z)| \leq \underbrace{\left(\liplf + \lipx\frac{\|B\|_2\lipf}{\gamma}\right)}_{\lip}\|\param_1 - \param_2\|_2\quad \text{for all } \f_{\param_1}, \f_{\param_2} \in \Fnl,\; \; \z \in \Z.
\end{equation*}
\end{lemma}
We prove the lemma at the end of the section. Taking this as given, we now establish the learnability of the function class $\Fnl$ in the following corollary.

\begin{corollary}[Online Non-Linear Control]\label{cor:nl}
  Consider any value of $\regp > 0$ and loss function $\loss$ which is $\lipx$-Lipschitz in the state space and $\liplf$-Lipschitz in the parameter space with respect to the $\ell_2$ norm. For the online non-linear control problem described above, we have that the value
\begin{equation*}
    \val_{\nlsf, \T}(\Fnl, \Z, \dyn) \leq \O\left(\sqrt{T\log(\T)} \right)\;,
  \end{equation*}
  where the $\O$ notation hides the dependence of the bound on problem-specific parameters (see equation~\eqref{eq:nl_final} for the exact dependencies).
\end{corollary}
Notice that the above corollary establishes an upper bound of $\Ot(\sqrt{T})$ for the value $\val_{\nlsf, \T}$. Thus, despite the fact that the  setup does not have the nice structure of the LQR problem, we are able to establish the learnability of the class $\Fnl$ in the online learning with dynamics framework.

\begin{proof}[Proof of Corollary~\ref{cor:nl}]
  We begin by establishing a bound on the mixing gap for the class $\Fnl$ as well as the ERM ergodic stability parameters. Throughout this section, we would often drop the dependence of the function $\f_\theta$ on the underlying parameter $\theta$ when it is clear from the context.

\paragraph{Bound on mixing gap. }
  Consider any policy $\f \in \F_{\nlsf}$ and the associated stationary loss
  \begin{equation*}
    \losss(\f, \z) = \loss(\f, \st^\f_*, \z) \quad \text{where}\quad \st^\f_* = \signl[\lqA\st^\f_*+ \lqB\f(\st^\f_*)]\;,
  \end{equation*}
  where the non-linearity $\signl$ is applied element-wise to its arguments. Consider now the difference between the stationary and the counterfactual loss as
  \begin{align}\label{eq:nl_unif_mix}
    \left\vert\losscf(\f, \z, t) - \loss(\f, \st^\f_*, \z)\right\vert &\stackrel{\1}{\leq} \lipx \|\st^\f_t - \st^\f_* \|_2\nonumber\\
    &\leq \lipx \|\signl[\lqA\st_{t-1}^\f +\lqB\f(\st_{t-1}^\f)] - \signl[\lqA\st_{*}^\f +\lqB\f(\st_{*}^\f))] \|_2\nonumber\\
    &\stackrel{\2}{\leq} \lipx(1-\stg)\|\st^\f_{t-1} - \st^\f_* \|_2\nonumber\\
    &\stackrel{\3}{\leq} 2\lipx\xbnd(1-\stg)^{t-1},
  \end{align}
  where $\1$ follows from the $\lipx$-Lipschitzness of the loss function, $\2$ follows from the  fact that $\f \in \Fnl$, and $\3$ follows from the boundedness of the states.

  \paragraph{Regularized ERM.} Similar to the proof for Corollary~\ref{cor:lqr}, we consider the following FTPL based ERM
  \begin{equation*}
    \param_{t, \pert} = \argmin_{\param \in \FTnl}\left(\sum_{s=1}^t\Eu{\z_s \sim \ad_s}\left[\losss(\f_\param, \z)\right] - \inner{\pert}{\param}\right),
  \end{equation*}
  where $\pert \in \real^{\pdim}$ such that each coordinate of $\pert\sim \expd(\regp)$, the exponential distribution with parameter $\regp > 0$. We establish in Lemma~\ref{lem:nl_lip} that the loss functions defined by $\losss(\cdot, \z)$ are $\lipnl$-Lipschitz in the parameter $\param$ and hence the iterates satisfy
  \begin{equation*}
    \Eu{\pert}\left[\|\param_{t, \pert} - \param_{t+1, \pert} \|_1 \right] \leq c\regp\cdot \lipnl (\pdim)^2\stk \defn \stp\;,
  \end{equation*}
  where the norm above is defined element-wise.

  \paragraph{Ergodic stability parameters.} With these set of regularized ERMs, we proceed to now bound the stability parameters of these solutions. Consider again the difference between the stationary and the instantaneous loss:
  \begin{equation}\label{eq:nl_bnd_mix_upp}
    \left\vert\Eu{\pert}\left[\loss(\f_t, \st_t[\f_{1:t-1}]), \z\right] - \Eu{\pert}\left[\losss(\f_t, \z) \right]  \right\vert \leq \lipx \Eu{\pert} \left[\|\st_t[\f_{1:t-1}] - \st_t^{\f_t} \|_2  + \|\st_*^{\f_t}- \st_t^{\f_t} \|_2 \right]\;,
  \end{equation}
  where the above inequality follows from the $\lipx$ Lipschitz property of the loss function in the state space. We have dropped the dependence of $\f$ on the noise perturbation $\pert$, underlying parameter $\param$ as well as the fact that these are RERM solutions.  In order to obtain the stability parameters, we proceed to obtain a  bound on the terms on the right.

\textbf{Bound on $\|\st_*^{\f_t}- \st_t^{\f_t} \|_2$.} The upper bound on this difference is similar to the one we obtained while bounding the mixing gap, the only difference being we have to handle the expectation with respect to the random perturbation $\pert$. Consider,
  \begin{align}\label{eq:nl_mix_st}
    \Eu{\pert}\left[\|\st_*^{\f_t}- \st_t^{\f_t}\|_2\right] &= \Eu{\pert}\left[\|\signl[\lqA\st_*^{\f_t} + \lqB\f_t(\st_*^{\K_t})] - \signl[\lqA\st_{t-1}^{\f_t} + \lqB{\f_t(\st_{t-1}^{\f_t})}]\|_2\right]\nonumber\\
    &\leq(1-\stg) \Eu{\pert}\left[\|\st_*^{\K_t}- \st_{t-1}^{\K_t}\|_2\right]\nonumber\\
   &\leq (1-\stg)^{t-1}\cdot2\xbnd\;,
  \end{align}
  where the sequence of inequalities follows since we have $\f_{\param, t, \pert} \in \Fnl$ for any sampling of the perturbation variables $\pert$.

  \paragraph{Bound on $\|\st_t[\f_{1:t-1}] - \st_t^{\f_t} \|_2$.} %In order to establish this bound, we introduce some notation first. We represent expand the state representation to make its dependence on the policies used to achieve it explicit as follows:
  %\begin{equation*}
  %  \stf{t}{\f_1, \f_2, \ldots, \f_{t-1}} = \st_t[\f_{1:t-1}] \quad \text{and} \quad \stf{t}{\f_t, \f_t, \ldots, \f_{t}} = \st_t^{\f_t}\;.
  %\end{equation*}
Consider a parameter $\tmix \geq 1$ to be specified later. We can then decompose the desired difference as follows:
  \begin{align*}
    \left\lVert\st_t[\f_{1:t-1}] - \st_t^{\f_t} \right\rVert_2 &= \sum_{i = 1}^{\tmix}\left(\left\lVert \stf{t}{\f_1, \ldots, \f_{t-i},\f_t, \ldots, \f_t } -  \stf{t}{\f_1, \ldots, \f_{t-i-1},\K_t, \ldots, \f_t }\right\rVert_2 \right) \\
    &\quad + \|\stf{t}{\f_1, \ldots, \f_{t-\tmix-1},\f_t, \ldots, \f_t } - \stf{t}{\f_t, \ldots, \f_t } \|_2\\
    &\leq \sum_{i= 1}^{\tmix}\left(\left\lVert \stf{t}{\f_1, \ldots, \f_{t-i},\f_t, \ldots, \f_t } -  \stf{t}{\f_1, \ldots, \f_{t-i-1},\f_t, \ldots, \f_t }\right\rVert_2 \right) + 2\xbnd(1-\stg)^{\tmix-1}\;,
  \end{align*}
  where the last inequality follows from a similar calculation as in equation~\eqref{eq:nl_mix_st}. We now focus on the terms in the summation above, focusing on a general term $i$. Let us redefine the state to be $\st_0^i = \stf{t-i}{\f_1, \ldots, \f_{t-i-1}}$. Now, denote by $\hat{\st}_j = \st_{t-i+j}[\f_{t-i}, \f_t, \ldots, \f_t]$ to be the state reached when we select $\f_{t-i}$ at the $(t-i)^{th}$ time instance, followed by $\f_t$ for $j-1$ steps. Similarly, $\tilde{\st}_j = \st_{t-i+j}[\f_{t}, \f_t, \ldots, \f_t]$ is the state reached when one begins from $\st_0^i$ and selects $\f_t$ for the next $j$ time steps. Bounding the sum above is equivalent to bounding the difference $\tilde{\st}_i - \hat{\st}_i$.
  \begin{align*}
    \| \tilde{\st}_{i} - \hat{\st}_i\|_2 &= \|\signl[\lqA\tilde{\st}_{i-1} + \lqB\f_t(\tilde{\st}_{i-1})] - \signl[\lqA\hat{\st}_{i-1} + \lqB\f_t(\hat{\st}_{i-1})] \|_2\\
    &\leq (1-\stg)^{i-1}\|\tilde{\st}_{1} - \hat{\st}_{1} \|_2\\
    &= (1-\stg)^{i-1}\|\signl[\lqA{\st}_{0}^i + \lqB\f_{t-i}({\st}_{0}^i)] - \signl[\lqA{\st}_{0}^i + \lqB\f_t({\st}_{0}^i)] \|_2\\
    &\leq (1-\stg)^{i-1}\cdot\|B\|_2\cdot \|\f_{t-i}({\st}_{0}^i) - \f_{t}({\st}_{0}^i)\|\\
    &\stackrel{\1}{\leq} (1-\stg)^{i-1}\cdot\|B\|_2\cdot \lipf\|\param_{t-i} - \param_{t}\|\\
    &\leq i(1-\stg)^{i-1}\cdot\|B\|_2\cdot \lipf\stp
  \end{align*}
  where $\1$ follows from the Lipschitz assumption on the function class in the parameter space. Setting $\tmix = t$ and summing up the above inequalities, we get,
  \begin{equation}\label{eq:bnd2}
      \left\lVert\st_t[\f_{1:t-1}] - \st_t^{\f_t} \right\rVert_2 \leq  \frac{\|B\|_2\cdot \lipf\stp}{\stg^2} + 2\xbnd(1-\stg)^{t-1}.
  \end{equation}
  Finally, substituting the bounds obtained in eq.~\eqref{eq:nl_mix_st} and eq.~\eqref{eq:bnd2} in eq.~\eqref{eq:nl_bnd_mix_upp}, we get that:
  \begin{equation}\label{eq:nl_mix_erm}
    \left\vert\Eu{\pert}\left[\loss(\f_t, \st_t[\f_{1:t-1}]), \z\right] - \Eu{\pert}\left[\losss(\f_t, \z) \right]  \right\vert \leq \lipx  \left( \frac{\|B\|_2\cdot \lipf\stp}{\stg^2} + 4\xbnd(1-\stg)^{t-1} \right) \defn \stables_{\RERM, t}\;
  \end{equation}

  \paragraph{Bound on the value.} Having established the mixing gap and the RERM stability parameters, we now upper bound the value $\val_{\nlsf, \T}(\Fnl, \Z, \dyn, \loss)$
  \begin{align}\label{eq:nl_final}
    \val_{\nlsf, \T}(\Fnl, \Z, \dyn, \loss) &\stackrel{\1}{\leq} \sum_{t=1}^\T \stables_{\RERM, t} + 2\Rseq_{T}(\losss\com \Fnl) + \sup_{\f \in \Fnl} \sum_{t=1}^\T\left\lvert\losscf(\f, \z_t, t) - \losss(\f, \z)\right\rvert\ + \frac{\pbnd\pdim}{\regp}\nonumber\\
   &\overset{\text{Eq.}~\eqref{eq:nl_unif_mix}}{\leq} \sum_{t=1}^\T \stables_{\RERM, t} + 2\Rseq_{\T}(\losss\com \Fnl) + \frac{2\lipx\xbnd}{\stg} + \frac{\pbnd\pdim}{\regp}\nonumber\\
   &\overset{\text{Eq.}~\eqref{eq:nl_mix_erm}}{\leq} \frac{\lipx\lipf\|\lqB\|_2}{\stg^2}\cdot \stp\T +  2\Rseq_{\T}(\losss\com \Fnl) + \frac{6\lipx\xbnd}{\stg} + \frac{\pbnd\pdim}{\regp}
  \end{align}
  where $\1$ follows from the fact that $\En[\pert_i] = 1/\regp$.

  Following the proof technique of Corollary~\ref{cor:lqr}, it suffices to establish the stationary loss is bounded (by definition) and is Lipschitz with respect to the underlying parameter (Lemma~\ref{lem:nl_lip}). Combining this with the fact that the parameter $\param \in \real^d$, we have that the sequential Rademacher complexity is bounded as
  \begin{equation*}
    \Rseq_{\T}(\losss\com \Flq) \leq c\sqrt{ d\cdot \T\log(d \T\lip )},
  \end{equation*}
for some universal constant $c>0$. Setting a value of $\regp  = 1/\sqrt{T}$ then concludes the proof of the corollary.
\end{proof}

\subsubsection{Proof of Lemma~\ref{lem:nl_lip}}
For any policies $\f_1 \defn \f_{\param_1} \in \Flq$ and $\f_2 \defn \f_{\param_2}\in \Flq$, and instance $\z \in \Z$, consider the difference in the stationary loss
\begin{align*}
  |\losss(\f_1, \z) - \losss(\f_2, \z)| &\leq |\loss(\f_1, \st^{\f_1}_*, \z) - \loss(\f_1, \st^{\f_2}_*, \z) | + |\loss(\f_1, \st^{\f_2}_*, \z) - \loss(\f_2, \st^{\f_2}_*, \z) | \\
  &\stackrel{\1}{\leq} \liplf\|\param_1 - \param_2\|_2 + \lipx\|\st_*^{\f_1} - \st_*^{\f_2} \|_2\\
  &\stackrel{\2}{\leq} \left(\liplf + \lipx\frac{\|B\|_2\lipf}{\gamma}\right)\|\param_1 - \param_2\|_2,
\end{align*}
where inequality $\1$ follows from the Lipschitz property of the loss function with respect to the policy and state space while inequality $\2$ follows from the Lipschitz property of the policy. This establishes the desired claim.
\qed

\subsection{Online LQR with adversarial disturbances}
In this section, we consider the example of an online learning with dynamics problem where the adversary is allowed to perturb the dynamics in addition to the adversarial losses at each time step. We will focus on the Linear-Quadratic setup where the dynamics function is linear and the costs quadratic in the state $\st_t$ and action $\ac_t$. Agarwal et al.~\cite{agarwal2019} studied a general version of this problem where they considered the convex cost functions with linear dynamics.

As in the Online LQR example in Section~\ref{app:ex-olqr}, we consider the class of linear policies $\Flq$ which are $(\kappa, \gamma)$-strongly stable. Given this policy class, the online learning with dynamics game proceeds as follows, starting from state $\st_0 = 0$\\

\noindent On round $t = 1, \ldots, \T,$\vspace{-2mm}
\begin{itemize}
  \item the learner selects a policy $\K_t \in \Flq$ and the adversary selects instance $\z_t = (\lqQ_t, \lqR_t)$ such that $\lqQ_t \succeq 0, \lqR_t \succeq 0$ and $\trace(Q_t), \trace(R_t) \leq C$ and $\zd_t$ such that $\|\zd_t\|_2 \leq W$
  \item the learner receives loss $\loss(\f_t, \st_t, \z_t) = \st_t^\top \lqrQ_t \st_t + \ac_t^\top \lqR_t \ac_t$ where action $\ac_t = \K_t \st_t$
  \item the state of the system transitions to $\st_{t+1} = \lqrA \st_t + \lqrB \ac_t + \zd_t$
\end{itemize}
\noindent where we assume that the transition matrices $A$ and $B$ are known to both the learner and adversary in advance. Observe that in this case, a stationary loss $\losss$ does not exist because of the adversarial perturbations $\zd_t$ in the dynamic; indeed, if a learner repeatedly plays the same policy $\K \in \Flq$, the state of the system is not guaranteed to converge to a unique stationary state. We now proceed to obtain an upper bound on the value $\val_{\sf{adv}, \T}$ in the following corollary, by directly controlling the dynamic stability parameters $\{\stable_{\RERM, t}\}$ for this policy class $\Flq$ with a similar FTPL based regularized ERM as used in the proof of Corollary~\ref{cor:lqr}.

\begin{corollary}[LQR with adversarial disturbances] For the online LQR with adversarial disturbances problem, the value $\val_{\sf{adv}, \T}$ is bounded as
  \begin{equation*}
    \val_{\sf{adv}, \T} \leq \mathcal{O}(\sqrt{T\log(\T)}),
  \end{equation*}
  where the $\mathcal{O}$ notation hides the dependence on problem-specific parameters.
\end{corollary}
\begin{proof}
  As we discussed above, the stationary losses $\losss$ do not exist for this setup. Instead, we will work directly with the counterfactual losses $\losscf$ for this setup. Recall from Definition~\ref{def:loss_cf}, the counterfactual loss at time $t$ for some linear policy $\K \in \Flq$ is defined as
  \begin{align*}
    \losscf_t(\K_t, \zd_{1:t}, \z_t) = \loss(\K_t, \st_t[\K_t^{(t-1)}, \zd_{1:t-1}], \z_t).
  \end{align*}
To instantiate the above counterfactual for the LQR problem, we will define some notation. Let us denote by $\st_t = \st_t[\K_{1:t-1}, \zd_{1:t-1}]$ the state at time reached by playing the sequence of policies $\K_{1:t-1}$ and by $\stt_t = \st_t[\K_t^{(t-1)}, \zd_{1:t-1}]$ the state when the learner plays polices $\K_t$ for the first $t-1$ time steps. Further, let us denote by $\cov_t = \st_t\st_t^\top$ the rank $1$ covariance matrix at time $t$ for state $\st_t$ and similarly $\covt_t = \stt_t\stt_t^\top$ for state $\stt_t$. With this notation, we have the losses
{\small
\begin{align}\label{eq:losscf-adv}
  \loss_t(\K_t, \st[\K_{1:t-1}, \zd_{1:t-1}], \z_t) = \trace((Q_t+\K_t^\top R_t \K_t)\cov_t) \quad \text{and} \quad   \losscf_t(\K_t, \zd_{1:t-1}, \z_t) = \trace((Q_t+\K_t^\top R_t \K_t)\covt_t).
\end{align}
}
We now proceed to define the regularized ERM that we shall use and derive an upper bound on the dynamic stability parameters.
\paragraph{Regularized ERM.} As in the proof of Corollary~\ref{cor:lqr}, we will consider the class of dual regularized ERM derived by the FTPL strategy
\begin{align}\label{eq:k-ftpl}
  \K_{t, \sig} = \argmin_{\K \in \Flq}\left(\sum_{s=1}^t \Eu{\z_s\sim\ad_s}[\inner{Q_s + \K^\top R \K}{\covt_s}] - \inner{\sig}{\K}\right)
\end{align}
where $\sig \in \real^{d\times k}$ such that each coordinate of $\sig$ is sampled i.i.d. from the exponential distribution with parameter $\regp > 0$. Following a similar argument as the one in the proof of Corollary~\ref{cor:lqr}, we have
\begin{align}
  \Eu{\sig}[\|\K_{t, \sig} - \K_{t+1, \sig} \|_1] \leq c\regp\cdot \lip(kd)^2\kappa \defn \regp_\K,
\end{align}
where the constant $\lip$ represents the Lipschitz constant of the function $\losscf$ (see Lemma~\ref{lem:adv-lip-bnd}).

\paragraph{Dynamic stability parameters.} For any time $t > 0$ and the policies $\{\K_t\}$ given by equation~\eqref{eq:k-ftpl} (we drop the dependence on the noise $\sig$) and any sequence of adversarial instances $(\zd_{1:t}, \z_{1:t})$, we have
\begin{align}\label{eq:dyn-stbl-adv}
  |\losscf_t(\K, \zd_{1:t-1}, \z_t) - \loss_t(\K_t, \st_t, \z_t)| &= |\inner{Q_t + \K_t^\top R_t\K_t}{\cov_t - \covt_t}|\nonumber\\
  &\leq \trace(Q_t +\K_t^\top R_t\K_t)\cdot \|\cov_t - \covt_t \|_2.
\end{align}
Thus, in order to obtain a bound on the dynamic stability parameters, we need to obtain a bound on the spectral norm of the difference $\cov_t - \covt_t$. To do so, we begin by bounding the distance between the states $\st_t$ and $\stt_t$ as
\begin{align}
  \st_t - \stt_t &= (A + B\K_{t-1})\st_{t-1} + \zd_{t-1} - (A+B\K_t)\stt_{t-1} - \zd_{t-1} \nonumber\\
  &= (A+B\K_t)(\st_{t-1} - \stt_{t-1}) + B (\K_{t-1} - \K_t)\st_{t-1}\nonumber \\
  &= (A+B\K_t)^{t-1}(\st_1 - \stt_1) + \sum_{s=2}^{t-1} (A+B\K_t)^{t-s}B(\K_s - \K_t)\st_s\;,
\end{align}
where the final inequality follows by unrolling the recursion. Observe that the first term in the above equality is $0$ since both $\st_1 = \stt_1 = \zd_1$. Taking the $\ell_2$ norm on both sides, we get,
\begin{align}\label{eq:bnd-norm-st}
  \|\st_t - \stt_t\|_2 &\stackrel{\1}{\leq} C_x\sig_B \kappa \regp_\K \sum_{s=2}^{t-1} (1-\gamma)^{t-s}(t-s) \nonumber\\
  &\stackrel{\2}{\leq} \frac{C_x\sig_B \kappa \regp_\K}{\gamma} = \; : C_{x,2}
\end{align}
where inequality $\1$ follows by using the fact that $\K_t$ is $(\kappa, \gamma)$-strongly stable and the bound on the norm of the state $\|\st_s\|_2 \leq \frac{\kappa W}{\gamma} \defn C_x$ and $\2$ follows by summing up the series. With this bound, we obtain an expression for the difference between the covariances at time $t+1$ as

\begin{align}
  \covt_{t+1} - \cov_{t+1} &= \left(\zd_t((A+B\K_{t+1})\stt_t)^\top + (A+B\K_{t+1})\stt_t\zd_t^\top + (A+B\K_{t+1}) \covt_{t} (A+B\K_{t+1})^\top\right) \nonumber\\
  &\quad - \left(\zd_t((A+B\K_{t})\st_t)^\top + (A+B\K_{t})\st_t\zd_t^\top + (A+B\K_{t}) \cov_{t} (A+B\K_{t})^\top\right)\nonumber\\
  &= \underbrace{\zd_t(\stt_t - \st_t)^\top(A+B\K_{t+1})^\top + (A+B\K_{t+1})(\stt_t - \st_t)\zd_t^\top}_{\err_1^{t}}\nonumber\\
  &\quad+ \underbrace{\zd_t\st_t^\top(\K_{t+1} - \K_t)^\top B^\top+ B(\K_{t+1} - \K_t)\st_t\zd_t^\top}_{\err_2^{t}}\nonumber\\
  &\quad+ \underbrace{B(K_{t+1} - \K_t)\cov_t(A+B\K_t)^\top + (A+B\K_{t+1})\cov_t(B(\K_{t+1} - \K_t))^\top}_{\err_3^{t}}\nonumber \\
  &\quad + (A+B\K_{t+1})(\covt_t - \cov_t)(A+B\K_{t+1})^\top.\nonumber
\end{align}
Let us denote by $\errX^{t+1} \defn \covt_{t+1} - \cov_{t+1}$ the difference between the covariance at time $t+1$ and by $\Kt \defn A+B\K_{t+1}$. With this notation, we can rewrite the above as
\begin{align}
  \errX^{t+1} &=  (\Kt)\errX^t(\Kt)^\top + \sum_{i=1}^3 \err_i^{t} \nonumber\\
  &= \Kt^2\errX^{t-1}(\Kt^2)^\top + \Kt \sum_{i=1}^3 \err_i^{t-1} \Kt^\top + \sum_{i=1}^3 \err_i^{t}\nonumber \\
  &= \Kt^{t}\errX^1(\Kt^t)^\top + \sum_{i=1}^3\sum_{s=2}^{t}\Kt^{t-s+1}\err_i^s (\Kt^{t-s+1})^\top,
\end{align}
where in the last equality observe that $\errX^1 = 0$. In order to bound the deviation $\|\errX^{t+1} \|_2$, we will now bound each of three terms in the above equation separately. \\

\textbf{Bound for $\err_1$.} To obtain a bound on the term with the error $\err_1$, recall from equation~\eqref{eq:bnd-norm-st} that we have $\|\st_t - \stt_t\|_2 \leq C_{x,2}$. With this, we have,
\begin{align}\label{eq:bnd-err1}
  \sum_{s=2}^{t}\Kt^{t-s+1}\err_1^s (\Kt^{t-s+1})^\top &\leq 2WC_{x,2}\kappa^3\cdot\sum_{s=2}^{t}(1-\gamma)^{2(t-s+1)}\nonumber\\
  &\leq \frac{2W\kappa^3}{\gamma}\cdot C_{x,2}\;.
\end{align}

\textbf{Bound for $\err_2$.} For the error term corresponding to $\err_2$, recall that we have $\|K_{t+1} - \K_s\|_2 \leq (t+1 -s)\cdot\regp_\K$. Substituting this in the error term, we have,
\begin{align}\label{eq:bnd-err2}
  \sum_{s=2}^{t}\Kt^{t-s+1}\err_2^s (\Kt^{t-s+1})^\top &\leq 2\sig_BWC_x \regp_\K \kappa^2\sum_{s=2}^{t} (1-\gamma)^{2(t-s+1)}(t-s+1)\nonumber\\
  &\leq \frac{2\sig_BWC_x \kappa^2}{\gamma^2}\cdot \regp_\K\;.
\end{align}

\textbf{Bound for $\err_3$.} Finally, for the error term $\err_3$, observe that the spectral norm of the covariance $\|\cov_t\|_2 \leq C_x^2$, and substituting this in the sum, we have
\begin{align}\label{eq:bnd-err3}
  \sum_{s=2}^{t}\Kt^{t-s+1}\err_3^s (\Kt^{t-s+1})^\top &\leq 2 C_x^2\sig_B\kappa^3\regp_\K \sum_{s=2}^t (1-\gamma)^{2(t-s+1)}(t-s+1)\nonumber\\
  &\leq \frac{ C_x^2\sig_B\kappa^3}{\gamma^2}\cdot\regp_\K\;.
\end{align}

\noindent Substituting the bounds obtained in equations~\eqref{eq:bnd-err1},~\eqref{eq:bnd-err2} and~\eqref{eq:bnd-err3} in the upper bound on the stability parameters in equation~\eqref{eq:dyn-stbl-adv}, we have that the dynamic stability parameters $\beta_{\RERM,t} = c_\stable \regp_\K$, where the constant $c_\stable$ depends on problem dependent parameters and can be obtained from the above equations. Having established a bound on the dynamic stability, we now upper bound the value for this problem.

\paragraph{Bound on the value.} The value of the online LQR with adversarial disturbance problem is
\begin{align}
  \val_{\sf{adv}, \T}(\Flq, \Z, \dyn, \loss) &\stackrel{\1}{\leq} \sum_{t=1}^\T \stable_{\RERM, t} + 2\Rseq_{T}(\losscf\com \Flq) + \frac{\stk\adim\sdim}{\regp}\nonumber\\
 &\stackrel{\2}{\leq} c_{\stable}\regp_\K\cdot T + 2\Rseq_{T}(\losscf\com \Flq) + \frac{\stk\adim\sdim}{\regp}\;,
\end{align}
where $\1$ follows from the fact that $\En[\pert_i] = 1/\regp$ and $\2$ follows from noting that each of the dynamic stability parameters is upper bounded by $c_\stable \regp_\K$.

To obtain a bound on the sequential Rademacher complexity of the class, observe that by Lemma~\ref{lem:adv-lip-bnd}, we have that the loss $\losscf$ is bounded by $B_{\sf max}$ and Lipschitz with respect to policies $\K$ with constant $\lip$. Using a standard covering number argument, one can get an $\epsilon$-net of the class $\Flq$ in the frobenius norm with at most $O(dk(\frac{1}{\epsilon})^{d k})$ elements. Given this cover, one can upper bound the complexity as
\begin{equation*}
  \Rseq_{\T}(\losss\com \Flq) \leq cB_{\sf max}\sqrt{ k d\cdot \T\log(k d \T\lip )}
\end{equation*}
for some universal constant $c>0$. Setting $\regp = O(1/\sqrt{T})$ concludes the proof of the corollary.
\end{proof}

\begin{lemma}\label{lem:adv-lip-bnd}
For the counterfactual loss $\losscf_t(\K, \zd_{1:t-1}, \z_t)$ defined in equation~\eqref{eq:losscf-adv} and policy class $\Flq$, we have
\begin{itemize}
  \item[a] $\losscf_t$ is bounded by $|\losscf(\K, \zd_{1:t-1, \z_t})| \leq C(1+\kappa^2)\cdot \left(\frac{\kappa W}{\gamma}\right)^2 :\,= B_{\sf{max}}$.
  \item[b] $\losscf_t$ is Lipschitz with respect to $\K$ with
  \begin{equation*}
    |\losscf_t(\K_1) - \losscf_t(\K_2)| \leq  C\left(2C_x^2\kappa + (\kappa^2+1)\left(\frac{2\kappa^4C_x \sig_B W }{\gamma^2} + \frac{\kappa^3C_x^2\sig_B}{\gamma} \right)\right) \cdot \|\K_1 - \K_2\|_2.
  \end{equation*}
\end{itemize}
\end{lemma}
\begin{proof}
  We will establish both the parts separately.
\paragraph{Proof for part (a).} Consider the counterfactual loss $\losscf_t$ at time $t$
\begin{align}
  |\losscf(\K, \zd_{1:t-1}, \z_t)| &= \inner{Q_t + K^\top R_t \K}{\stt_t\stt_t^\top}\nonumber\\
  &\leq \|Q_t + K^\top R_t \K\|_2 \cdot\|\stt_t\|_2^2 \nonumber\\
  &\leq C(1+\kappa^2)\cdot \left(\frac{\kappa W}{\gamma}\right)^2\;,
\end{align}
where the last inequality follows by using the fact that the policy $\K$ is $(\kappa, \gamma)$-strongly stable and that $\|\st\|_2 \leq C_x \defn\frac{\kappa W}{\gamma} $.

\paragraph{Proof for part (b).} Consider any two linear policies $\K_1, \K_2 \in \Flq$. The difference in the counterfactual losses is given by
\begin{align}
  |\losscf_t(\K_1, \zd_{1:t-1}, \z_t) - \losscf_t(\K_2, \zd_{1:t-1}, \z_t)| &= |\inner{Q+\K_1^\top R \K_1}{\covt_1} - \inner{Q+\K_2^\top R \K_2}{\covt_2}|\nonumber\\
  &\leq C(\kappa^2+1) \|\covt_{1,t} - \covt_{2,t}\|_2 + 2C_x^2C\kappa \|\K_1 - \K_2\|_2,
\end{align}
where we have used the notation $\covt_t = \stt_t\stt_t^\top$ to denote the covariance at time $t$ and the final inequality follows by noting that $\trace(Q), \trace{R} \leq C$ and $\|\K_i\|_2 \leq \kappa$ for $i = \{ 1,2\}$. Let us denote by $\tK = A+B\K$ the effective state transition matrix. We now focus on the term corresponding to the difference of the covariances $\covt_{1,t}- \covt_{2,t}$.
\begin{align}\label{eq:lip-decomp}
\covt_{1,t} - \covt_{2,t} &= \left(\tK_1\stt_{1, t-1}\zd_{t-1}^\top + \zd_{t-1}\stt_{1, t-1}^\top \tK_1^\top + \tK_1\covt_{1, t-1}\tK_1^\top\right) - \left(\tK_2\stt_{2, t-1}\zd_{t-1}^\top + \zd_{t-1}\stt_{2, t-1}^\top \tK_2^\top + \tK_2\covt_{2, t-1}\tK_2^\top\right)\nonumber\\
&=\underbrace{(\tK_1 - \tK_2)\stt_{1, t-1}\zd_{t-1}^\top+\zd_{t-1}\stt_{1, t-1}^\top(\tK_1 - \tK_2)^\top}_{\err_{1, t-1}} + \underbrace{\tK_2(\stt_{1, t-1}- \stt_{2, t-1})\zd_{t-1}^\top+\zd_{t-1}(\stt_{1,t-1}- \stt_{2,t-1})^\top \tK_2^\top}_{\err_{2, t-1}}\nonumber\\
&\quad + \underbrace{(\tK_1 - \tK_2)\covt_{1, t-1}\tK_1^\top + \tK_2 \covt_{1,t-1}(\tK_1 -\tK_2)^\top}_{\err_{3, t-1}} + \tK_2 (\covt_{1,t-1} - \covt_{2, t-1})\tK_2^\top\nonumber\\
&= \sum_{i=1}^3 \sum_{s = 2}^{t-1} \tK_2^{t-1+s}\err_{i, s}(\tK_2^{t-1+s})^\top.
\end{align}
In order to show establish the Lipschitzness of the loss function $\losscf$, we will obtain a bound on each of the three error terms comprising $\err_i$ separately now.

\underline{Bound on $\err_1$.} For the term corresponding to $\err_1$, observe that both the state $\stt$ and the disturbance $\zd$ are bounded vectors. Using the $(\kappa,\gamma)$-strong stability of the policy $\K_1$, we have
\begin{align}\label{eq:lip-err1}
  \|\sum_{s = 2}^{t-1} \tK_2^{t-1+s}\err_{1,s}(\tK_2^{t-1+s})^\top\|_2 \leq \frac{\kappa^2C_x\sig_BW}{\gamma}\cdot \|\K_1 - \K_2\|_2.
\end{align}

\underline{Bound on $\err_2$.} For the second term, observe that
$$\|\stt_{1,t}-\stt_{2,t}\|_2 \leq \frac{\kappa \sig_BC_x}{\gamma}\cdot\|\K_1 - \K_2\|_2.$$ With this, we can bound the second term in equation~\eqref{eq:lip-decomp} as
\begin{align}\label{eq:lip-err2}
    \|\sum_{s = 2}^{t-1} \tK_2^{t-1+s}\err_{2,s}(\tK_2^{t-1+s})^\top\|_2 \leq \frac{\kappa^4C_x \sig_B W }{\gamma^2}\cdot \|\K_1 - \K_2\|_2.
\end{align}

\underline{Bound on $\err_3$.} For the final term, note that $\|\covt_{t}\|_2 \leq C_x^2$. With this, we can bound the term corresponding to $\err_3$ as
\begin{align}\label{eq:lip-err3}
  \|\sum_{s = 2}^{t-1} \tK_2^{t-1+s}\err_{2,s}(\tK_2^{t-1+s})^\top\|_2 \leq \frac{\kappa^3C_x^2\sig_B}{\gamma}\cdot\|\K_1 - \K_2\|_2.
\end{align}

Combining the bounds obtained in equations~\eqref{eq:lip-err1},~\eqref{eq:lip-err2} and~\eqref{eq:lip-err3}, with the upper bound in equation~\eqref{eq:lip-decomp} establishes the desired claim.
\end{proof}

\newpage
%\bibliographystyle{alpha}
%\bibliography{refs}
\printbibliography
\end{document}